\documentclass[twoside,11pt]{article}

\usepackage{jmlr2e}

\usepackage{hyperref}
\usepackage{url}

\usepackage[tight,TABTOPCAP]{subfigure}
\usepackage{algorithm}
\usepackage{hyperref}
\usepackage{amsmath}
\usepackage{color}
\usepackage{algpseudocode}

\newcommand{\beq}{\vspace{0mm}\begin{equation}}
\newcommand{\eeq}{\vspace{0mm}\end{equation}}
\newcommand{\beqs}{\vspace{0mm}\begin{eqnarray}}
\newcommand{\eeqs}{\vspace{0mm}\end{eqnarray}}
\newcommand{\barr}{\begin{array}}
\newcommand{\earr}{\end{array}}

\newcommand{\Imat}{{\bf I}}

\newcommand{\bv}[0]{{\boldsymbol{b}}}

\newcommand{\mv}[0]{{\boldsymbol{m}}}

\newcommand{\rv}{\boldsymbol{r}}

\newcommand{\xv}{\boldsymbol{x}}
\newcommand{\yv}{\boldsymbol{y}}

\newcommand{\cdotv}{\boldsymbol{\cdot}}

\newcommand{\Phimat}{\boldsymbol{\Phi}}

\newcommand{\thetav}{\boldsymbol{\theta}}

\newcommand{\phiv}{\boldsymbol{\phi}}

\newcommand{\E}{\mathbb{E}}

\newtheorem{thm}{Theorem} 
\newtheorem{cor}[thm]{Corollary}
\newtheorem{lem}[thm]{Lemma}

\usepackage{booktabs}
\usepackage{array}
\newcolumntype{L}[1]{>{\raggedright\let\newline\\\arraybackslash\hspace{0pt}}m{#1}}
\newcolumntype{C}[1]{>{\centering\let\newline\\\arraybackslash\hspace{0pt}}m{#1}}
\newcolumntype{R}[1]{>{\raggedleft\let\newline\\\arraybackslash\hspace{0pt}}m{#1}}

\jmlrheading{17}{2016}{1-44}{12/15; Revised 7/16}{9/16}{Mingyuan Zhou, Yulai Cong, and Bo Chen}

\ShortHeadings{Augmentable Gamma Belief Networks}{Zhou, Cong, and Chen}
\firstpageno{1}

\begin{document}

\title{Augmentable Gamma Belief Networks}

\author{\name Mingyuan Zhou\,\thanks{Correspondence should be addressed to M. Zhou 
or B. Chen. 
 } 
 \email mingyuan.zhou@mccombs.utexas.edu\\
       \addr Department of Information, Risk, and Operations Management\\
       McCombs School of Business\\
        The University of Texas at Austin\\
       Austin, TX 78712, USA
       \AND
       \name Yulai Cong \email yulai\_cong@163.com\\
       \name Bo Chen\,$^*$ 
       \email bchen@mail.xidian.edu.cn \\
       \addr National Laboratory of Radar Signal Processing\\ 
       Collaborative Innovation Center of Information Sensing and Understanding\\
       Xidian University\\
       Xi'an, Shaanxi 710071, China}

\editor{Francois Caron}

\maketitle

\begin{abstract}
To infer multilayer deep representations of high-dimensional discrete and nonnegative real vectors, we propose an augmentable gamma belief network (GBN) that factorizes each of its hidden layers into the product of a sparse connection weight matrix and the nonnegative real hidden units of the next layer. The GBN's hidden layers are jointly trained with an upward-downward Gibbs sampler that solves each layer with the same subroutine. The gamma-negative binomial process combined with a layer-wise training strategy allows inferring the width of each layer given a fixed budget on the width of the first layer. Example results illustrate interesting relationships between the width of the first layer and the inferred network structure, and demonstrate that the GBN can add more layers to improve its performance in both unsupervisedly extracting features and predicting heldout data. For exploratory data analysis, we extract trees and subnetworks from the learned deep network to visualize how the very specific factors discovered at the first hidden layer and the increasingly more general factors discovered at deeper hidden layers are related to each other, and we generate synthetic data by propagating random variables through the deep network from the top hidden layer back to the bottom data layer.
\end{abstract}


\section{Introduction}

There has been significant recent interest in deep learning. Despite its tremendous success in supervised learning, inferring a multilayer data representation in an unsupervised manner remains a challenging problem \citep{Bengio+chapter2007,ranzato2007unsupervised,Bengio-et-al-2015-Book}. To generate data with a deep network, it is often unclear how to set the structure of the network, including the depth (number of layers) of the network and the width (number of hidden units) of each layer. 
In addition, for some commonly used deep generative models, including the sigmoid belief network (SBN),  deep belief network (DBN), and deep Boltzmann machine (DBM),  the hidden units are often restricted to be binary. More specifically, 
the SBN, which connects the binary units of adjacent layers via the sigmoid functions, infers a deep representation of multivariate binary vectors \citep{neal1992connectionist,saul1996mean}; the DBN \citep{hinton2006fast} is a SBN whose top hidden layer is replaced by the restricted Boltzmann machine (RBM) \citep{POE} that is undirected; and the DBM is an undirected deep network that connects the binary units of adjacent layers using the RBMs \citep{salakhutdinov2009deep}. 
 All these three 
  deep networks are designed to model binary observations, without principled ways to set the network structure. Although one may modify the bottom layer 
 to model Gaussian and multinomial observations, the hidden units of these networks are still typically restricted to be binary \citep{salakhutdinov2009deep,larochelle2012neural,salakhutdinov2013learning}. 
  To generalize these models, one may consider the exponential family harmoniums \citep{welling2004exponential,xing2005mining} to construct more general networks with non-binary hidden units, but often at the expense of noticeably increased complexity in training and data fitting. 
  To model real-valued data without restricting the hidden units to be binary, one may consider 
  the general framework of 
  nonlinear Gaussian belief networks \citep{frey1999variational}  that constructs continuous  hidden units by nonlinearly transforming Gaussian distributed latent variables, including as special cases both the continuous SBN of \citet{frey1997continuous,frey1997variational}
 and the rectified Gaussian nets of \citet{hinton1997generative}. More recent scalable generalizations under that framework include  variational auto-encoders \citep{kingma2013auto} and deep latent Gaussian models \citep{rezende2014stochastic}.
 
Moving beyond conventional deep generative models using binary or nonlinearly transformed Gaussian  hidden units and setting the network structure in a heuristic manner,
we construct  deep networks using gamma distributed nonnegative real hidden units, and combine the gamma-negative binomial process \citep{NBP2012,NBP_CountMatrix} with a greedy-layer wise training strategy to automatically infer the network structure.
The proposed model is called the augmentable gamma belief network, referred to hereafter for  brevity as the GBN, which  factorizes the observed or latent count vectors under the Poisson likelihood into the product of a factor loading matrix and the gamma distributed hidden units (factor scores) of layer one; and further factorizes the shape parameters of the gamma hidden units of each layer  into the product of a connection weight matrix and the gamma hidden units of the next layer. The GBN together with Poisson factor analysis can unsupervisedly infer a multilayer representation from multivariate count vectors, with a simple but powerful mechanism to capture the correlations between the visible/hidden features across all layers and handle highly overdispersed counts. With the Bernoulli-Poisson link function \citep{EPM_AISTATS2015}, the GBN is further applied to high-dimensional sparse binary vectors by truncating latent counts, and with a Poisson randomized gamma distribution, the GBN is further applied to high-dimensional sparse nonnegative real data by randomizing the gamma shape parameters with latent counts. 

For tractable inference of a deep generative model, one often applies either a sampling based procedure \citep{neal1992connectionist,frey1997continuous} or variational inference \citep{saul1996mean,frey1997variational,ranganath2014deep,kingma2013auto}. However, conjugate priors on the model parameters that connect adjacent layers are often unknown, making it difficult to develop fully Bayesian inference that infers the posterior distributions of these parameters. It was not until recently that a Gibbs sampling algorithm,  imposing priors on the network connection weights and sampling from their conditional  posteriors,
was developed for the SBN by \citet{gan2015learning}, using the Polya-Gamma data augmentation technique developed for logistic models \citep{PolyaGamma}. In this paper, we will develop data augmentation technique unique for the augmentable GBN, allowing us to develop a fully Bayesian upward-downward Gibbs sampling algorithm to infer the posterior distributions of not only the hidden units, but also the connection weights between adjacent layers.    



Distinct from previous deep networks that often require tuning both the width (number of hidden units) of each layer and the network depth (number of layers), the GBN employs nonnegative real hidden units and automatically infers the widths of subsequent layers given a fixed budget on the width of its first layer.
Note that the budget could be infinite and hence the whole network can grow without bound as more data are being observed. 
Similar to other belief networks that can often be improved by adding more hidden layers 
\citep{hinton2006fast,sutskever2008deep,Bengio-et-al-2015-Book},
for the proposed model, 
when the budget on the first layer is finite and hence the ultimate capacity of the network could be limited, our experimental results also show that a GBN equipped with a narrower first layer could increase its depth to match or even outperform a shallower one with a substantially wider first layer.

The gamma distribution density function has the highly desired strong non-linearity for deep learning, 
but the existence of neither a conjugate prior nor a closed-form maximum likelihood estimate \citep{choi1969maximum} 
for its shape parameter makes a deep network with gamma hidden units appear unattractive. Despite seemingly difficult, we discover that, by generalizing the data augmentation and marginalization techniques for discrete data modeled with the Poisson, gamma, and negative binomial distributions \citep{NBP2012}, one may propagate latent counts one layer at a time from the bottom data layer to the top hidden layer, with which one may derive an efficient upward-downward Gibbs sampler that, 
one layer at a time in each iteration, upward samples Dirichlet distributed connection weight vectors and then downward samples gamma distributed hidden units, with the latent parameters of each layer solved with the same subroutine. 

With extensive experiments in text and image analysis, we demonstrate that the deep GBN with two or more hidden layers clearly outperforms the shallow GBN with a single hidden layer in both unsupervisedly extracting latent features for classification and predicting heldout data. Moreover, we demonstrate the excellent ability of the GBN in exploratory data analysis: by extracting trees and subnetworks from the learned deep network, we can follow the paths of each tree to visualize 
various aspects of the data, from very general to very specific, and understand how they are related to each other. 

In addition to constructing a new deep network that well fits high-dimensional sparse binary, count, and nonnegative real data, developing an efficient upward-downward Gibbs sampler, and applying the learned deep network for exploratory data analysis, other contributions of the paper include: 1) proposing novel link functions, 2) combining the gamma-negative binomial process \citep{NBP2012,NBP_CountMatrix} with a layer-wise training strategy to automatically infer the network structure; 3) revealing the relationship between the upper bound imposed on the width of the first layer and the inferred widths of subsequent layers; 4) revealing the relationship between the depth of the network and the model's ability to model overdispersed counts; 
and 5) generating  multivariate high-dimensional discrete or nonnegative real vectors, whose distributions are governed by the GBN, 
 by propagating 
the gamma hidden units of
the top hidden layer back to the bottom data layer. We note this paper significantly extends 
our recent conference publication \citep{PGBN_NIPS2015} that proposes the Poisson GBN. 

\section{Augmentable Gamma Belief Networks}

Denoting $\thetav_j^{(t)}\in\mathbb{R}_+^{K_{t}}$ as the $K_{t}$ hidden units of sample $j$ at layer $t$, where $ \mathbb{R}_+=\{x:x\ge 0\}$, the generative model of the augmentable gamma belief network (GBN) with $T$ hidden layers, from top to bottom, 
 is expressed as
\beqs\displaystyle
&\thetav_j^{(T)}\sim\mbox{Gam}\left(\rv,1\big/c_j^{(T+1)}\right),\notag\\ 
&\vdots\notag\\
&\thetav_j^{(t)}\sim\mbox{Gam}\left(\Phimat^{(t+1)}\thetav_j^{(t+1)},1\big/c_j^{(t+1)}\right),\notag\\
&\vdots\notag\\
&
\thetav_j^{(1)}\sim\mbox{Gam}\left(\Phimat^{(2)}\thetav_j^{(2)},
{p_j^{(2)}}\big/{\big(1-p_j^{(2)}\big)}\right), 
 \label{eq:PGBN}
\eeqs
where $x\sim\mbox{Gam}(a,1/c)$ represents a gamma distribution with mean $a/c$ and variance $a/c^2$.
For $t=1,2,\ldots,T-1$, the GBN 
factorizes the shape parameters of the gamma distributed hidden units $\thetav_j^{(t)}\in\mathbb{R}_+^{K_{t}}$ of layer $t$ into the product of the connection weight matrix $\Phimat^{(t+1)}\in\mathbb{R}_+^{K_{t}\times K_{t+1}}$ and the hidden units $\thetav_j^{(t+1)}\in\mathbb{R}_+^{ K_{t+1}}$ of layer $t+1$; the top layer's hidden units $\thetav_j^{(T)}$ share the same vector $\rv=(r_1,\ldots,r_{K_T})' $ as their gamma shape parameters; and the $p_j^{(2)}$ are probability parameters and 
$\{1/c^{(t)}\}_{3,T+1}$ are gamma scale parameters, with $c_j^{(2)}:=\big(1-p_j^{(2)}\big)\big/p_j^{(2)}$. We will discuss later how to measure the connection strengths between the nodes of adjacent layers and the overall popularity of a factor at a particular hidden layer.

For scale identifiability and ease of inference and interpretation, each column of $\Phimat^{(t)}\in\mathbb{R}_+^{K_{t-1}\times K_{t}}$ is restricted to have a unit $L_1$ norm and hence $0\le\Phimat^{(t)}(k',k)\le 1$. 
To complete the hierarchical model, for $t\in\{1,\ldots,T-1\}$, we let 
\beqs
&\phiv_k^{(t)}\sim\mbox{Dir}\big(\eta^{(t)},\ldots,\eta^{(t)}\big), ~~r_k\sim\mbox{Gam}\big(\gamma_0/K_T,1/c_0\big) \label{eq:Phi}
\eeqs 
where $\phiv_k^{(t)}\in\mathbb{R}_+^{K_{t-1}}$ is the $k$th column of $\Phimat^{(t)}$; we impose $c_0\sim\mbox{Gam}(e_0,1/f_0)$ 
and $\gamma_0\sim\mbox{Gam}(a_0,1/b_0)$; 
and for $t\in\{3,\ldots,T+1\}$, we let
\beqs
&p_j^{(2)}\sim\mbox{Beta}(a_0,b_0),~~~c_j^{(t)}\sim\mbox{Gam}(e_0,1/f_0). \label{eq:c_j}
\eeqs
We expect 
the correlations between
the $K_t$ rows (latent features) of $(\thetav_1^{(t)}, \ldots,\thetav_J^{(t)})$ to be captured by the columns of $\Phimat^{(t+1)}$. Even if $\Phimat^{(t)}$ for $t\ge 2$ are all identity matrices, indicating no correlations between the latent features to be captured, our analysis in Section \ref{sec:distributed} will show that a deep structure with $T\ge 2$ could still benefit data fitting by better modeling the variability of the latent features $\thetav_j^{(1)}$. 
Before further examining the network structure, below we first introduce a set of distributions that will be used to either model different types of data or augment the model for simple inference. 


\subsection{Distributions for Count, Binary, and Nonnegative Real Data}
Below we  first describe some useful count distributions that will be used later. 
\subsubsection{Useful Count Distributions and Their Relationships}

Let the Chinese restaurant table (CRT) distribution $l\sim\mbox{CRT}(n,r)$ represent the random number of tables seated by $n$ customers in a Chinese restaurant process \citep{PolyaUrn,DP_Mixture_Antoniak, aldous,csp} with concentration parameter~$r$. Its probability mass function (PMF) can be expressed as $$
 P(l\,|\,n,r) = \frac{\Gamma(r)r^l}{\Gamma(n+r)}|s(n,l)|,\notag
$$
where $l\in\mathbb{Z}$, $\mathbb{Z}:=\{0,1,\ldots,n\}$, and $|s(n,l)|$ 
 are unsigned Stirling numbers of the first kind. 
A CRT distributed sample can be
generated by taking the summation of $n$ independent Bernoulli random variables as
$$
l=\sum_{i=1}^{n} b_i,~b_i\sim\mbox{Bernoulli}\left[{r}/{(r+i-1)}\right].\notag
$$
Let $u\sim\mbox{Log}(p)$ denote the logarithmic distribution \citep{Fisher1943,Sampling_NB_1950,johnson1997discrete} with PMF
$$
P(u\,|\,p) = \frac{1}{-\ln(1-p)}\frac{p^u}{u},\notag
$$
where $u\in\{1,2,\ldots\}$, and let $n\sim\mbox{NB}(r,p)$ denote the negative binomial (NB) distribution \citep{Yule,NB_Fitting_53} with PMF
$$
P(n\,|\,r,p)=\frac{\Gamma(n+r)}{n!\Gamma(r)} p^n(1-p)^r,\notag
$$
where $n\in\mathbb{Z}$. The NB distribution $n\sim\mbox{NB}(r,p)$ can be generated as a gamma mixed Poisson distribution as
$$
n\sim\mbox{Pois}(\lambda),~\lambda\sim\mbox{Gam}\left[r,{p}/({1-p})\right],
$$
where $p/(1-p)$ is the gamma scale parameter. 

As shown in \citep{NBP2012},
the joint distribution of $n$ and $l$ given $r$ and $p$ in 
$$
l\sim\mbox{CRT}(n,r),~n\sim\mbox{NB}(r,p),\notag
$$
where $l\in\{0,\ldots,n\}$ and $n\in\mathbb{Z}$, is the same as that in
\beq
n = \textstyle \sum_{t=1}^l u_t,~u_t\sim\mbox{Log}(p), ~l\sim\mbox{Pois}[-r\ln(1-p)], \label{eq:CompoundPo}
\eeq
which is called the Poisson-logarithmic bivariate distribution, with PMF 
$$
P(n,l\,|\,r,p)=\frac{|s(n,l)|r^l}{n!}p^n(1-p)^r\label{eq:Po-log}
.$$ We will exploit these relationships to derive efficient inference for the proposed models. 

\subsubsection{Bernoulli-Poisson Link and Truncated Poisson Distribution}

As in \citet{EPM_AISTATS2015}, the Bernoulli-Poisson (BerPo) link thresholds a random count at one to obtain a binary variable as
\beq
b =\mathbf{1}(m\ge 1),~ m\sim\mbox{Pois}(\lambda),\label{eq:BerPo}
\eeq
where $b=1$ if $m\ge 1$ and $b=0$ if $m=0$.  If $m$ is marginalized out from (\ref{eq:BerPo}), then given $\lambda$, one obtains a Bernoulli random variable as
$
b\sim\mbox{Ber}\big(1-e^{-\lambda}\big).\notag 
$
The conditional posterior of the latent count $m$ can be expressed as
$$(m\,|\,b,\lambda)\sim b\cdotv \mbox{Pois}_{+}(\lambda),$$ where $x\sim\mbox{Pois}_{+}(\lambda)$ follows a truncated Poisson distribution, with 
$P(x=k) =(1-e^{-\lambda})^{-1} {\lambda^k e^{-\lambda}}/{k!}$ for $k\in\{1,2,\ldots\}$.
Thus if $b=0$, then $m= 0$ almost surely (a.s.), and if $b=1$, then 
$m\sim\mbox{Pois}_{+}(\lambda)$, which can be simulated with a rejection sampler that has a minimal acceptance rate of 63.2\% at $\lambda=1$ \citep{EPM_AISTATS2015}. 
Given the latent count $m$ and a gamma prior on $\lambda$, one can then update $\lambda$ using the gamma-Poisson conjugacy. 
The BerPo link shares some similarities with the probit link that thresholds a normal random variable at zero, and the logistic link that lets $b\sim\mbox{Ber}[e^{x}/(1+e^{x})]$. 
We advocate the BerPo link as an alternative to the probit and logistic links since if $b=0$, then $m= 0$ a.s., which could lead to significant computational savings if the binary vectors are sparse. 
In addition, 
 the conjugacy between the gamma and Poisson distributions makes it convenient to construct hierarchical Bayesian models amenable to posterior simulation.

\subsubsection{Poisson Randomized Gamma and Truncated Bessel Distributions}
To model nonnegative data that include both zeros and positive observations,
we introduce the Poisson randomized gamma (PRG) distribution as $$x\sim\mbox{PRG}(\lambda,c),$$ whose distribution has a point mass at $x=0$ and is continuous for $x>0$. The PRG distribution is generated as a Poisson mixed gamma distribution as
\beq
x\sim\mbox{Gam}(n,1/c),~n\sim\mbox{Pois}(\lambda),\notag
\eeq
in which we define $\mbox{Gam}(0,1/c)=0$ a.s. and hence $x=0$ if and only $n=0$.
Thus the PMF of $x\sim\mbox{PRG}(\lambda,c)$ can be expressed as
\begin{align}
f_X(x\,|\,\lambda,c) &= \sum_{n=0}^\infty \mbox{Gam}(x;n,1/c) \mbox{Pois}(n;\lambda)\notag\\
&= \left(e^{-\lambda}\right)^{\mathbf{1}(x=0)} \left[{e^{-\lambda-cx}} \sqrt{\frac{\lambda c}{x}} ~ { I_{-1}\left(2\sqrt{\lambda c x }\right)}\right]^{\mathbf{1}(x>0)}\,,
\end{align}
where 
\beq\notag
I_{-1}(\alpha) = \left(\frac{\alpha}{2}\right)^{-1} \sum_{n=1}^\infty \frac{\left(\frac{\alpha^2}{4}\right)^n}{n!\Gamma(n)},~~\alpha>0
\eeq
is the modified Bessel function of the first kind $I_{\nu}(\alpha)$ with $\nu$ fixed at $-1$.
Using the laws of total expectation and total variance, or using the PMF directly, one may show that 
\beq
\E[x\,|\,\lambda,c] = \lambda/c,~~ \mbox{var}[x\,|\,\lambda,c] = 2\lambda/c^2.\notag
\eeq
Thus the variance-to-mean ratio of the PRG distribution is $2/c$, as controlled by $c$. 

The conditional posterior of $n$ given $x$, $\lambda$, and $c$ can be expressed as
\begin{align}
f_N(n\,|\,x,\lambda,c) &= \frac{\mbox{Gam}(x;n,1/c) \mbox{Pois}(n;\lambda)}{\mbox{PRG}(x;\lambda,c)} \notag\\
&= \mathbf{1}(x=0)\delta_0 +  \mathbf{1}(x>0)\sum_{n=1}^\infty \frac{ 1} { I_{-1}\left(2\sqrt{\lambda c x }\right)}\frac{(\lambda c x)^{n-\frac{1}{2}}}{n!\Gamma(n)} \delta_n\notag\\
& = \mathbf{1}(x=0)\delta_0 +  \mathbf{1}(x>0)\sum_{n=1}^\infty \mbox{Bessel}_{-1}(n;2\sqrt{cx\lambda}) \delta_n\,,\end{align}
where we define $n\sim \mbox{Bessel}_{-1}(\alpha)$ as the truncated Bessel distribution, with PMF
\beq\notag
\mbox{Bessel}_{-1}(n;\alpha)
=\frac{\left(\frac{\alpha}{2}\right)^{2n-1}}{I_{-1} (\alpha) n! \Gamma(n)}, ~~n\in\{1,2,\ldots\}.
\eeq
Thus $n=0$ if and only if $x=0$, and $n$ is a positive integer drawn from a truncated Bessel distribution if $x>0$. In Appendix \ref{app:bessel}, we plot the probability distribution functions of  the proposed PRG and truncated Bessel distributions and show how they differ from the randomized gamma  and Bessel distributions \citep{yuan2000bessel}, respectively.   

\subsection{Link Functions for Three Different Types of Observations}

If the observations are multivariate count vectors $\xv_j^{(1)}\in\mathbb{Z}^{V}$, where $V:=K_0$, then we link the integer-valued visible units to the nonnegative real hidden units at layer one using Poisson factor analysis (PFA) as
\beq\label{eq:PFA}
\xv_j^{(1)} \sim \mbox{Pois}\left(\Phimat^{(1)}\thetav_j^{(1)}\right).
\eeq
Under this construction, the correlations between the $K_0$ rows (features) of $(\xv_1^{(1)},\ldots,\xv_J^{(1)})$ are captured by the columns of $\Phimat^{(1)}$.
Detailed descriptions on how PFA is related to a wide variety of discrete latent variable models, 
including nonnegative matrix factorization \citep{NMF}, latent Dirichlet allocation \citep{LDA}, the gamma-Poisson model \citep{CannyGaP}, discrete Principal component analysis \citep{DCA}, and the focused topic model \citep{FocusTopic}, can be found in \citet{BNBP_PFA_AISTATS2012} and \citet{NBP2012}. 

We call PFA using the GBN in (\ref{eq:PGBN}) as the prior on its factor scores as the Poisson gamma belief network (PGBN), as proposed in \citet{PGBN_NIPS2015}. The PGBN can be naturally applied to factorize the term-document frequency count matrix of a text corpus, not only extracting semantically meaningful topics at multiple layers, but also capturing the relationships between the topics of different layers using the deep network, as discussed below in both Sections \ref{sec:interprete_network_structure} and \ref{sec:examples}.

If the observations are high-dimensional sparse binary vectors $\bv_j^{(1)}\in\{0,1\}^{V}$, then we factorize them 
using Bernoulli-Poisson factor analysis (Ber-PFA) as
\beq
 \bv_j^{(1)}=\mathbf{1}\big( \xv_j^{(1)} \ge 0\big),~ \xv_j^{(1)} \sim \mbox{Pois}\left(\Phimat^{(1)}\thetav_j^{(1)}\right).
\eeq
We call Ber-PFA with the augmentable GBN as the prior on its factor scores $\thetav_j^{(1)}$ as the Bernoulli-Poisson gamma belief network (BerPo-GBN).

If the observations are high-dimensional sparse nonnegative real-valued vectors $\yv_j^{(1)}\in\mathbb{R}_+^{V}$, then we 
factorize them 
using Poisson randomized gamma (PRG) factor analysis as
\beq
 \yv_j^{(1)}\sim\mbox{Gam}( \xv_j^{(1)}, 1/a_j ),~ \xv_j^{(1)} \sim \mbox{Pois}\left(\Phimat^{(1)}\thetav_j^{(1)}\right).
\eeq
We call PRG factor analysis with the augmentable GBN as the prior on its factor scores $\thetav_j^{(1)}$ as the PRG gamma belief network (PRG-GBN).

\begin{figure}[!tb]
\begin{center}
\includegraphics[width=0.45\textwidth]{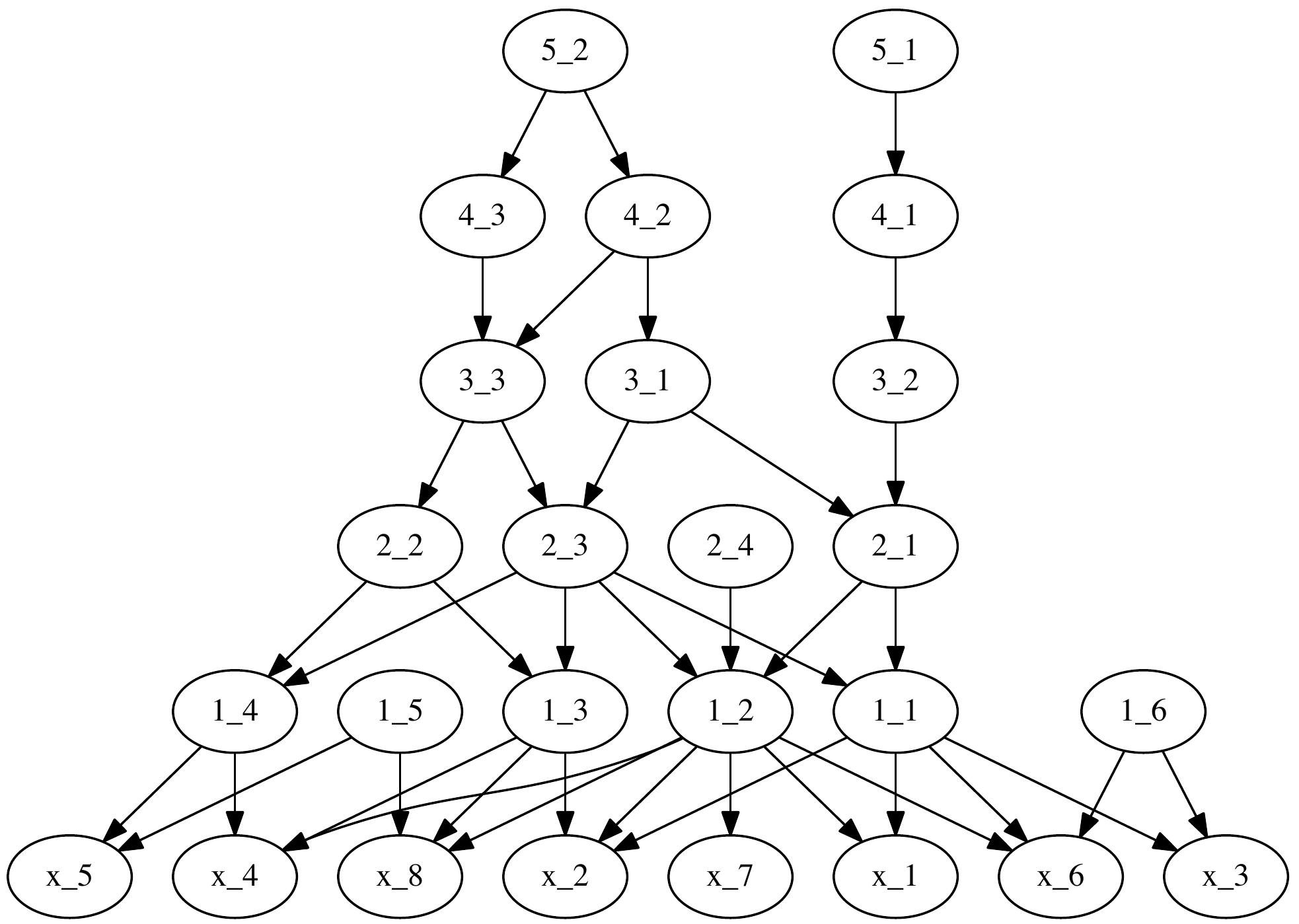}
\end{center}
\vspace{-7.mm}
\caption{\small \label{fig:exampletree} 
An example directed network of five hidden layers, with $K_0=8$ visible units, $[K_1,K_2,K_3,K_4,K_5]=[6,4,3,3,2]$, and sparse connections between the units of adjacent layers.
 }
\end{figure}

We show in Figure \ref{fig:exampletree} an example directed belief network of five hidden layers, with $K_0=8$ visible units, with $6$, $4$, $3$, $3$, and $2$ hidden units for layers one, two, three, four, and five, respectively, and with sparse connections between the units of adjacent layers.

\subsection{Exploratory Data Analysis} 
\label{sec:interprete_network_structure}
To interpret the network structure of the GBN, we notice that
\beq
\E\big[\xv_j^{(1)}\,\big|\,\thetav_j^{(t)}, \{ \Phimat^{(\ell)},c_j^{(\ell)} \}_{1,t} \big] = \left[\prod_{\ell =1}^t \Phimat^{(\ell )} \right] \frac{\thetav_j^{(t)}}{\prod_{\ell =2}^t c_j^{(\ell)}},\label{eq:x_given_theta}
\eeq
\beq
\E\big[\thetav_j^{(t)}\,\big|\, \{ \Phimat^{(\ell)},c_j^{(\ell)} \}_{t+1,T}, \rv \big] = \left[\prod_{\ell ={t+1}}^T \Phimat^{(\ell )} \right] \frac{\rv}{\prod_{\ell ={t+1}}^{T+1} c_j^{(\ell)}}.\label{eq:phi_given_r}
\eeq
Thus for visualization, it is straightforward to project the $K_t$ topics/hidden units/factor loadings/nodes of layer $t\in\{1,\ldots,T\}$ to the bottom data layer 
as the columns of the $V\times K_t$ matrix 
\beq\label{eq:projection}
\prod_{\ell =1}^t \Phimat^{(\ell )},
\eeq 
and rank their popularities 
using the $K_t$ dimensional nonnegative weight vector
\beq\label{eq:weight}
\rv^{(t)}:=
 \left[\prod_{\ell ={t+1}}^T \Phimat^{(\ell )} \right] \rv\,. 
 \eeq
To measure the connection strength between node $k$ of layer $t$ and node $k'$ of layer $t-1$, we use 
the value of $$\Phimat^{(t)}(k',k),$$ which is also expressed as $\phiv_k^{(t)}(k')$ or $\phi_{k'k}^{(t)}$. 

\begin{figure}[!tb]
\begin{center}
\includegraphics[width=0.1\textwidth]{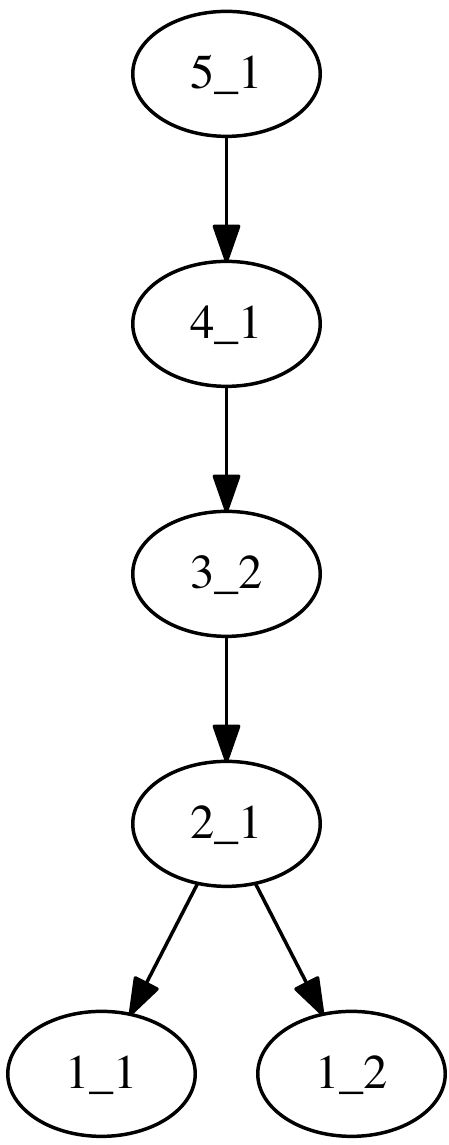}\,\,\,\,\,\,\,\,\,\,\,\,\,\,\,\,
\includegraphics[width=0.215\textwidth]{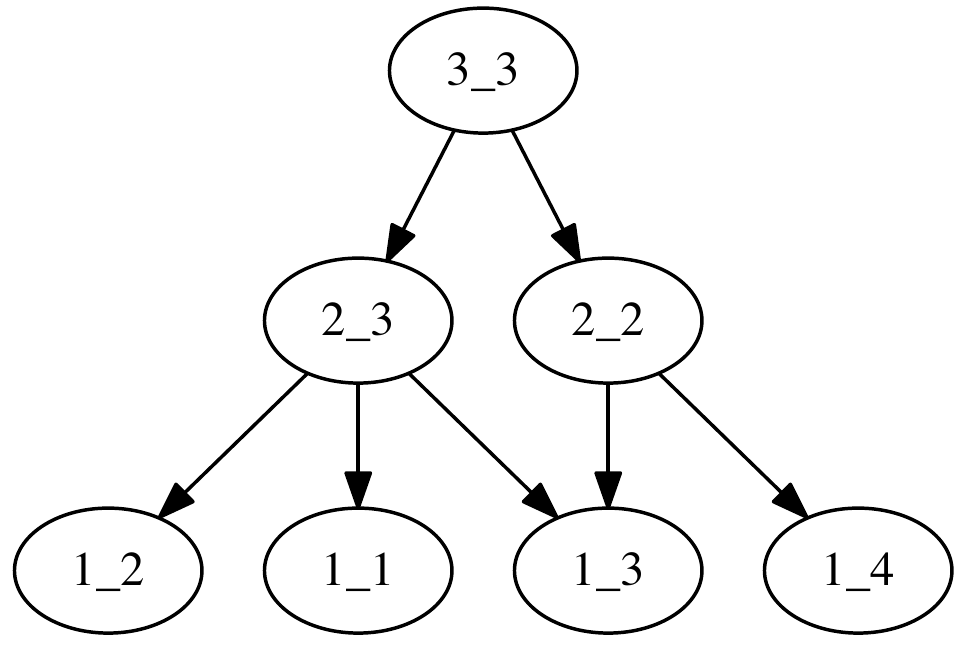}\,\,\,\,\,\,\,\,\,\,\,\,\,\,\,\,
\includegraphics[width=0.21\textwidth]{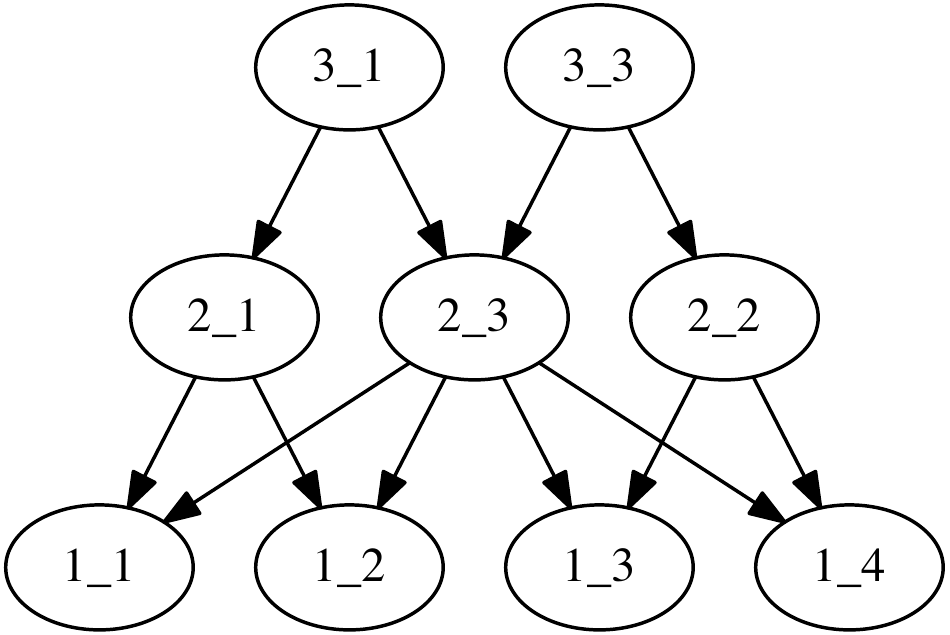}
\end{center}
\vspace{-7.mm}
\caption{\small \label{fig:exampletree1} 
Extracted from the network shown in Figure \ref{fig:exampletree}, the left plot  is a tree rooted at node $5\_1$, the middle plot is a tree rooted at node $3\_3$, and the right plot is a subnetwork consisting of both the tree rooted at node $3\_1$ and the tree rooted at node $3\_3$.
 }
\end{figure}

Our intuition is that examining the nodes of the hidden layers, via their projections to the bottom data layer, from the top to bottom layers will gradually reveal less general and more specific aspects of the data. 
 To verify this intuition and further understand the relationships between the general and specific aspects of the data, we consider extracting a tree for each node of layer $t$, where $t\ge 2$, to help visualize the inferred multilayer deep structure. 
To be more specific, to construct a tree rooted at a node of layer $t$, 
 we grow the tree downward by linking the root node (if at layer $t$) or each leaf node of the tree (if at a layer below layer $t$) to all the nodes at the layer below that are connected to the root/leaf node with non-negligible weights. Note that a tree in our definition permits a node to have more than one parent, which means that different branches of the tree can overlap with each other.  In addition, we also consider extracting subnetworks, each of which consists of multiple related trees from the full deep network.
 For example, shown in the left of Figure \ref{fig:exampletree1} is the tree extracted from the network in Figure \ref{fig:exampletree} using node $5\_1$ as the root, shown in the middle is the tree
 using node $3\_3$ as the root, and shown in the right is a subnetwork consisting of two related trees that are rooted at nodes $3\_1$ and $3\_3$, respectively. 
%

\subsubsection{Visualizing Nodes of Different Layers}
Before presenting the technical details, we first provide some example results obtained with the PGBN on extracting multilayer representations from the 11,269 training documents of the 20newsgroups 
data set ({\href{http://qwone.com/~jason/20Newsgroups/}{http://qwone.com/$\sim$jason/20Newsgroups/}}). Given a fixed budget of $K_{1\max}=800$ on the width of the first layer, with $\eta^{(t)}=0.1$ for all $t$,  a five-layer deep network inferred by the PGBN has a network structure as $[K_1,K_2,K_3,K_4,K_5]=[386 ,  63 , 58,  54 ,  51]$, meaning that there are $386$, $63$, $58$, $54$, and $51$ nodes at layers one to five, respectively.

\begin{figure}[!tb]
\begin{center}
 \includegraphics[width=0.47\textwidth]{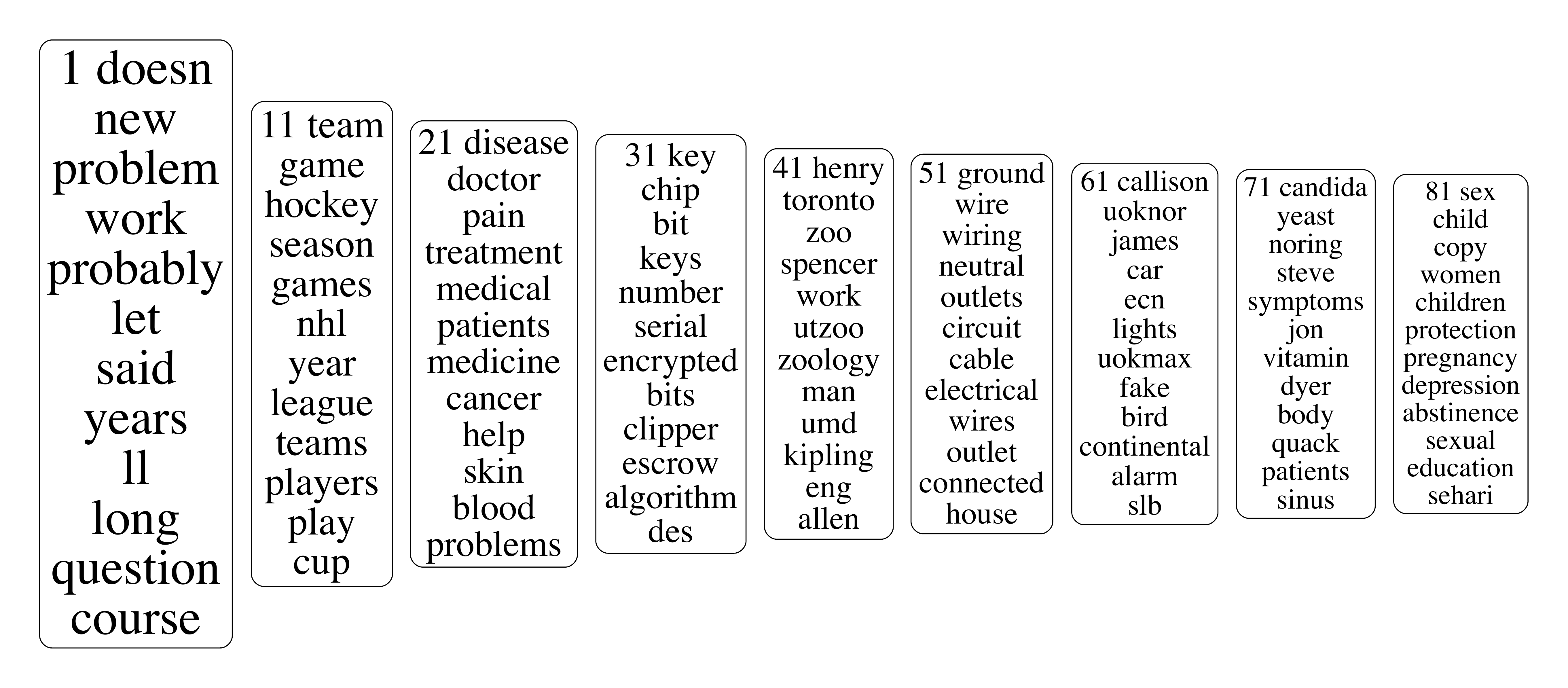}\!\!
 \includegraphics[width=0.48\textwidth]{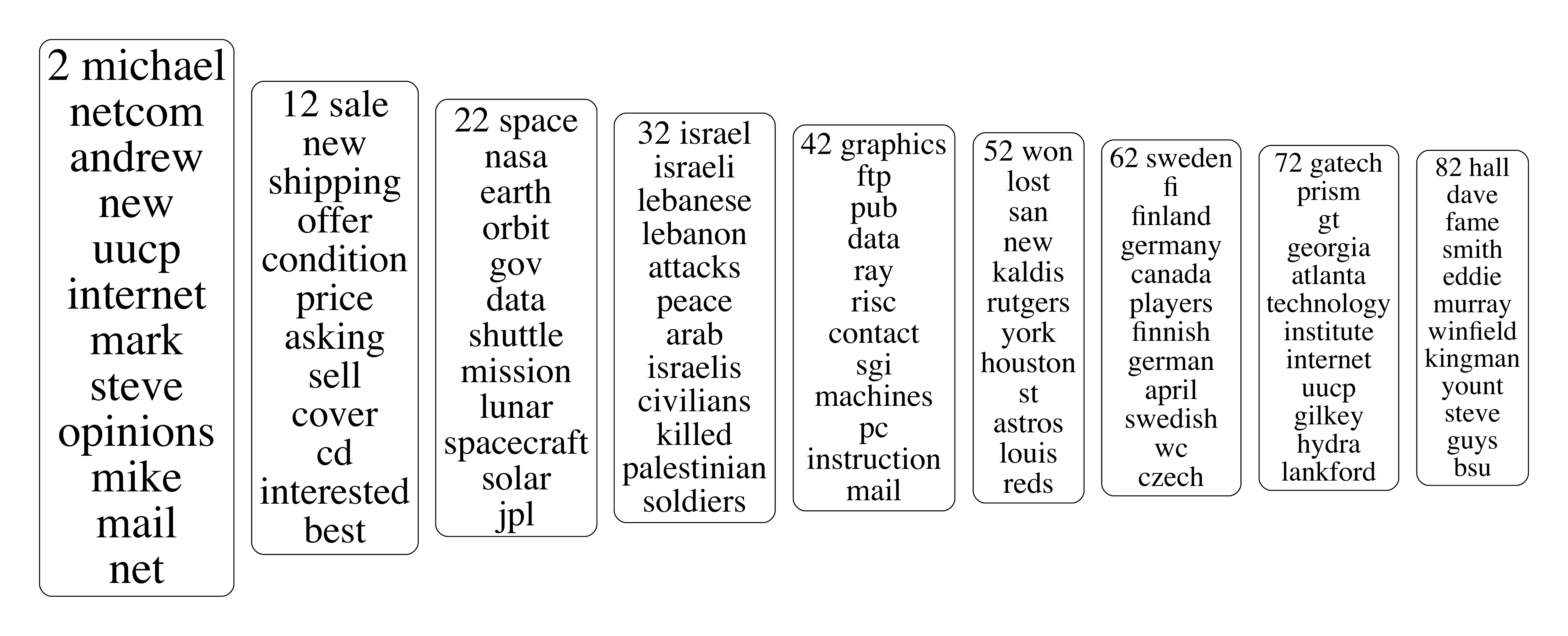}
 \end{center}
 \vspace{-12mm}
 \begin{center}
  \includegraphics[width=0.48\textwidth]{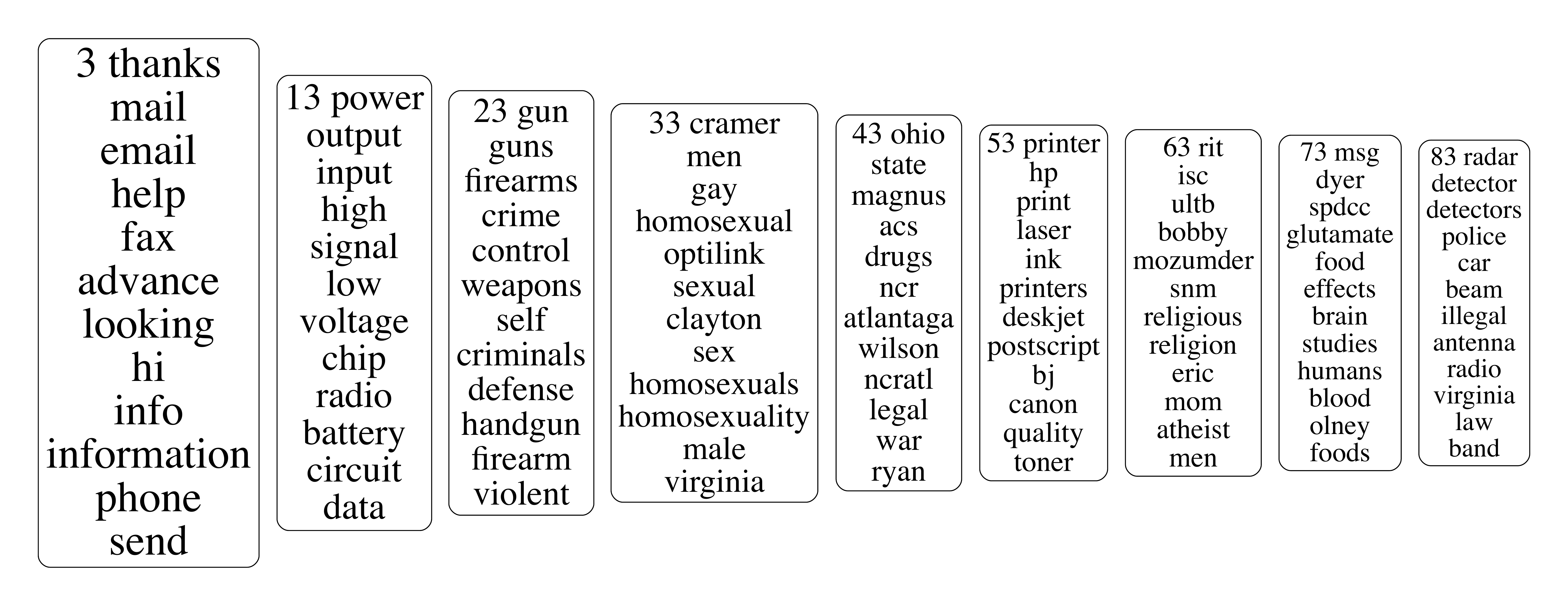}\!\!
  \includegraphics[width=0.50\textwidth]{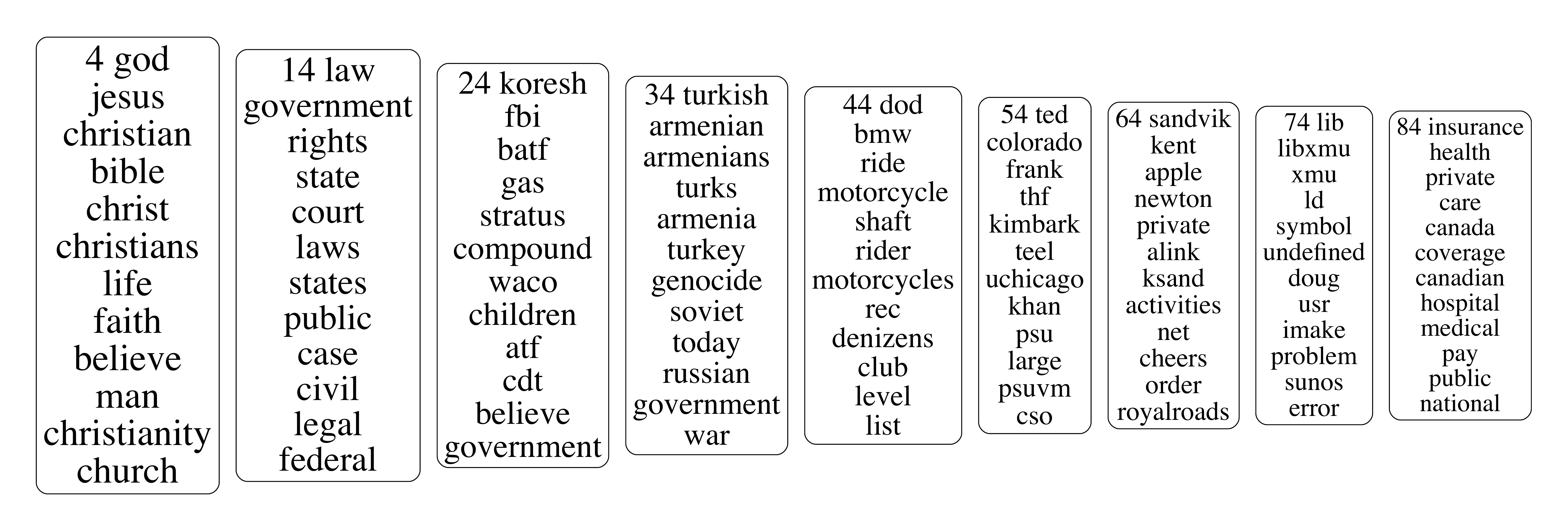}
   \end{center}
 \vspace{-12mm}
 \begin{center}
   \includegraphics[width=0.48\textwidth]{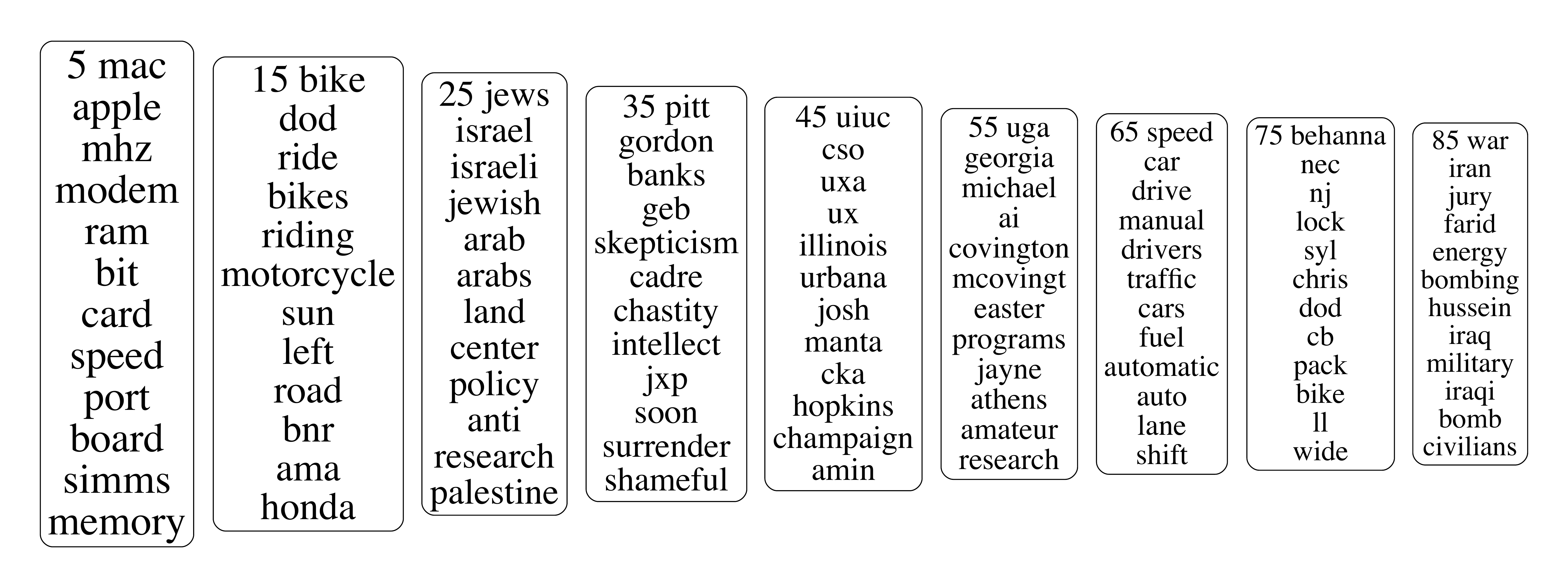}\!\!
   \includegraphics[width=0.48\textwidth]{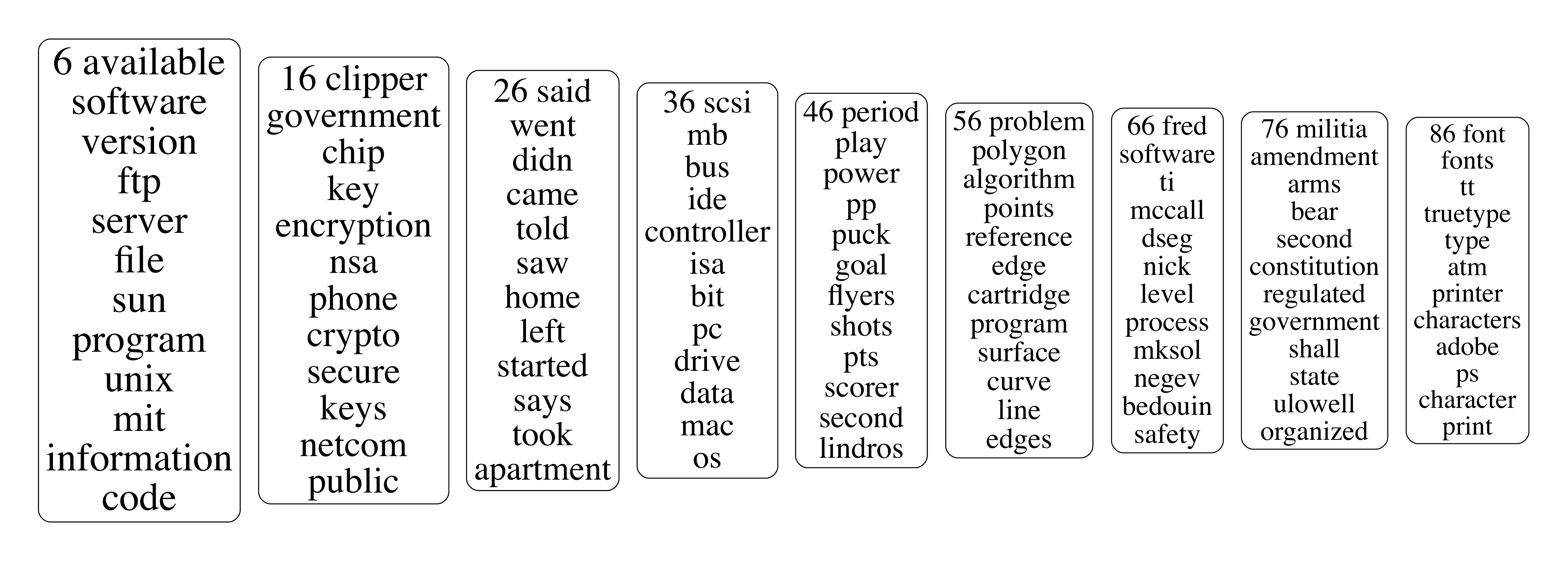}
\end{center}
\vspace{-9mm}
\caption{\small \label{fig:topic_layer1}
Example topics of layer one of the PGBN trained on the 20newsgroups corpus.
}

\begin{center}
  \includegraphics[width=1\textwidth]{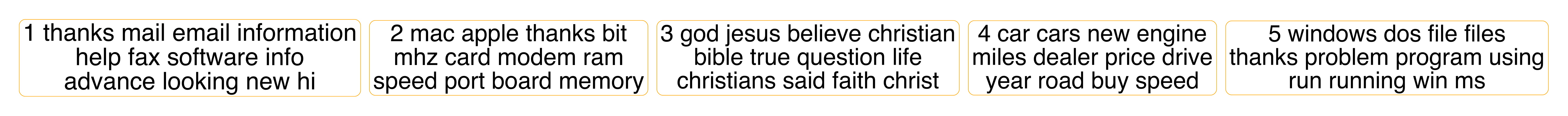}\vspace{-3mm}
 \includegraphics[width=.985\textwidth]{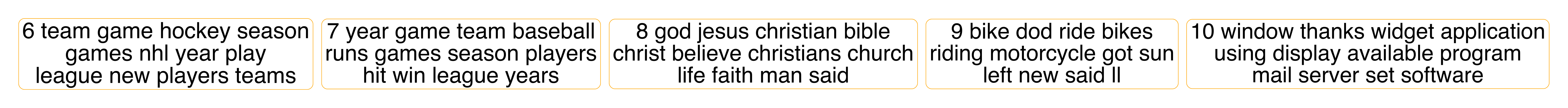}\vspace{-3mm}
  \includegraphics[width=.97\textwidth]{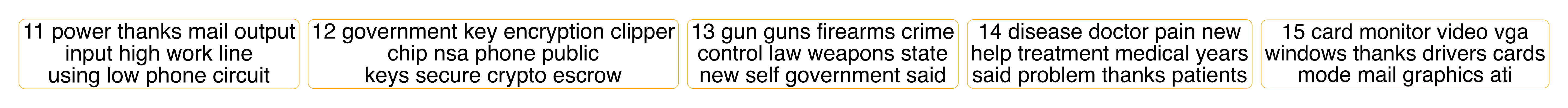}\vspace{-3mm}
  \includegraphics[width=.955\textwidth]{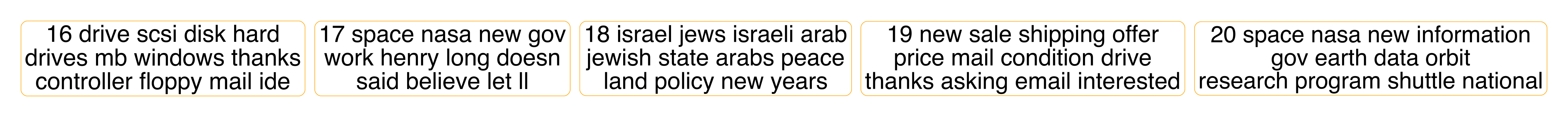}\vspace{-3mm}
   \includegraphics[width=.94\textwidth]{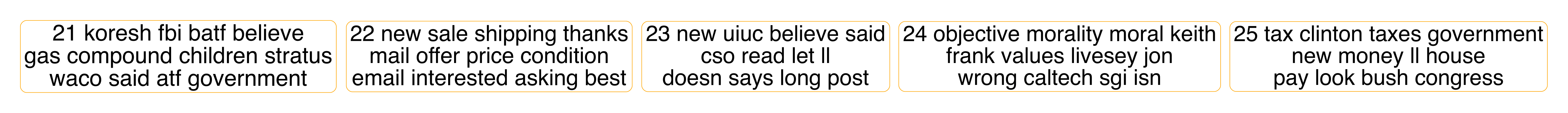}\vspace{-3mm}
   \includegraphics[width=0.925\textwidth]{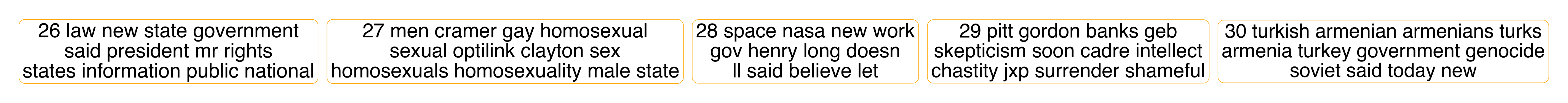}

\end{center}
\vspace{-9mm}
\caption{\small \label{fig:topic_layer3}
The top 30 topics of layer three of the PGBN trained on the 20newsgroups corpus. 
}\vspace{-0mm}

\begin{center}
  \includegraphics[width=1\textwidth]{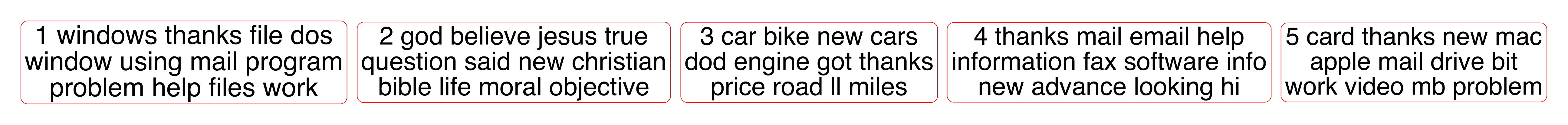}\vspace{-3mm}
 \includegraphics[width=.985\textwidth]{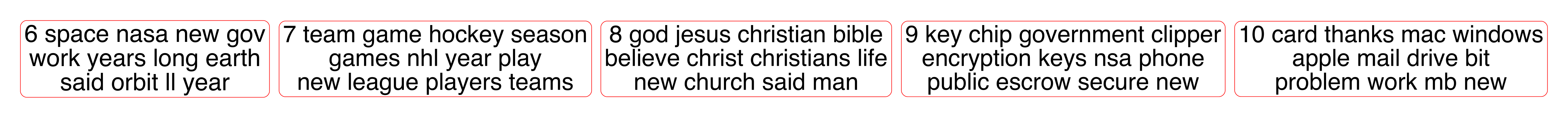}\vspace{-3mm}
  \includegraphics[width=.97\textwidth]{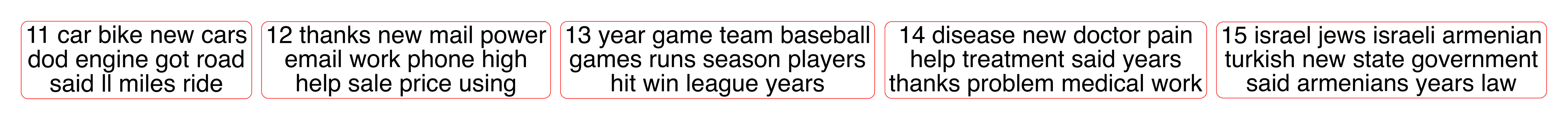}\vspace{-3mm}
  \includegraphics[width=.955\textwidth]{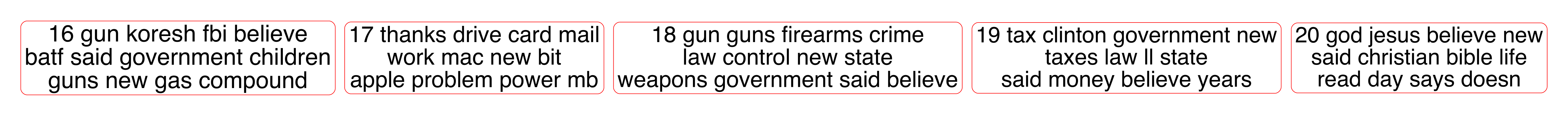}\vspace{-3mm}
   \includegraphics[width=.94\textwidth]{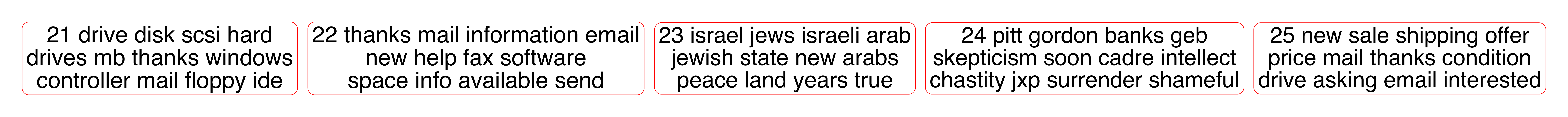}\vspace{-3mm}
   \includegraphics[width=0.925\textwidth]{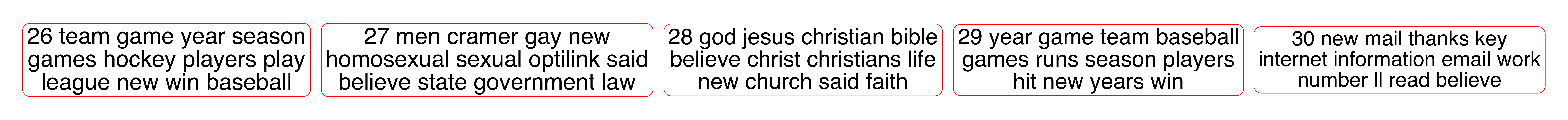}
\end{center}
\vspace{-8mm}
\caption{\small \label{fig:topic_layer5}
The top 30 topics of layer five of the PGBN trained on the 20newsgroups corpus. 
}
\end{figure}

For visualization, we first relabel the nodes at each layer based on their weights $\{r^{(t)}_k\}_{1,K_t}$, calculated as in \eqref{eq:weight}, 
 with a more popular (larger weight) node assigned with a smaller label. We visualize node $k$ of layer $t$ by displaying its top 12 words ranked according to their probabilities in $\big(\prod_{\ell =1}^{t-1} \Phimat^{(\ell )}\big) \phiv_{k}^{(t)}$, the $k$th column of the projected representation calculated as in \eqref{eq:projection}. We set the font size of node $k$ of layer $t$ proportional to 
$\big(r^{(t)}_k/r^{(t)}_1\big)^{\frac 1 {10}}$ in each subplot, and color the outside border of a text box as red, green, orange, blue, or black for a node of layer five, four, three, two, or one, respectively. For better interpretation, we also exclude from the vocabulary the top 30 words of node 1 of layer one: ``don just like people think know time good make way does writes edu ve want say really article use right did things point going better thing need sure used little,'' and the top 20 words of node 2 of layer one: ``edu writes article com apr cs ca just know don like think news cc david university john org wrote world.'' These 50 words are not in the standard list of stopwords but can be considered as stopwords specific to the 20newsgroups corpus discovered by the PGBN. 

For the $[386 ,  63 , 58,  54 ,  51]$ PGBN learned on the 20newsgroups corpus, we plot 54 example topics of layer one in Figure \ref{fig:topic_layer1}, the top 30 topics of layer three in Figure \ref{fig:topic_layer3}, and the top 30 topics of layer five in  Figure \ref{fig:topic_layer5}. 
Figure \ref{fig:topic_layer1} clearly shows that the topics of layer one, except for topics 1-3  that mainly consist of common functional words of the corpus, are all very specific. For example, topics 71 and 81 shown in the first row are about ``candida yeast symptoms'' and ``sex,'' respectively, topics 53, 73, 83, and 84 shown in the second row are about ``printer,'' ``msg,'' ``police radar detector,'' and ``Canadian health care system,'' respectively, and topics 46 and 76 shown in third row are about ``ice hockey'' and ``second amendment,'' respectively. 
By contrast, the topics of layers three and five, shown in Figures~\ref{fig:topic_layer3} and \ref{fig:topic_layer5}, respectively, 
are much less specific and can in general be matched to one or two news groups out of the 20 news groups, including comp.\{graphics, os.ms-windows.misc, sys.ibm.pc.hardware, sys.mac.hardware, windows.x\},
rec.\{autos, motorcycles\}, rec.sport.\{baseball, hockey\},	sci.\{crypt, electronics, med, space\},
misc.forsale, talk. politics.\{misc, guns, mideast\},	and \{talk.religion.misc,
alt.atheism,
soc.religion.christian\}. 

\subsubsection{Visualizing Trees Rooted at The Top-Layer Hidden Units}

\begin{figure}[!tb]
\begin{center}
\!\!\!\!\!\!\!\!\!\!\!\!\! \includegraphics[width=1.0\textwidth]{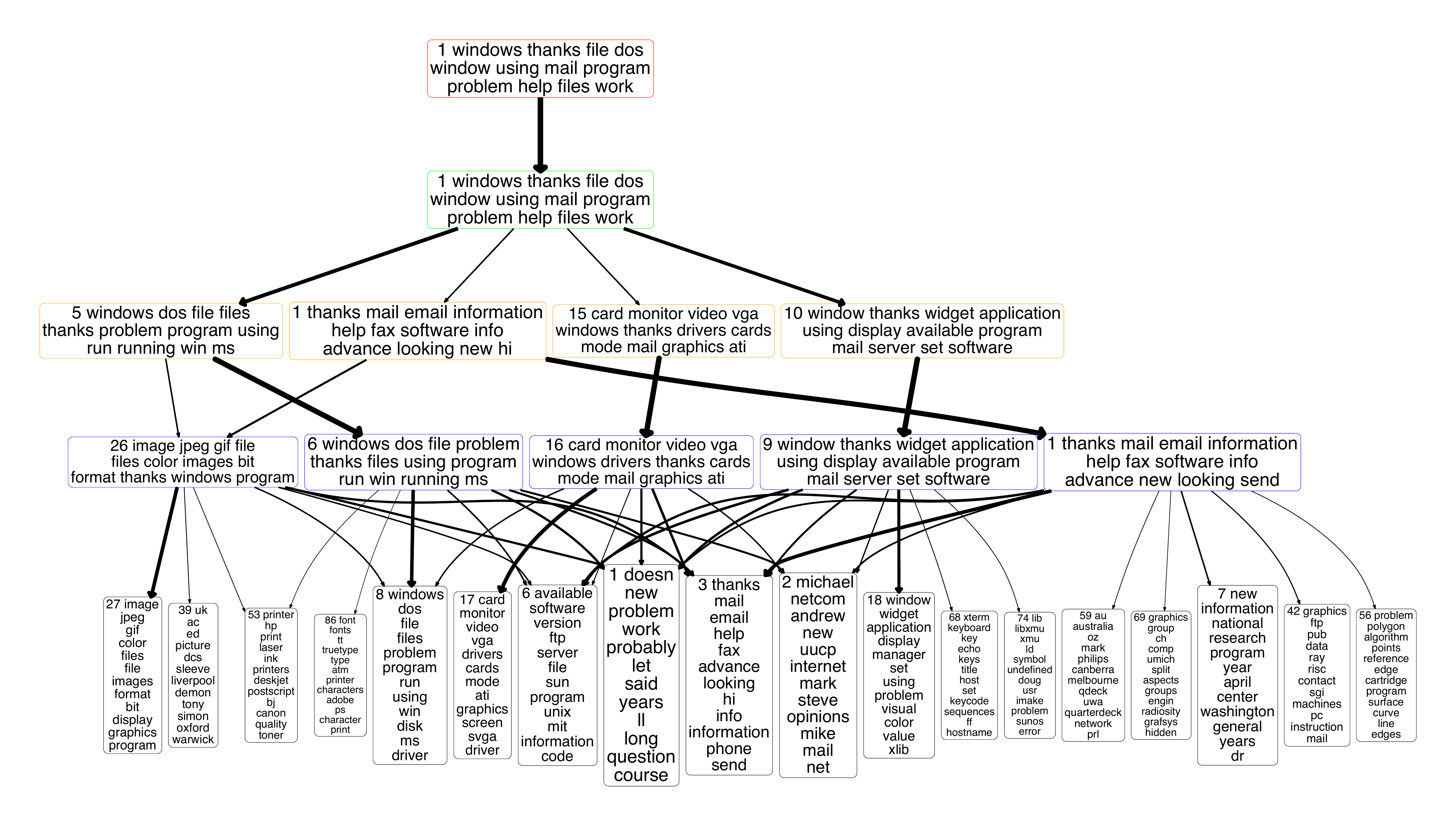}
\end{center}
\vspace{-9mm}
\caption{\small \label{fig:windows}
A $[18,5, 4 ,1 , 1]$ tree that includes all the lower-layer nodes (directly or indirectly) linked with non-negligible weights to the top ranked node of the top layer, 
taken from the full $[386 ,  63 , 58,  54 ,  51]$ network inferred by the PGBN on the 11,269 training documents of the 20newsgroups corpus, with $\eta^{(t)}=0.1$ for all $t$. 
 A line from node $k$ at layer $t$ to node $k'$ at layer $t-1$ indicates that $\Phimat^{(t)}(k',k) > 3/K_{t-1}$, 
with the width of the line 
proportional to 
 $\sqrt{\Phimat^{(t)}(k',k)}$.
For each node, the rank (in terms of popularity) at the corresponding layer and the top 12 words of the corresponding topic are displayed inside the text box, where the text font size monotonically decreases as the popularity of the node decreases, and the outside border of the text box is colored as red, green, orange, blue, or black if the node is at layer five, four, three, two, or one, respectively. 
}
\end{figure}

\begin{figure}[!tb]
\begin{center}
\!\!\!\!\!\!\!\!\!\!\!\!\! \includegraphics[width=1.0\textwidth]{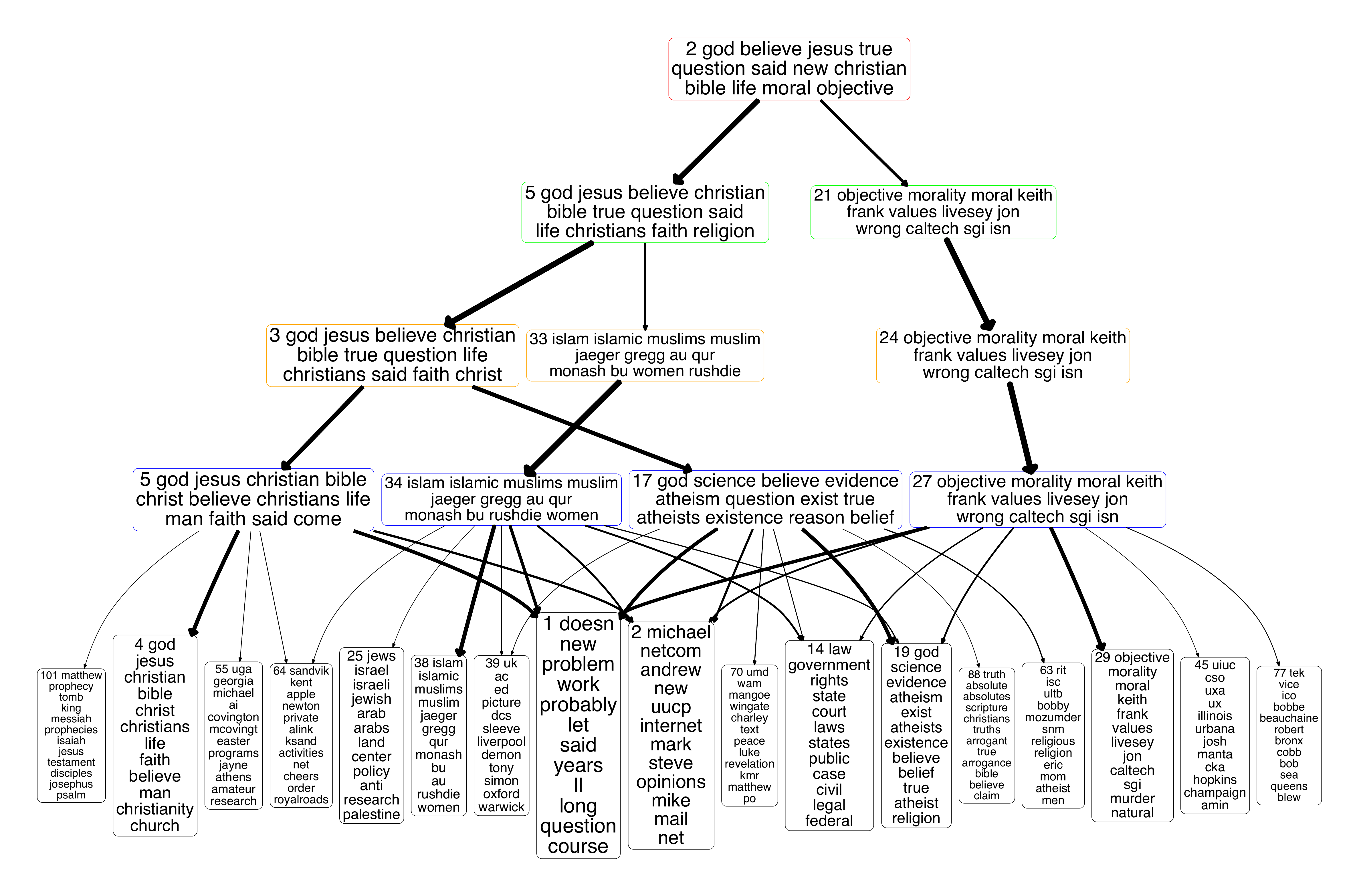}
\end{center}
\vspace{-9mm}
\caption{\small \label{fig:god}
Analogous plot to Figure \ref{fig:windows} for a tree on ``religion,'' rooted at node 2 of the top-layer. 
}
\end{figure}

 While it is interesting to examine the topics of different layers to understand the general and specific aspects of the corpus used to train the PGBN, it would be more informative to further illustrate how the topics of different layers are related to each other. Thus we consider constructing trees 
 to visualize the PGBN. 
We first pick a node as the root of a tree and grow the tree downward by drawing a line from node $k$ at layer $t$, the root or a leaf node of the tree, 
 to node $k'$ at layer $t-1$ for all $k'$ in the set $\{ k': \Phimat^{(t)}(k',k) > \tau_t/K_{t-1}\}$, where we set the width of the line connecting node $k$ of layer $t$ to node $k'$ of layer $t-1$ 
be proportional to $\sqrt{\Phimat^{(t)}(k',k)}$ and use $\tau_t$ to adjust the complexity of a tree. In general, increasing $\tau_t$ would discard more weak connections and hence make the tree simpler and easier to visualize. 

We set $\tau_t=3$ for all $t$ to visualize both a five-layer tree rooted at the top ranked node of the top hidden layer, as shown in Figure \ref{fig:windows}, and a five-layer tree rooted at the second ranked node of the top hidden layer, as shown in Figure \ref{fig:god}.
For the tree in Figure \ref{fig:windows}, while it is somewhat vague to determine the actual meanings of both node 1 of layer five and node 1 of layer four based on their top words, examining the more specific topics of layers three and two within the tree clearly indicate that this tree 
is primarily about ``windows,'' ``window system,'' ``graphics,'' ``information,'' and ``software,'' which are relatively specific concepts that are all closely related to each other. The similarities and differences between the five nodes of layer two can be further understood by examining the nodes of layer one that are connected to them. For example, while nodes 26 and 16 of layer two share their connections to multiple nodes of layer one, node 27 of layer one on ``image'' is strongly connected to node 26 of layer two but not to node 16 of layer two, and node 17 of layer one on ``video'' is strongly connected to node 16 of layer two but not to node 26 of layer two. 

Following the branches of each tree shown in both figures, 
it is clear that 
the topics become more and more specific when moving along the tree from the top to bottom. Taking the tree on ``religion'' shown in Figure \ref{fig:god} for example, the root node splits into two nodes when moving from layers five to four: while the left node is still mainly about ``religion,'' the right node is on ``objective morality.'' When moving from layers four to three, node 5 of layer four splits into a node about ``Christian'' and another node about ``Islamic.'' When moving from layers three to two, node 3 of layer three splits into a node about ``God, Jesus, \& Christian,'' and another node about ``science, atheism, \& question of the existence of God.'' 
When moving from layers two to one, all four nodes of layer two split into multiple topics, and they are all strongly connected to both topics 1 and 2 of layer one, whose top words are those that appear frequently in the 20newsgroups corpus. 

\subsubsection{Visualizing Subnetworks Consisting of Related Trees}

Examining the top-layer topics shown in Figure \ref{fig:topic_layer5}, one may find that some of the nodes seem to be closely related to each other. For example, topics 3 and 11 share eleven words out of the top twelve ones; 
topics 15 and 23 both have ``Israel'' and ``Jews'' as their top two words; topics 16 and 18 are both related to ``gun;'' and topics 7, 13, and 26 all share ``team(s),'' ``game(s),'' ``player(s),'' ``season,'' and ``league.''

To understand the relationships and distinctions between these related nodes, we construct subnetworks that include the trees rooted at them, as shown in Figures \ref{fig:car}-\ref{fig:sports} in Appendix \ref{sec:fig}. It is clear from Figure \ref{fig:car} that the top-layer topic 3 differs from topic 11 in that it is not only strongly connected to topic 2 of layer four on``car \& bike,'' but also has a non-negligible connection to topic 27 of layer four on ``sales.'' It is clear from Figure \ref{fig:mideast} that topic 15 differs from topic 23 in that it is not only about ``Israel \& Arabs,'' but also about ``Israel, Armenia, \& Turkey.'' It is clear from Figure \ref{fig:gun} in that topic 16 differs from topic 18 in that it is mainly about Waco siege happened in 1993 involving David Koresh, the Federal Bureau of Investigation (FBI), and the Bureau of Alcohol, Tobacco, Firearms and Explosives (BATF).
 It is clear from Figure \ref{fig:sports} that topics 7 and 13 are mainly about ``ice hockey'' and ``baseball,'' respectively, and topic 26 is a mixture of both.

\subsubsection{Capturing Correlations Between Nodes}
For the augmentable GBN, as in (\ref{eq:x_given_theta}), given the weight vector $\thetav^{(1)}_{j}$, we have
\beq
\E\big[\xv^{(1)}_{j}\,\big|\,\Phimat^{(1)},\thetav^{(1)}_{j} \big] = \Phimat^{(1)}\thetav^{(1)}_{j}.
\eeq
A distinction between a shallow augmentable  GBN with $T=1$ hidden layer and a deep augmentable  GBN with $T\ge 2$ hidden layers is that the prior for $\thetav^{(1)}_{j}$ changes from 
$\thetav^{(1)}_{j}\sim\mbox{Gam}(\rv,1/c^{(2)}_{j})$ for $T=1$ to $\thetav^{(1)}_{j}\sim\mbox{Gam}(\Phimat^{(2)}\thetav^{(2)}_{j},1/c^{(2)}_{j})$ for $T\ge 2$.
For the GBN with $T=1$, given the shared weight vector $\rv$, we have
\beq
\E\big[\xv_j^{(1)}\,\big|\, \Phimat^{(1)},\rv \big] = \Phimat^{(1 )}\rv/ c_j^{(2)};\label{eq:x_given_theta}
\eeq
for the GBN with $T=2$, given the shared weight vector $\rv$, we have
\beq
\E\big[\xv_j^{(1)}\,\big|\, \Phimat^{(1)},\Phimat^{(2)},\rv \big] = \Phimat^{(1 )}\Phimat^{(2)}\rv\Big / \left(c_j^{(2)}c_j^{(3)}\right);\label{eq:x_given_theta}
\eeq
and for the GBN with $T\ge 2$,  given the weight vector $\thetav^{(2)}_{j}$, we have
\beq
\E\big[\xv_j^{(1)}\,\big|\,\Phimat^{(1)}, \Phimat^{(2)},\thetav_j^{(2)} \big] = \Phimat^{(1 )}\Phimat^{(2)} {\thetav_j^{(2)}}/ c_j^{(2)}.\label{eq:x_given_theta}
\eeq
Thus in the prior, the co-occurrence patterns of the columns of $\Phimat^{(1)}$ are modeled by only a single vector $\rv$ 
when $T=1$, but are captured in the columns of $\Phimat^{(2)}$ when $T\ge 2$. 
Similarly, in the prior, if $T\ge t+1$, the co-occurrence patterns of the $K_t$ columns of the projected topics $\prod_{\ell =1}^t \Phimat^{(\ell )}$ will be captured in the columns of the $ K_t \times K_{t+1}$ matrix $\Phimat^{(t+1)}$.

To be more specific, we show in Figure \ref{fig:turkey} in Appendix \ref{sec:fig} three example trees rooted at three different nodes of layer three, where we lower the threshold to $\tau_t=1$ to reveal more weak links between the nodes of adjacent layers. 
The top subplot reveals that, in addition to strongly co-occurring with the top two topics of layer one,
 topic 21 
 of layer one on ``medicine'' tends to co-occur not only with topics 7, 21, and 26, which are all common topics that frequently appear, but also with some much less common topics that are related to very specific diseases or symptoms, such as topic 67 on ``msg'' and ``Chinese restaurant syndrome,'' topic 73 on ``candida yeast symptoms,'' and topic 180 on ``acidophilous'' and ``astemizole (hismanal).'' 
 
 The middle subplot reveals that topic 31 of layer two on ``encryption \& cryptography'' tends to co-occur with topic 13 of layer two on ``government \& encryption,'' and it also indicates that topic 31 of layer one is more purely about ``encryption'' and more isolated from ``government'' in comparison to the other topics of layer one. 
 
 The bottom subplot reveals that in layer one, topic 14 on ``law \& government,'' topic 32 on ``Israel \& Lebanon,'' topic 34 on ``Turkey, Armenia, Soviet Union, \& Russian,'' topic 132 on ``Greece, Turkey, \& Cyprus,'' topic 98 on ``Bosnia, Serbs, \& Muslims,'' topic 143 on ``Armenia, Azeris, Cyprus, Turkey, \& Karabakh,'' 
 and several other very specific topics related to Turkey and/or Armenia all tend to co-occur with each other. 

We note that capturing the co-occurrence patterns between the topics not only helps exploratory data analysis, but also helps extract better features for classification in an unsupervised manner and improves prediction for held-out data, as will be demonstrated in detail in Section~\ref{sec:examples}.

\subsection{Related Models}
The structure of the augmentable  GBN  resembles the sigmoid belief network and  recently proposed deep exponential family model \citep{ranganath2014deep}. Such kind of gamma distribution based network and its inference procedure were vaguely hinted in Corollary~2 of \citet{NBP2012}, and had been exploited by \citet{GP_DFA_AISTATS2015} to develop a gamma Markov chain to model the temporal evolution of the factor scores of a dynamic count matrix, but have not yet been investigated for extracting multilayer data representations.  The proposed augmentable GBN may also be considered as an exponential family harmonium \citep{welling2004exponential,xing2005mining}.



 \subsubsection{Sigmoid and Deep Belief Networks}
Under the hierarchical  model in (\ref{eq:PGBN}), given the connection weight matrices, 
 the joint distribution of the observed/latent counts and gamma hidden units of the GBN can be expressed, similar to those of the sigmoid and deep belief networks \citep{Bengio-et-al-2015-Book}, as
 \beq
P\left(\xv_j^{(1)},\{\thetav_{j}^{(t)}\}_t \,\Big |\,\{\Phimat^{(t)}\}_t\right) = P\left(\xv_j^{(1)}\,\Big|\,\Phimat^{(1)},\thetav_{j}^{(1)}\right) \left[\prod_{t=1}^{T-1}P\left(\thetav_{j}^{(t)}\,\Big|\,\Phimat^{(t+1)},\thetav_{j}^{(t+1)}\right)\right] P\left(\thetav_{j}^{(T)}\right).\notag
\eeq
With $\phiv_{v:}$ representing the $v$th row $\Phimat$, for the gamma hidden units $\theta_{vj}^{(t)}$ we have 
\beq
P\left(\theta_{vj}^{(t)}\,\Big |\,\phiv_{v:}^{(t+1)},\thetav_j^{(t+1)},c_{j+1}^{(t+1)}\right) =\frac{ \left(c_{j+1}^{(t+1)}\right)^{\phiv_{v:}^{(t+1)}\thetav_j^{(t+1)} }}{\Gamma\left(\phiv_{v:}^{(t+1)}\thetav_j^{(t+1)}\right)} \left(\theta_{vj}^{(t)}\right)^{\phiv_{v:}^{(t+1)}\thetav_j^{(t+1)}-1}e^{-c_{j+1}^{(t+1)} \theta_{vj}^{(t)} }, \label{eq:gamma}
\eeq
which are highly nonlinear functions that are strongly desired in deep learning. 
By contrast, with the sigmoid function $\sigma(x)=1/(1+e^{-x})$ and bias terms $b_{v}^{(t+1)}$, a sigmoid/deep belief network would connect the binary hidden units $\theta_{vj}^{(t)}\in\{0,1\}$ of layer $t$ (for deep belief networks, $t<T-1$ ) to the product of the connection weights and binary hidden units of the next layer with
\beq
P\left(\theta_{vj}^{(t)}=1\,\Big|\,\phiv_{v:}^{(t+1)},\thetav_j^{(t+1)},b_{v}^{(t+1)}\right) =\sigma\left(b_{v}^{(t+1)}+\phiv_{v:}^{(t+1)}\thetav_j^{(t+1)}\right). \label{eq:sigmoid}
\eeq
Comparing (\ref{eq:gamma}) with (\ref{eq:sigmoid}) clearly shows the distinctions between 
 the gamma distributed nonnegative hidden units and the sigmoid link function based binary hidden units. The limitation of binary units in capturing the approximately linear data structure  over small ranges is a key motivation for \citet{frey1999variational} to investigate nonlinear Gaussian belief networks with real-valued  units. 
 As a new alternative to binary units, it would be interesting to further investigate whether the gamma distributed nonnegative real units 
 can in theory carry richer information and model more complex nonlinearities given the same network structure. 
 Note that the rectified linear units have emerged as powerful alternatives of sigmoid units to introduce nonlinearity \citep{nair2010rectified}. It would be interesting to investigate whether the gamma units can be used to 
 introduce nonlinearity into the positive region of the rectified linear units.

\subsubsection{Deep Poisson Factor Analysis} 
With $T=1$, the PGBN  specified by (\ref{eq:PGBN})-(\ref{eq:c_j}) and (\ref{eq:PFA})
reduces to Poisson factor analysis (PFA) using the (truncated) gamma-negative binomial process \citep{NBP2012}, with a truncation level of $K_1$. 
As discussed in \citep{BNBP_PFA_AISTATS2012, NBP2012}, with priors imposed on neither $\phiv_k^{(1)}$ nor $\thetav_j^{(1)}$, PFA is related to nonnegative matrix factorization 
\citep{NMF}, 
 and with the Dirichlet priors imposed on both $\phiv_k^{(1)}$ and $\thetav_j^{(1)}$, PFA is related to latent Dirichlet allocation \citep{LDA}. 
 
 Related to the PGBN and the dynamic model in \citep{GP_DFA_AISTATS2015}, the deep exponential family model of \citet{ranganath2014deep} also considers a gamma chain under Poisson observations, but it is the gamma scale parameters that are chained and factorized, which allows learning the network parameters using black box variational inference \citep{ranganath2014black}. 
In the proposed PGBN, we chain the gamma random variables via the gamma shape parameters. Both strategies worth through investigation. We prefer chaining the shape parameters in this paper, which leads to efficient upward-downward Gibbs sampling via data augmentation and makes it clear how the latent counts are propagated across layers, as discussed in detail in the following sections. The sigmoid belief network has also been recently incorporated into PFA for deep factorization of count data \citep{Gan2015DeepPFA}, however, that deep structure captures only the correlations between binary factor usage patterns but not the full connection weights. 
 In addition, neither \citet{ranganath2014deep} nor \citet{Gan2015DeepPFA} provide a principled way to learn the network structure, whereas the proposed GBN 
uses the gamma-negative binomial process together with a greedy layer-wise training strategy 
to automatically infer the widths of the hidden layers, 
which will be described in Section \ref{sec:greedy}.

\subsubsection{Correlated and Tree-Structured Topic Models}
The PGBN with $T=2$ can also be related to correlated topic models \citep{CTM,DILN_BA,chen2013scalable,CorrRandomMeasures,linderman2015dependent}, which typically use the logistic normal distributions to replace the topic-proportion Dirichlet distributions used in latent Dirichlet allocation \citep{LDA}, capturing the co-occurrence patterns between the topics in the latent Gaussian space using a covariance matrix. By contrast, the PGBN factorizes the topic usage weights (not proportions) under the gamma likelihood, capturing the co-occurrence patterns between the topics of the first layer (i.e., the columns of $\Phimat^{(1)}$) in the columns of $\Phimat^{(2)}$, the latent weight matrix connecting the hidden units of layers two and one. 
For the PGBN, the 
 computation does not involve matrix inversion, which is  often necessary for correlated topic models without specially structured covariance matrices, and scales linearly 
 with  the number of topics, hence it is suitable to be used to capture the correlations between hundreds of or thousands of topics.



 As in Figures \ref{fig:windows}, \ref{fig:god}, and \ref{fig:car}-\ref{fig:turkey}, trees and subnetworks can be extracted from the inferred deep network to visualize the data. Tree-structured topic models have also been proposed before, such as those in \citet{nCRP}, \citet{treeSBP}, and \citet{nHDP}, but they usually artificially impose the tree structures to be learned, whereas the PGBN learns a directed network, from which trees and subnetworks can be extracted for visualization, 
 without the need to specify the number of nodes per layer, restrict the number of branches per node, and forbid a node to have multiple parents. 

\section{Model Properties and Inference}
Inference for the GBN shown in \eqref{eq:PGBN} appears challenging,  because not only  the conjugate prior is unknown for the shape parameter of a gamma distribution, but also the gradients are difficult to evaluate for the parameters of the (log) gamma probability density function, which, as  in  \eqref{eq:gamma}, includes the parameters inside the (log) gamma function. To address these challenges, we consider data augmentation \citep{DataAugment} that introduces auxiliary variables to make it simple to compute the conditional posteriors of  model parameters via the joint distribution of the auxiliary and existing random variables. 
We will first show that 
each gamma hidden unit can be linked to a Poisson distributed latent count variable, leading to a negative binomial  likelihood for the parameters of the gamma hidden unit if it is margined out from the Poisson distribution; we then 
introduce an auxiliary  count variable, which is sampled from the CRT distribution parametrized by the negative binomial latent  count and shape parameter,  to make the joint likelihood of the auxiliary CRT count and latent negative binomial count given the parameters of the gamma hidden unit amenable to posterior simulation. 
More specifically,
under the proposed augmentation scheme, the gamma shape parameters will be linked to auxiliary counts under the Poisson likelihoods, making it straightforward for posterior simulation, as described below in detail. 

\subsection{The Upward Propagation of Latent Counts}
We
 break the inference of the GBN of $T$ hidden layers into $T$ related subproblems, each of which is solved with the same subroutine. Thus for implementation, it is straightforward for the GBN to adjust its depth $T$. Let us denote  $\xv_j^{(t)}\in\mathbb{Z}^{K_{t-1}}$ as the observed or  latent count vector of layer $t\in\{1,\ldots,T\}$, and $x_{vj}^{(t)}$ as its $v$th element, where  $v\in\{1,\ldots,K_{t-1}\}$. 


%
%
\begin{lem} [Augment-and-Conquer The Gamma Belief Network]\label{lem:PGBN}
With $p_j^{(1)}: = 1-e^{-1}$ 
and
\beq
p_{j}^{(t+1)} := {-\ln(1-p_j^{(t)})}\Big/\left[c_j^{(t+1)}-\ln(1-p_j^{(t)})\right] \label{eq:p}
\eeq
for $t=1,\ldots,T$, one may connect the observed or latent counts $\xv_j^{(t)}\in\mathbb{Z}^{K_{t-1}}$ to the product $\Phimat^{(t)}\thetav_j^{(t)}$ at layer $t$ under the Poisson likelihood as 
\beq
\xv_j^{(t)}\sim\emph{\mbox{Pois}}\left[-\Phimat^{(t)}\thetav_j^{(t)}\ln\left(1-p_j^{(t)}\right)\right].\label{eq:deepPFA_aug}
\eeq
\end{lem}

\begin{proof}
By definition (\ref{eq:deepPFA_aug}) is true for layer $t=1$.
Suppose that (\ref{eq:deepPFA_aug}) is also true for layer $t>1$,
then we can augment each count $x^{(t)}_{vj}$, where $v\in\{1,\ldots,K_{t-1}\}$, into the summation of $K_{t}$ latent counts, which are smaller than or equal to $x^{(t)}_{vj}$ as
\beq 
x^{(t)}_{vj}=\sum_{k=1}^{K_{t}} x^{(t)}_{vjk},~~x^{(t)}_{vjk}\sim\mbox{Pois}\left[-\phi_{vk}^{(t)}\theta_{kj}^{(t)}\ln\left(1-p_j^{(t)}\right)\right].\label{eq:PoAug}
\eeq
Let the $\cdotv$ symbol represent summing over the corresponding index and let 
$$m^{(t)(t+1)}_{kj}:=x^{(t)}_{\cdotv jk} := \sum_{v=1}^{K_{t-1}}x^{(t)}_{vjk}$$ represent the number of times that factor $k\in\{1,\ldots,K_t\}$ of layer $t$ appears in observation~$j$ and $\mv^{(t)(t+1)}_{j}: = \left(x^{(t)}_{\cdotv j1},\ldots,x^{(t)}_{\cdotv jK_{t}}\right)'$. Since $\sum_{v=1}^{K_{t-1}}\phi_{vk}^{(t)} = 1$, we can marginalize out $\Phimat^{(t)}$ as in \citep{BNBP_PFA_AISTATS2012}, leading to 
\beq
\mv^{(t)(t+1)}_{j}\sim\mbox{Pois}\left[-\thetav_j^{(t)}\ln\left(1-p_j^{(t)}\right)\right]. \notag
\eeq
Further marginalizing out the gamma distributed $\thetav_j^{(t)}$ from the  Poisson likelihood leads~to
\beq
\mv^{(t)(t+1)}_{j} \sim\mbox{NB}\left(\Phimat^{(t+1)}\thetav_j^{(t+1)}, p_j^{(t+1)}\right). \label{eq:NBAug}
\eeq
Element $k$ of $\mv^{(t)(t+1)}_{j}$ can be augmented under its compound Poisson representation as 
\beq
m^{(t)(t+1)}_{ kj} = \sum_{\ell=1}^{x^{(t+1)}_{kj}} u_{\ell},~~u_{\ell}\sim\mbox{Log}(p_j^{(t+1)}), ~~
x_{kj}^{(t+1)}\sim\mbox{Pois}\left[-\phiv_{k:}^{(t+1)}\thetav_j^{(t+1)}\ln\left(1-p_j^{(t+1)}\right)\right]. \notag
\eeq
Thus if (\ref{eq:deepPFA_aug}) is true for layer $t$, then 
it is also true for layer $t+1$. 
\end{proof}

\begin{cor} [Propagate the latent counts upward] \label{cor:PGBN} Using Lemma 4.1 of \citep{BNBP_PFA_AISTATS2012} on (\ref{eq:PoAug}) and Theorem 1 of \citep{NBP2012} on (\ref{eq:NBAug}), we can propagate the latent counts $x^{(t)}_{vj} $ of layer \emph{$t$} upward to layer \emph{$t+1$} as
\begin{align}
&
\left\{\left(x^{(t)}_{vj1},\ldots,x^{(t)}_{vjK_{t}}\right)\,\Big| \, x^{(t)}_{vj}, \phiv_{v:}^{(t)}, \thetav_j^{(t)}\right\}\sim\emph{\mbox{Mult}}\left(x^{(t)}_{vj}, \frac{\phi^{(t)}_{v1}\theta^{(t)}_{1j}}{\sum_{k=1}^{K_{t}}\phi^{(t)}_{vk}\theta^{(t)}_{kj}},\ldots,\frac{\phi^{(t)}_{vK_{t}}\theta^{(t)}_{K_{t}j}}{\sum_{k=1}^{K_{t}}\phi^{(t)}_{vk}\theta^{(t)}_{kj}}\right)\!,\! \label{eq:step1}\\
&~~~\left( \left.x^{(t+1)}_{kj} \, \right | \,m^{(t)(t+1)}_{kj}, \phiv_{k:}^{(t+1)}, \thetav_j^{(t+1)}\right) \sim \emph{\mbox{CRT}}\left(m^{(t)(t+1)}_{kj}, \phiv_{k:}^{(t+1)}\thetav_{j}^{(t+1)}\right).\label{eq:CRT}
\end{align}
\end{cor}

We provide a set of graphical representations in Figure \ref{fig:graphical_model} to describe the GBN model and illustrate the augment-and-conquer inference scheme. We provide the upward-downward Gibbs sampler in Appendix \ref{sec:sampling}.

\begin{figure}[t]
\centering

\subfigure[]{\includegraphics[width=0.20\textwidth]{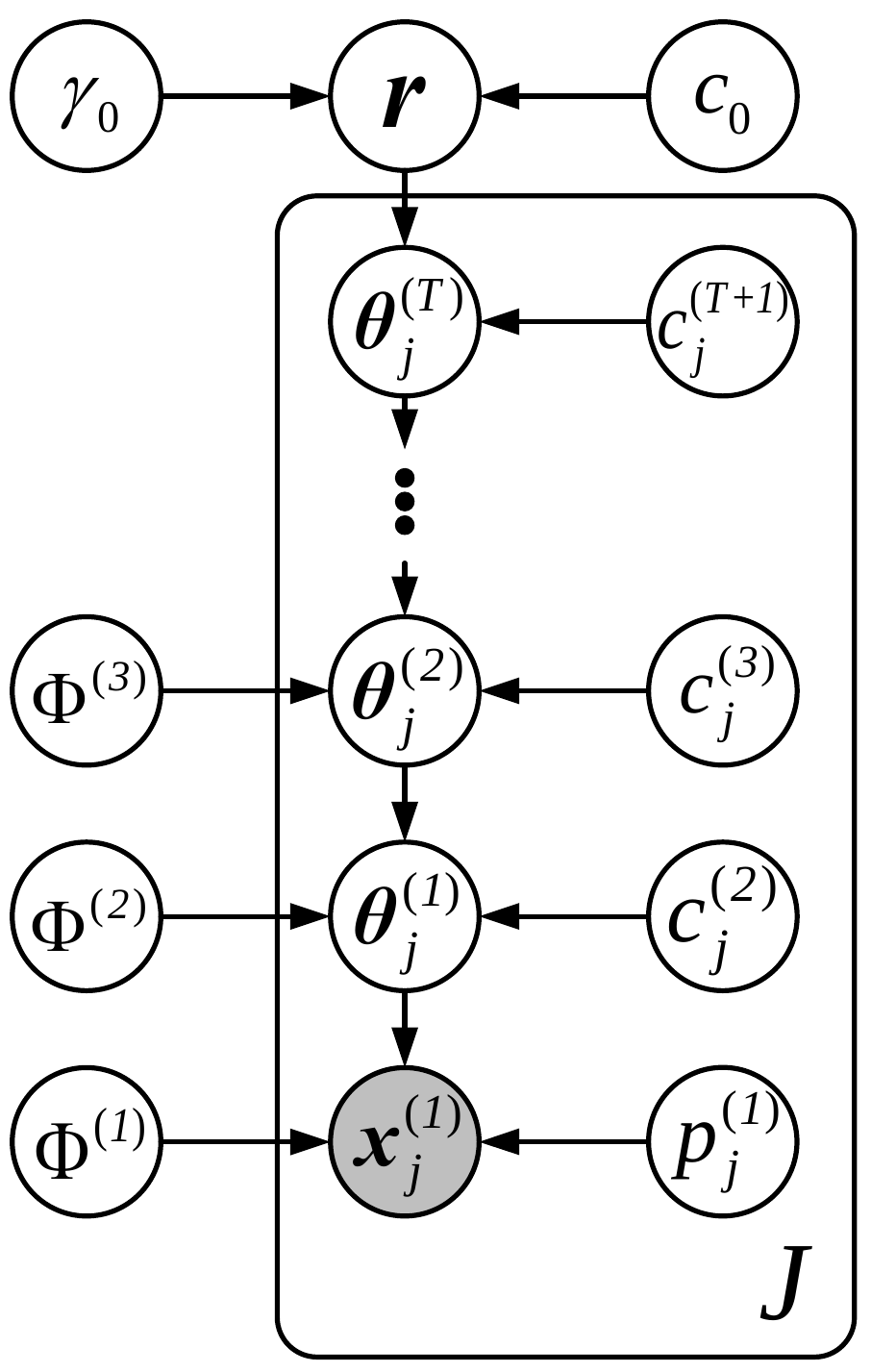}}
\subfigure[]{\includegraphics[width=0.20\textwidth]{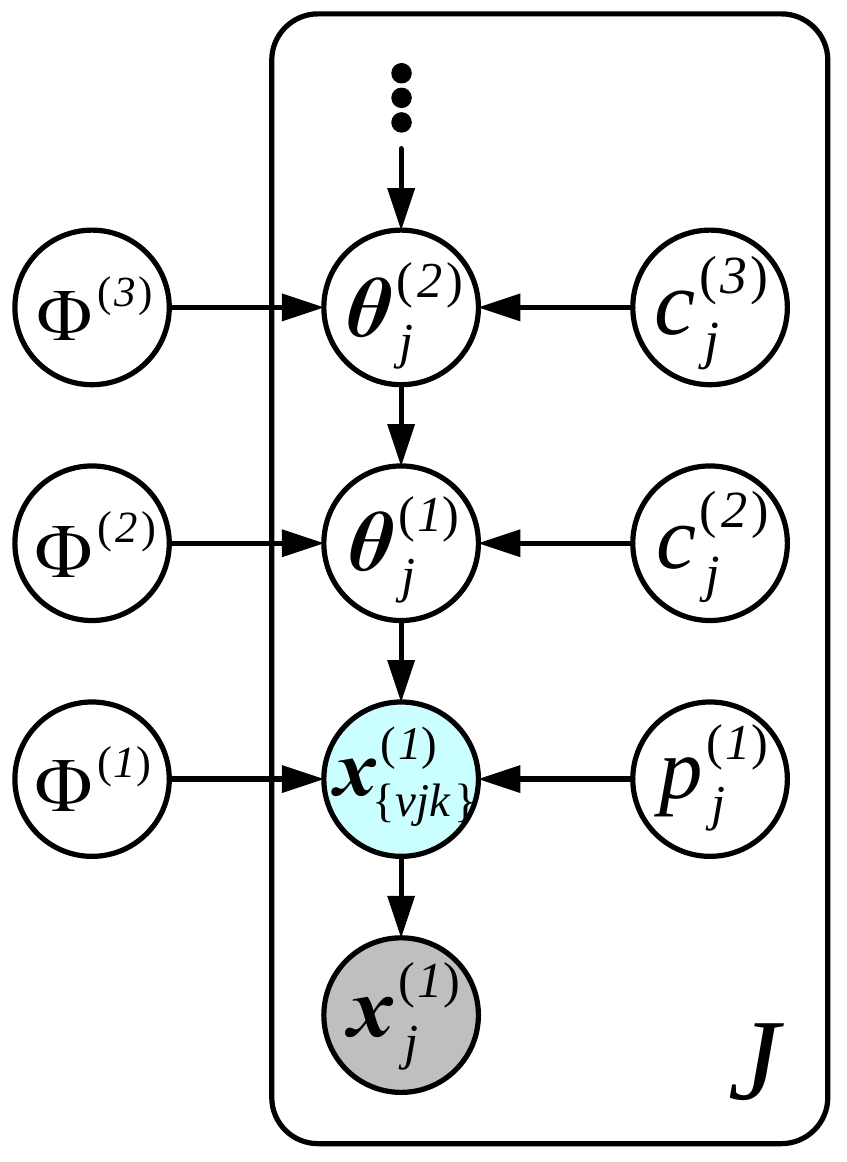}}
\subfigure[]{\includegraphics[width=0.20\textwidth]{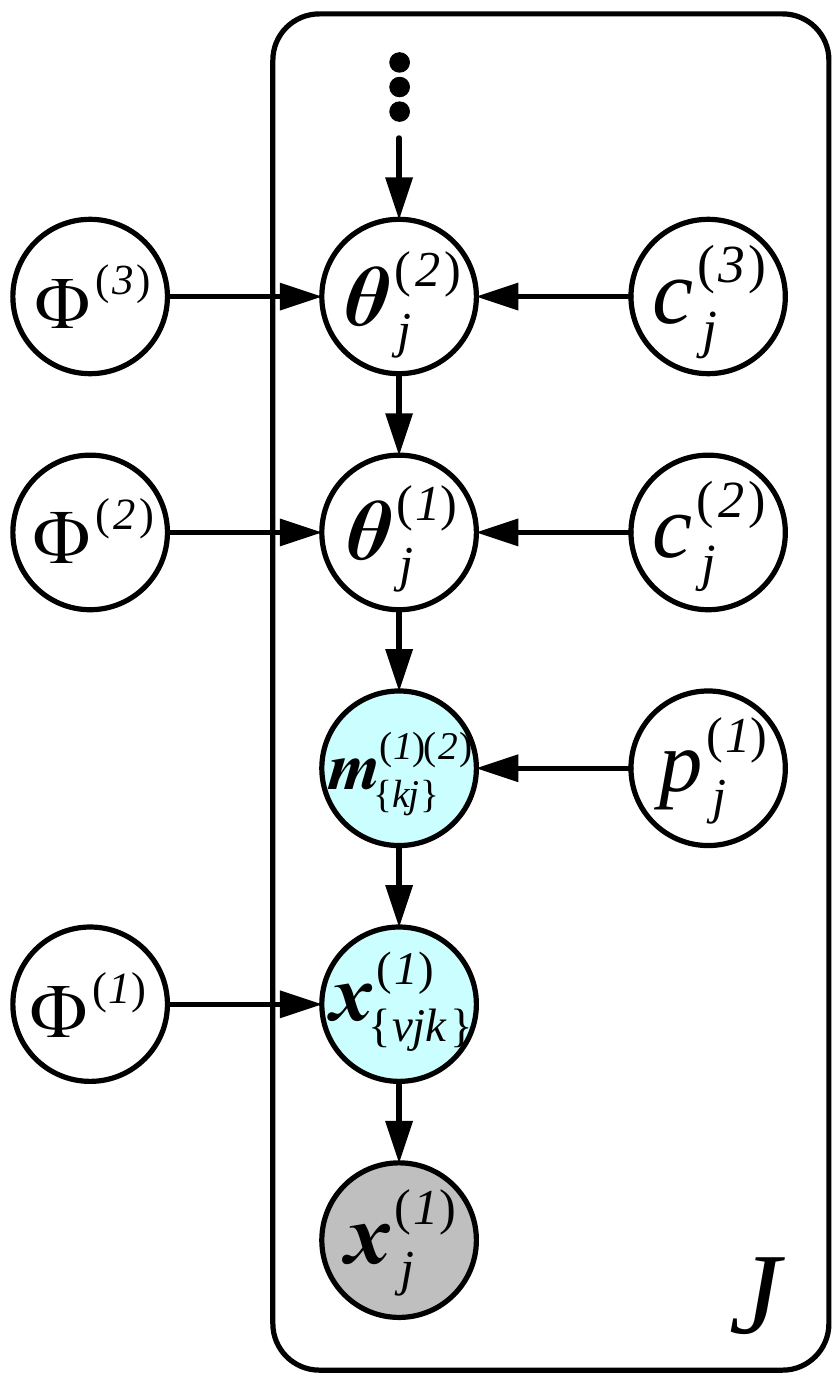}}
\subfigure[]{\includegraphics[width=0.20\textwidth]{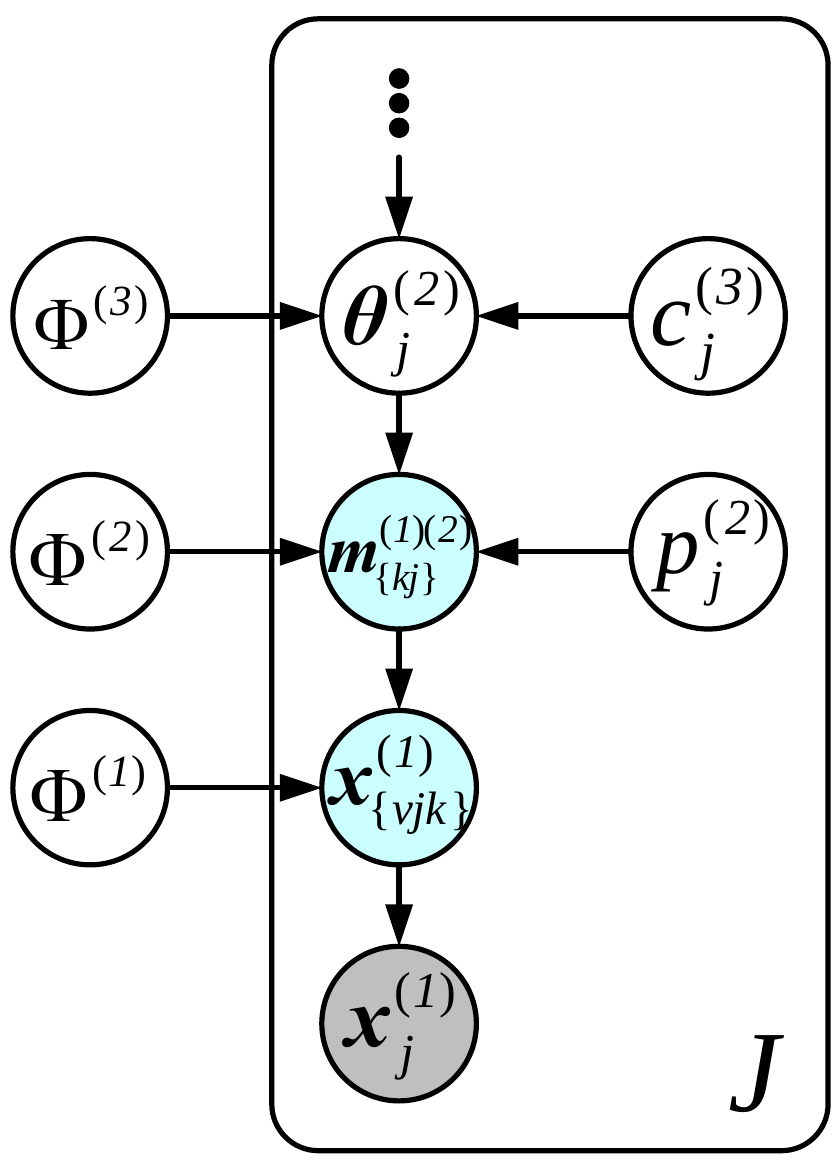}}\\
\subfigure[]{\includegraphics[width=0.20\textwidth]{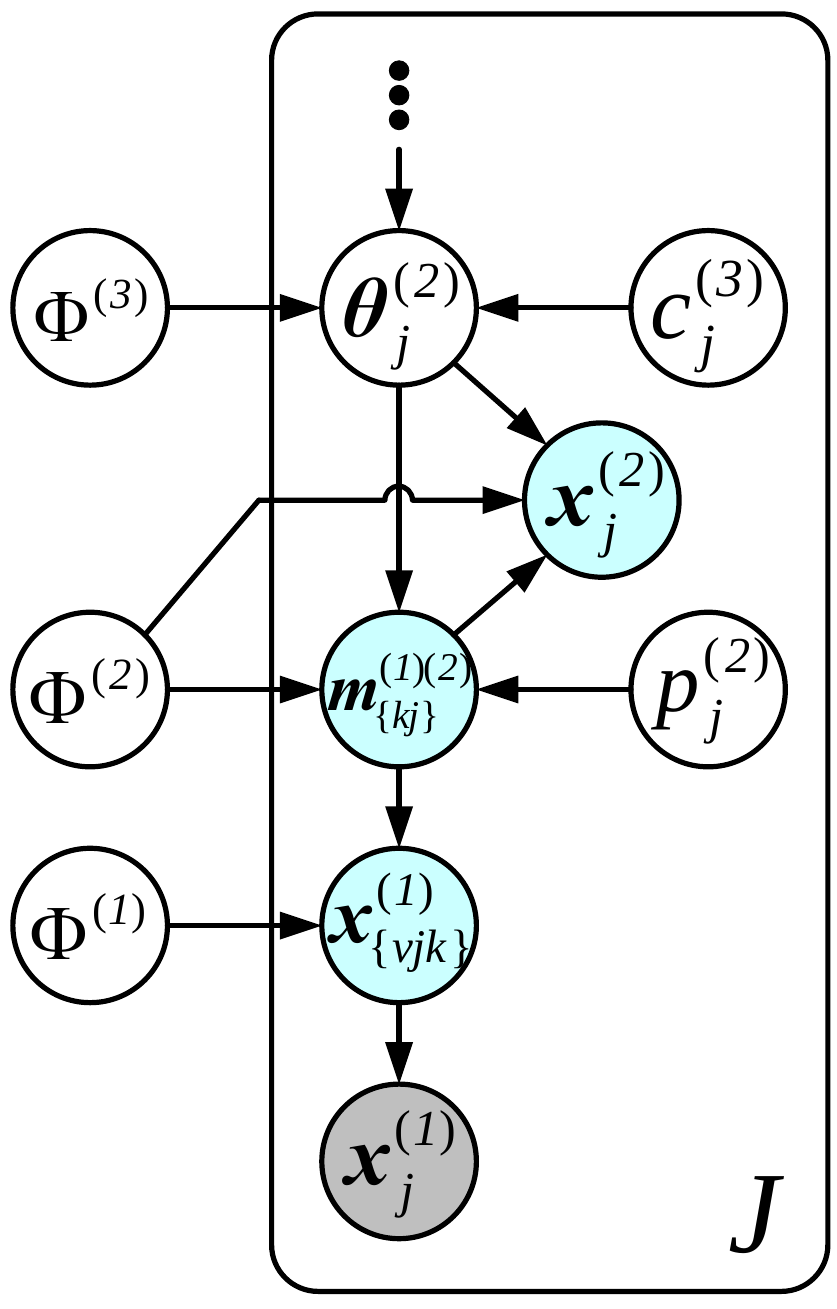}}
\subfigure[]{\includegraphics[width=0.20\textwidth]{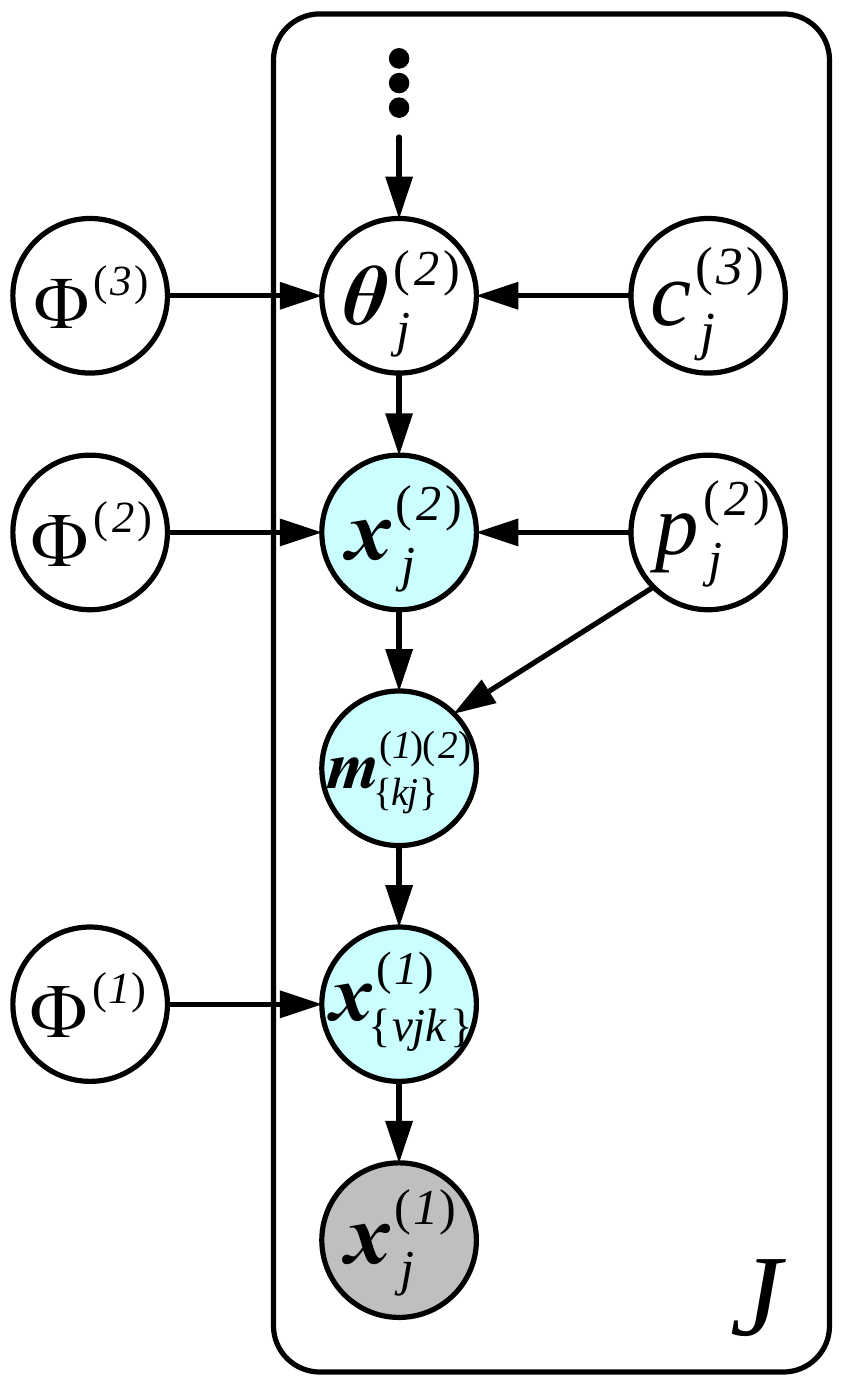}}
\subfigure[]{\includegraphics[width=0.20\textwidth]{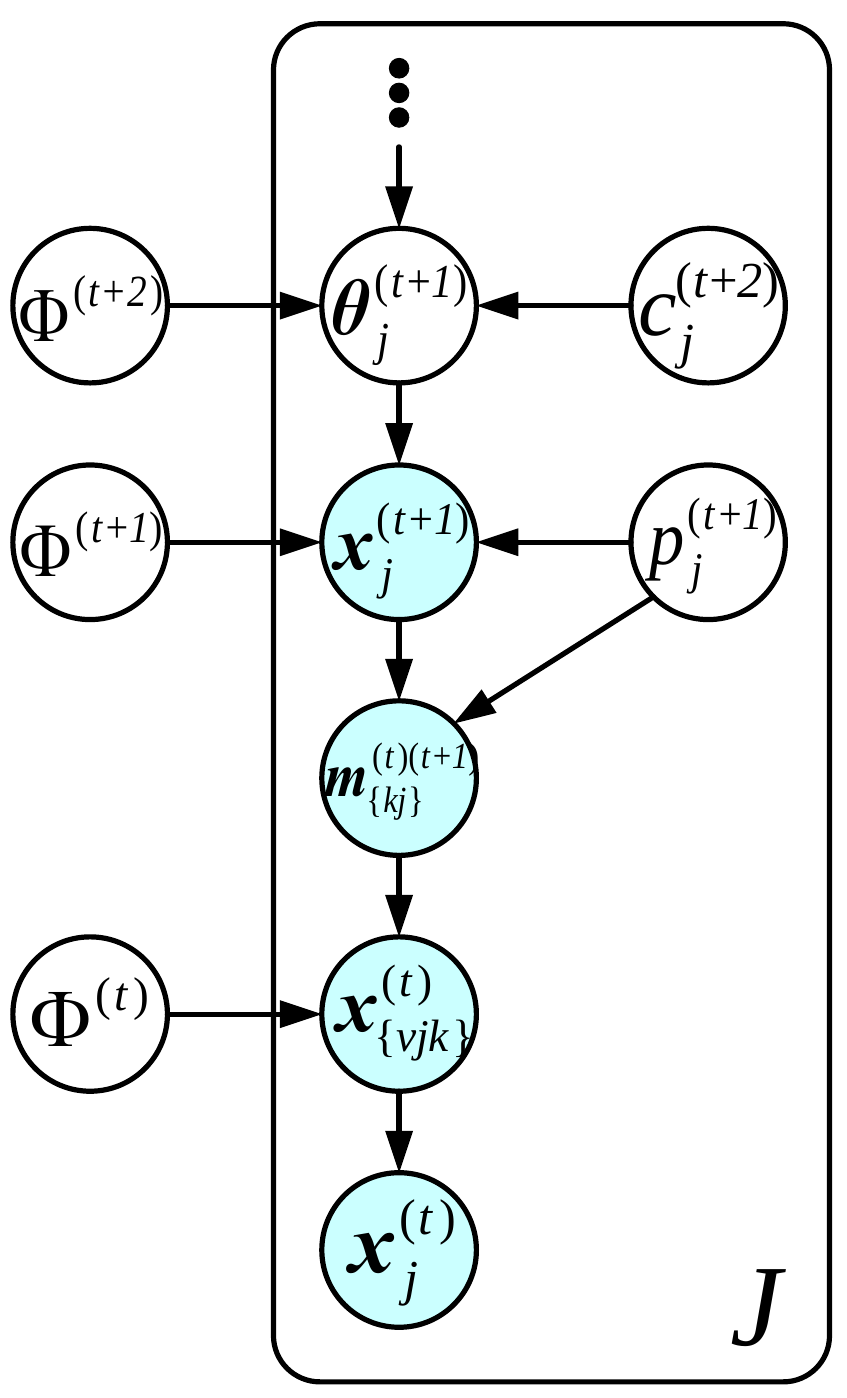}}
\subfigure[]{\includegraphics[width=0.20\textwidth]{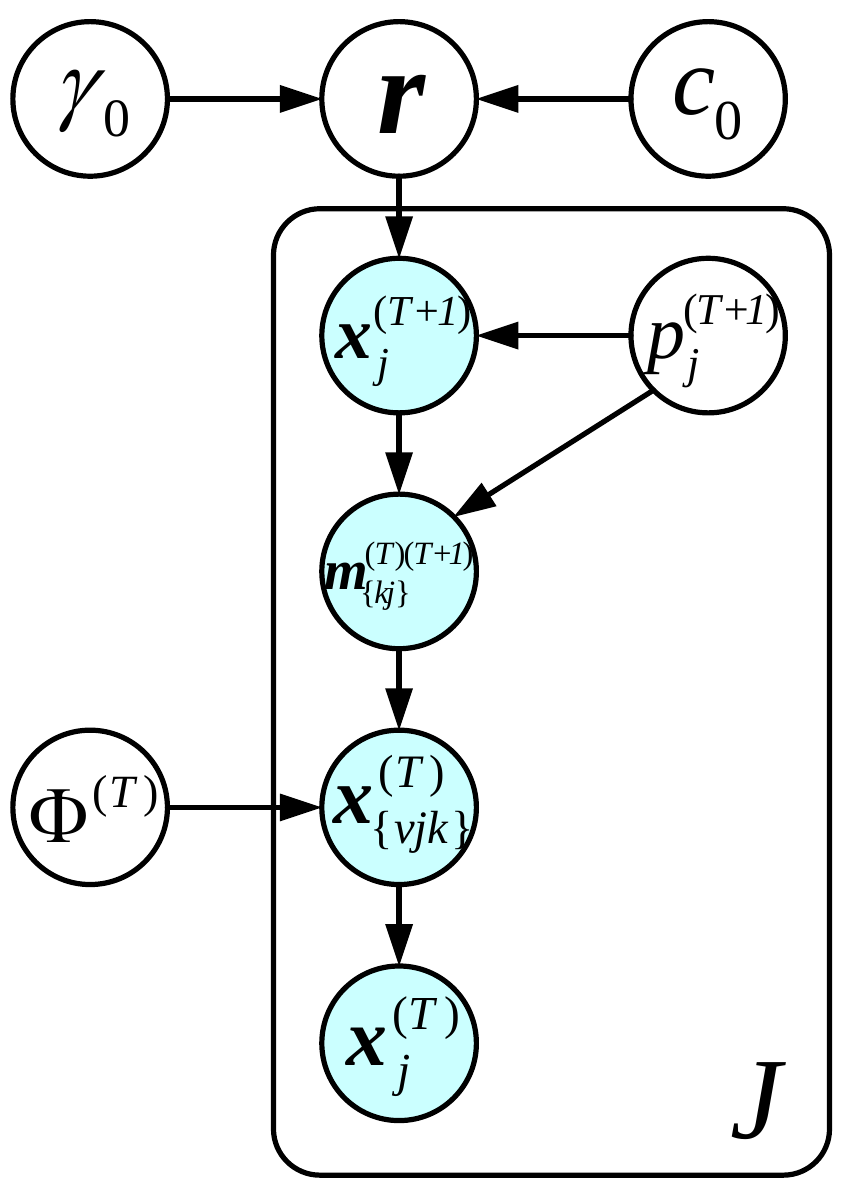}}


\vspace{-3.mm}
\caption{\small \label{fig:graphical_model} 
Graphical representations of the model and data augmentation and marginalization based inference scheme. (a) graphical representation of the GBN hierarchical model. (b) an augmented representation of Poisson factor model of layer $t=1$, corresponding to \eqref{eq:PoAug} with $t=1$. (c) an alternative representation using the relationships between the Poisson and multinomial distributions, obtained by applying Lemma 4.1 of \citep{BNBP_PFA_AISTATS2012} on \eqref{eq:PoAug} for $t=1$. (d) a negative binomial distribution based representation that marginalizes out the gamma from the Poisson distributions, corresponding to \eqref{eq:NBAug} for $t=1$. (e) an equivalent representation that introduces CRT distributed auxiliary variables, corresponding to \eqref{eq:CRT} with $t=1$. (f) an equivalent representation using Theorem 1 of \citep{NBP2012} on (\ref{eq:NBAug}) and \eqref{eq:CRT} for $t=1$. (g) An representation obtained by repeating the same augmentation-marginalization steps described in (b)-(f) one layer at a time from layers $1$ to $t$. (h) An representation 
of the top hidden layer.
 }\vspace{-0mm}
\end{figure}

Note that $x^{(t)}_{\cdotv j} = m^{(t)(t+1)}_{\cdotv j}$, and  as the number of tables occupied by the customers is in the same order as the logarithm of the customer number in a Chinese restaurant process, $x^{(t+1)}_{kj}$ is in the same order as $\ln\big( m^{(t)(t+1)}_{kj}\big)$.
Thus the total count of layer $t+1$ as $\sum_j x^{(t+1)}_{\cdotv j}$ would often be much smaller than that of layer $t$ as $\sum_j x^{(t)}_{\cdotv j}$ (though in general not as small as a count that is in the same order of the logarithm of $\sum_j x^{(t)}_{\cdotv j}$), and hence one may use the total count $\sum_j x^{(T)}_{\cdotv j}$ as a simple criterion to decide whether it is necessary to add more layers to the GBN. 
In addition, if the latent count $x_{k'\cdotv k}^{(t)}:= \sum_j x^{(t)}_{k'jk}$ becomes close or equal to zero, then the posterior mean of $\Phimat^{(t)}(k',k)$ could become so small that node $k'$ of layer $t-1$ can be considered to be disconnected from node $k$ of layer $t$.


\subsection{Modeling Data Variability With Distributed Representation}\label{sec:distributed}
In comparison to a single-layer model with $T=1$, which  assumes that the hidden units of layer one are independent in the prior, the multilayer model with $T\ge 2$ captures the correlations between them. 
Note that for the extreme case that $\Phimat^{(t)}=\Imat_{K_{t}}$ for $t\ge 2$ are all identity matrices, which indicates that there are no correlations between the features of $\thetav_j^{(t-1)}$ left to be captured, 
the deep structure could still provide benefits as it helps 
model latent counts $\mv_{j}^{(1)(2)}$ that may be highly overdispersed. For example, let us assume $\Phimat^{(t)}=\Imat_{K_{2}}$ for all $t\ge2$, 
then from (\ref{eq:PGBN}) and (\ref{eq:NBAug}) we have 
\beq
m_{kj}^{(1)(2)}\sim\mbox{NB}(\theta_{kj}^{(2)},p_j^{(2)}), ~\ldots,~\theta_{kj}^{(t)}\sim\mbox{Gam}(\theta_{kj}^{(t+1)},1/c_j^{(t+1)}),~\ldots,~\theta_{kj}^{(T)}\sim\mbox{Gam}(r_k,1/c_j^{(T+1)}).\notag
\eeq
Using the laws of total expectation and total variance, we have
\beq
\E\big[\theta_{kj}^{(2)}\,|\,r_k\big] = \frac{ r_k}{\prod_{t=3}^{T+1} c_j^{(t)}},~~~~~~
  \mbox{var}\big[\theta_{kj}^{(2)}\,|\,r_k\big] =r_k\sum_{t=3}^{T+1} \left[\prod_{\ell=3}^{t} \left(c_j^{(\ell)}\right)^{-2}\right]\left[ \prod_{\ell=t+1}^{T+1} \left(c_j^{(\ell)}\right)^{-1}\right].\notag
\eeq
Further applying the same laws, we have
$$
\E\big[m_{kj}^{(1)(2)}\,|\,r_k\big] =  \frac{r_k p_j^{(2)}}{\left(1-p_j^{(2)}\right)\prod_{t=3}^{T+1} c_j^{(t)}},$$
$$  \mbox{var}\big[m_{kj}^{(1)(2)}\,|\,r_k\big] = \frac{r_k p_j^{(2)}}{\big(1-p_j^{(2)}\big)^2 {\prod_{t=3}^{T+1} c_j^{(t)}} } \left\{1 + p_j^{(2)} \sum_{t=3}^{T+1} \left[\prod_{\ell=3}^{t} \left(c_j^{(\ell)}\right)^{-1}\right]
\right\}. 
$$
Thus the variance-to-mean ratio (VMR) of the count $m_{kj}^{(1)(2)}$ given $r_k$ can be expressed as
\beq
\mbox{VMR}\big[m_{kj}^{(1)(2)}\,|\,r_k\big] = \frac{1}{\big(1-p_j^{(2)}\big) } \left\{1 + p_j^{(2)} \sum_{t=3}^{T+1} \left[\prod_{\ell=3}^{t} \left(c_j^{(\ell)}\right)^{-1}\right]
\right\}. 
\eeq

In comparison to  PFA with $m_{kj}^{(1)(2)}\sim\mbox{NB}(r_k,p_j^{(2)})$ given $r_k$,  
with a VMR of $1/(1- p_j^{(2)})$, 
the GBN with $T$ hidden layers, which mixes the shape of $m_{kj}^{(1)(2)}\sim\mbox{NB}(\theta_{kj}^{(2)},p_j^{(2)})$ with a chain of gamma random variables, 
increases $\mbox{VMR}\big[m_{kj}^{(1)(2)}\,|\,r_k\big]$ 
 by a factor of 
$$1 + p_j^{(2)} \sum_{t=3}^{T+1} \left[\prod_{\ell=3}^{t} \left(c_j^{(\ell)}\right)^{-1}\right],$$
which is equal to
$$1 + (T-1)p_j^{(2)} $$ 
if we further assume $c_j^{(t)} = 1$ for all $t\ge 3$.
Therefore, by increasing the depth of the network to distribute the variability into more layers, the multilayer structure could increase its capacity to model data variability.

\subsection{Learning The Network Structure With Layer-Wise Training}\label{sec:greedy}

As jointly training all layers together is often difficult, 
existing deep networks are typically trained 
using a greedy layer-wise unsupervised training algorithm, such as the one proposed in \citep{hinton2006fast} to train the deep belief networks. The effectiveness of this training strategy is further analyzed in \citep{bengio2007greedy}. 
By contrast, the augmentable GBN has a simple 
Gibbs sampler to jointly train all its hidden layers, as described in Appendix~\ref{sec:sampling}, and 
hence does not necessarily require greedy layer-wise training, but the same as these commonly used deep learning algorithms, it still needs to specify the number of layers and the width of each layer.

In this paper, we adopt the idea of layer-wise training for the GBN, not because of the lack of an effective joint-training algorithm that trains all layers together in each  iteration, but for the purpose of learning the width of each hidden layer in a greedy layer-wise manner, given a fixed budget on the 
width of the first layer. The basic idea is to first train a GBN with a single hidden layer, $i.e.$, $T=1$, for which we know how to use the gamma-negative binomial process \citep{NBP2012,NBP_CountMatrix} to infer the posterior distribution of the number of active factors; we fix the width of the first layer $K_1$ with the number of active factors inferred at  iteration $B_1$, prune all inactive factors of the first layer, and continue Gibbs sampling for another $C_1$ iterations. Now we describe the proposed recursive procedure to build a GBN with $T\ge 2$ layers.
With a GBN of $T-1$ hidden layers that has already been inferred, 
for which the hidden units of the top layer are distributed as $\thetav_{j}^{(T-1)}\sim\mbox{Gam}(\rv, 1/c_j^{(T)})$, where $ \rv=(r_1,\ldots,r_{K_{T-1}})'$, 
 we add another layer by letting $\thetav_{j}^{(T-1)}\sim\mbox{Gam}(\Phimat^{^{(T)}}\thetav_{j}^{(T)},1/c_j^{(T)}), ~\thetav_{j}^{(T)}\sim\mbox{Gam}(\rv, 1/c_j^{(T+1)})$, where $\Phimat^{^{(T)}}\in\mathbb{R}_+^{K_{T-1}\times K_{T_{\max}}}$ and $\rv$ is redefined as $ \rv=(r_1,\ldots,r_{K_{T_{\max}}})'$. 
The key idea is with latent counts $m_{kj}^{(T)(T+1)}$ upward propagated from the bottom data layer, one may marginalize out    $\theta_{kj}^{(T)}$, leading to $m_{kj}^{(T)(T+1)}\sim\mbox{NB}(r_k,p_j^{(T+1)}),~ r_k\sim\mbox{Gam}(\gamma_0/K_{T\max},1/c_0)$, and hence can again rely on the shrinkage mechanism of a truncated gamma-negative binomial process  
 to prune inactive factors (connection weight vectors, columns of $\Phimat^{(T)}$) of layer $T$,  making $K_T$, the inferred layer width for the newly added layer,  smaller than $K_{T\max}$ if $K_{T\max}$ is set to be sufficiently large. The newly added layer and all the layers below  would be jointly trained, but with 
the structure below the newly added layer kept unchanged. 
 Note that when $T=1$, the GBN infers the number of active factors if $K_{1\max}$ is set large enough, otherwise, it still assigns the factors with different weights $r_k$, but may not be able to prune any of them. 
The details of the proposed layer-wise training strategies are summarized in Algorithm~\ref{tab:algorithm} for multivariate count data, and in Algorithm~\ref{tab:algorithm2} for multivariate binary and nonnegative real data.

 %
%
%

\section{Experimental Results}\label{sec:examples}
In this section, 
 we present experimental results for count, binary, and nonnegative real data.
\subsection{Deep Topic Modeling}
We first analyze multivariate count data with the Poisson gamma belief network (PGBN).
We apply the PGBNs for topic modeling of text corpora, each document of which 
is represented as a term-frequency 
count vector. 
Note that the PGBN with a single hidden layer is identical to the (truncated) gamma-negative binomial process PFA of \citet{NBP2012}, which is a nonparametric Bayesian algorithm that performs similarly to the hierarchical Dirichlet process latent Dirichlet allocation of \citet{HDP} for text analysis, and is considered as a strong baseline. 
Thus we will focus on making comparison to the PGBN with a single layer, with its layer width set to be large to approximate the performance of the gamma-negative binomial process PFA.
We evaluate the PGBNs' performance by examining both how well they unsupervisedly extract low-dimensional features for document classification, and how well they predict heldout word tokens. Matlab code will be available in \href{http://mingyuanzhou.github.io/}{http://mingyuanzhou.github.io/}.

We use Algorithm \ref{tab:algorithm} to learn, in a layer-wise manner, from the training data the connection weight matrices $\Phimat^{(1)},\ldots,\Phimat^{(T_{\max})}$ and the top-layer hidden units' gamma shape parameters $\rv$: 
to add layer $T$ to a previously trained network with $T-1$ layers, we use $B_T$ 
 iterations to jointly train $\Phimat^{(T)}$ and $\rv$ together with $\{\Phimat^{(t)}\}_{1,T-1}$, prune the inactive factors of layer $T$, and continue the joint training with another $C_T$ iterations. 
We set the hyper-parameters as $a_0=b_0=0.01$ and $e_0=f_0=1$. 
Given the trained network, we apply the upward-downward Gibbs sampler to collect 500 MCMC samples after 500 burnins to estimate the posterior mean of the feature usage proportion vector $\thetav_j^{(1)}/\theta_{\cdotv j}^{(1)}$ at the first hidden layer, for every document in both the training and testing sets.

\subsubsection{Feature Learning for Binary Classification}
We consider the 20newsgroups data set  that consists of 18,774 documents from 20 different news groups, with a vocabulary of size $K_0= $ 61,188. It is partitioned into a training set of 11,269 documents and a testing set of 7,505 ones. 
We first consider 
two binary classification tasks that distinguish between 
the $comp.sys.ibm.pc.hardware$ and $comp.sys.mac.hardware$, and between the
$sci.electronics$ and $sci.med$ news groups. 
For each binary classification task, we remove a standard list of stop words and only consider the terms that appear at least five times, 
and report the classification accuracies based on 12 independent random trials. 
With the upper bound of the first layer's width set as $K_{1\max}\in\{25,50,100, 200, 400, 600, 800\}$, and $B_t=C_t=1000$ and $\eta^{(t)}=0.01$ 
for all $t$, we use Algorithm \ref{tab:algorithm} to train a network with $T\in\{1, 2, \ldots,8\}$ layers. 
Denote $\bar{\thetav}_j$ as the estimated $K_1$ dimensional feature vector for document $j$, where $K_1\le K_{1\max}$ is the inferred number of active factors of the first layer that is bounded by the pre-specified truncation level $K_{1\max}$. We use the $L_2$ regularized logistic regression provided by the LIBLINEAR package \citep{REF08a} to train a linear classifier on $\bar{\thetav}_j$ in the training set and use it to classify $\bar{\thetav}_j$ in the test set, where the regularization parameter is five-folder 
 cross-validated on the training set from $(2^{-10}, 2^{-9},\ldots, 2^{15})$.

\begin{figure}[!tb]
\begin{center}
\includegraphics[width=0.6\textwidth]{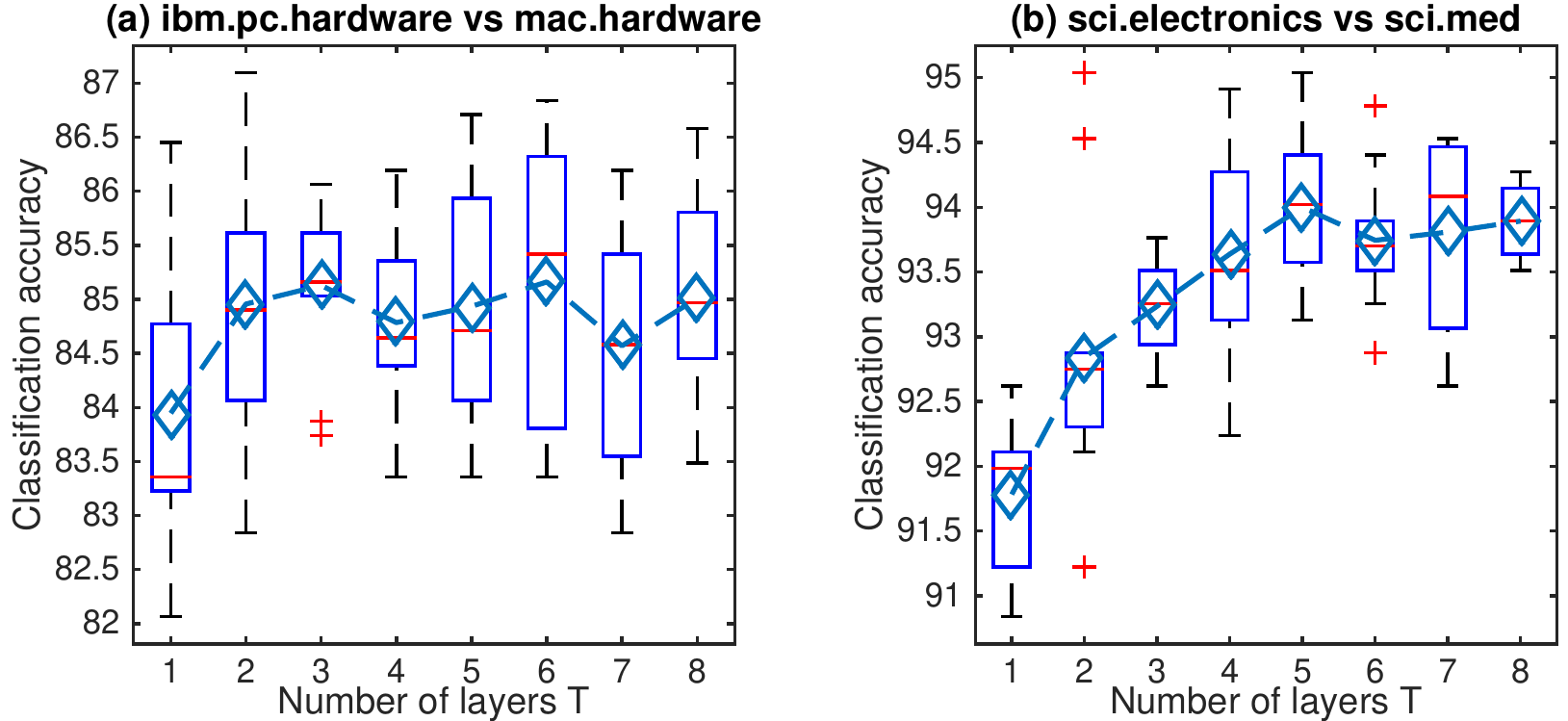}\\
\, \, \, \hspace{3mm}\includegraphics[width=0.6\textwidth]{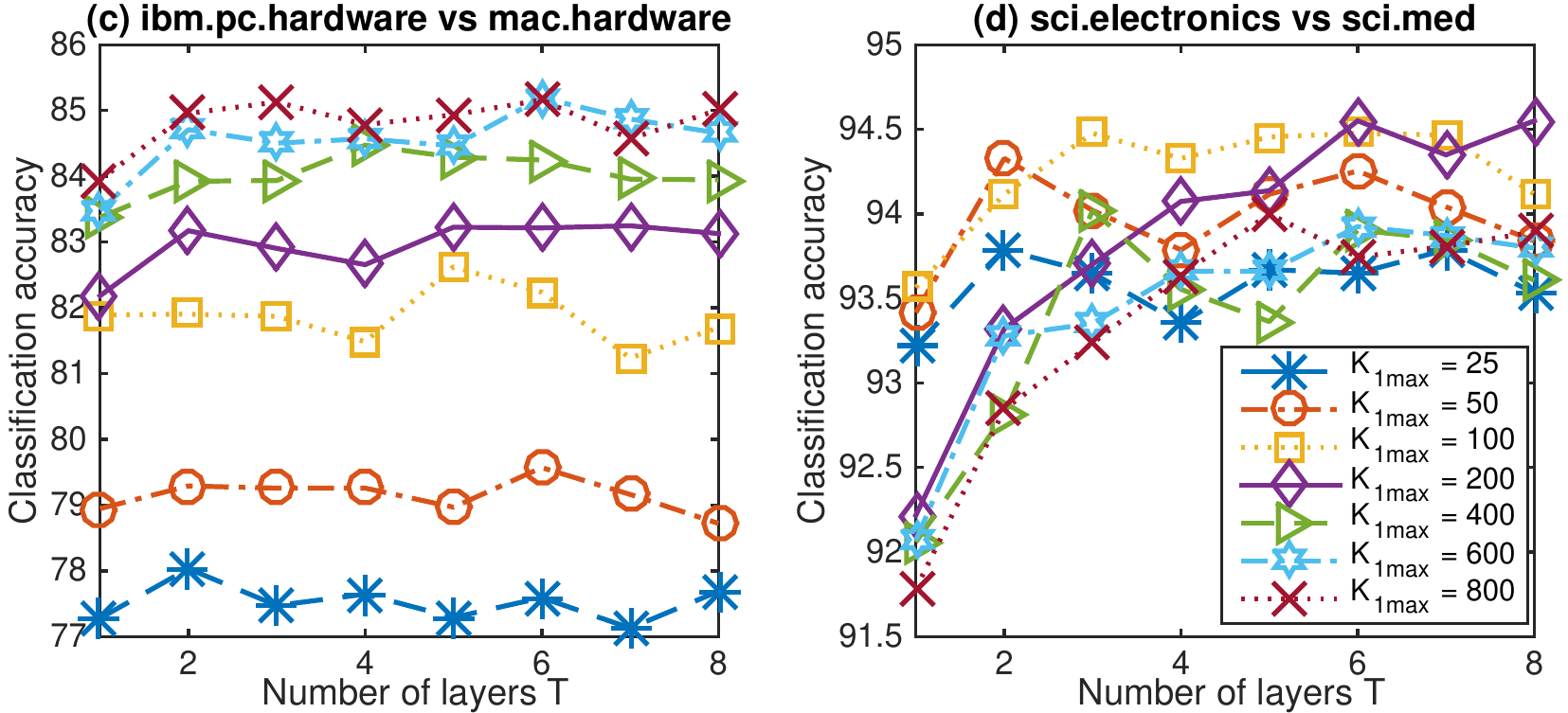}
\end{center}
\vspace{-6mm}
\caption{\small \label{fig:binary_20news}
Classification accuracy (\%) as a function of the network depth $T$ for two 20newsgroups binary classification tasks, 
with $\eta^{(t)} =0.01$ for all layers. 
(a)-(b): the boxplots of the accuracies of 12 independent runs with $K_{1\max}=800$. (c)-(d): the average accuracies of these 12 runs for various $K_{1\max}$ and $T$.
Note that $K_{1\max}=800$ is large enough to cover all active first-layer topics (inferred to be around 500 for both binary classification tasks), whereas all the first-layer topics would be used if $K_{1\max}=25$, $50$, $100$, or $200$.
 \vspace{-2.mm}
}
\end{figure}

As shown in Figure~\ref{fig:binary_20news}, modifying the PGBN from a single-layer shallow network to a multilayer deep one clearly improves the qualities of the unsupervisedly extracted 
feature vectors. 
In a random trial, with $K_{1\max}=800$, we infer a network structure of $[K_1,\ldots,K_8]=[512,  154,  75,  54 ,  47 , 37,  34 ,  29]$ for the first binary classification task, and $[K_1,\ldots,K_8]=[491, 143,  74,  49,  36,  32,  28 ,  26]$ for the second one. 
 Figures~\ref{fig:binary_20news}(c)-(d) also show that increasing the network depth in general improves the performance, but the first-layer width clearly plays a critical role in controlling the ultimate network capacity. This insight is further illustrated below. 


%
%

\subsubsection{Feature Learning for Multi-Class Classification}\label{sec:multiclass}

\begin{figure}[!tb]
\begin{center}
\includegraphics[width=0.46\textwidth]{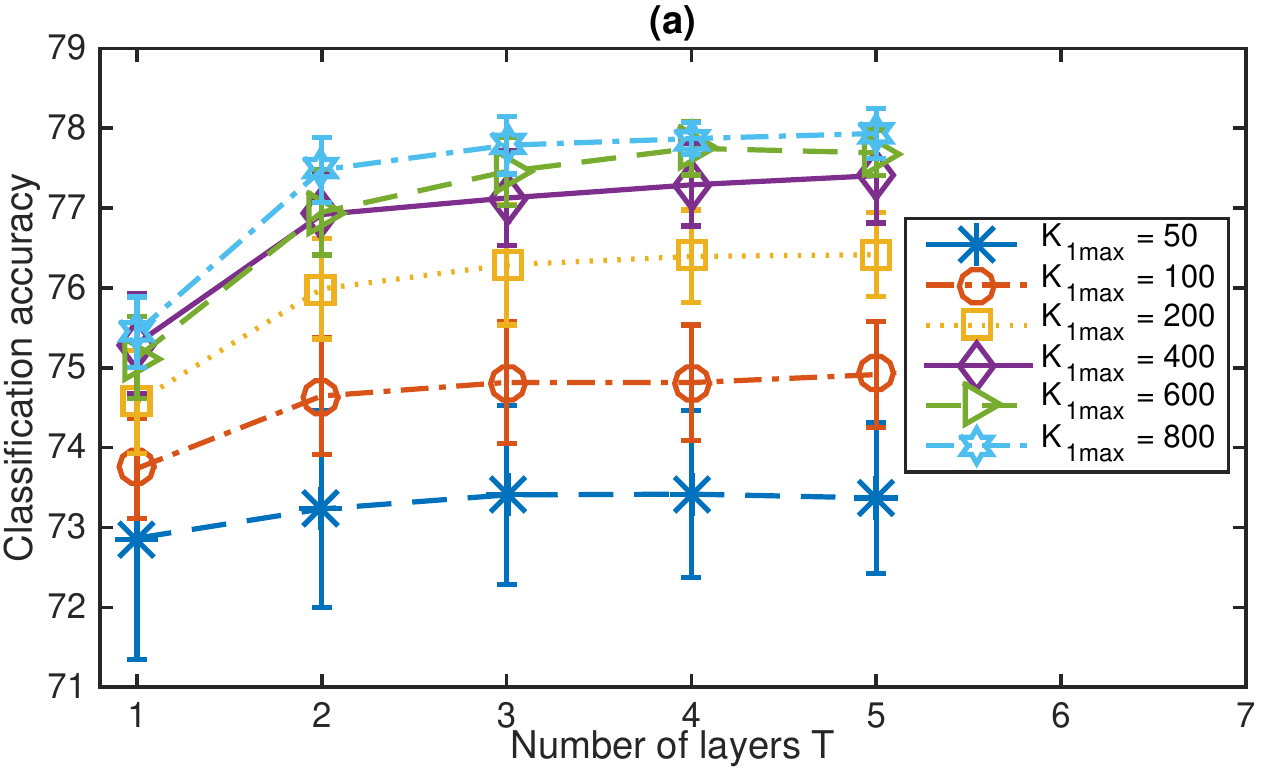}
\,\,
\includegraphics[width=0.46\textwidth]{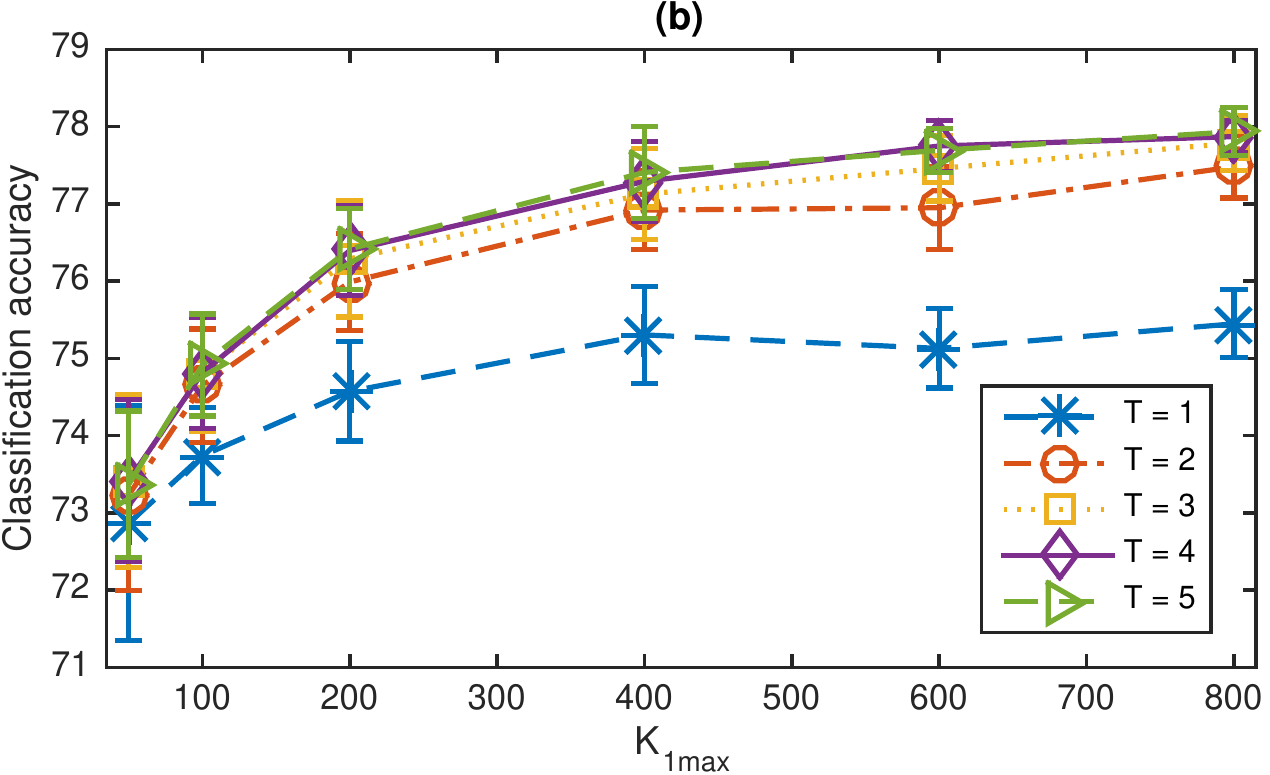}
\end{center}
\vspace{-6mm}
\caption{\small \label{fig:full_20news}
Classification accuracy (\%) of the PGBNs with Algorithm 1 for 20newsgroups multi-class classification (a) as a function of the depth $T$ with various $K_{1\max}$ and (b) as a function of $K_{1\max}$ with various depths, 
with $\eta^{(t)}=0.05$ for all layers. 
The widths of the hidden layers are automatically inferred. 
In a random trial, 
the inferred network widths $[K_1,\ldots,K_5]$ for $K_{1\max}=50,100,200,400,600$, and $800$ 
are 
$[ 50 , 50  , 50  , 50,  50]$, 
$[100,  99,  99 , 94 , 87]$, 
$[200,  161 , 130 , 94,  63]$, 
$[396 , 109 , 99  ,82 ,  68]$, 
$[528 , 129  ,109,  98,  91]$, and
$[608,  100,  99,  96,  89]$, respectively. 
}

\begin{center}
\includegraphics[width=0.45\textwidth]{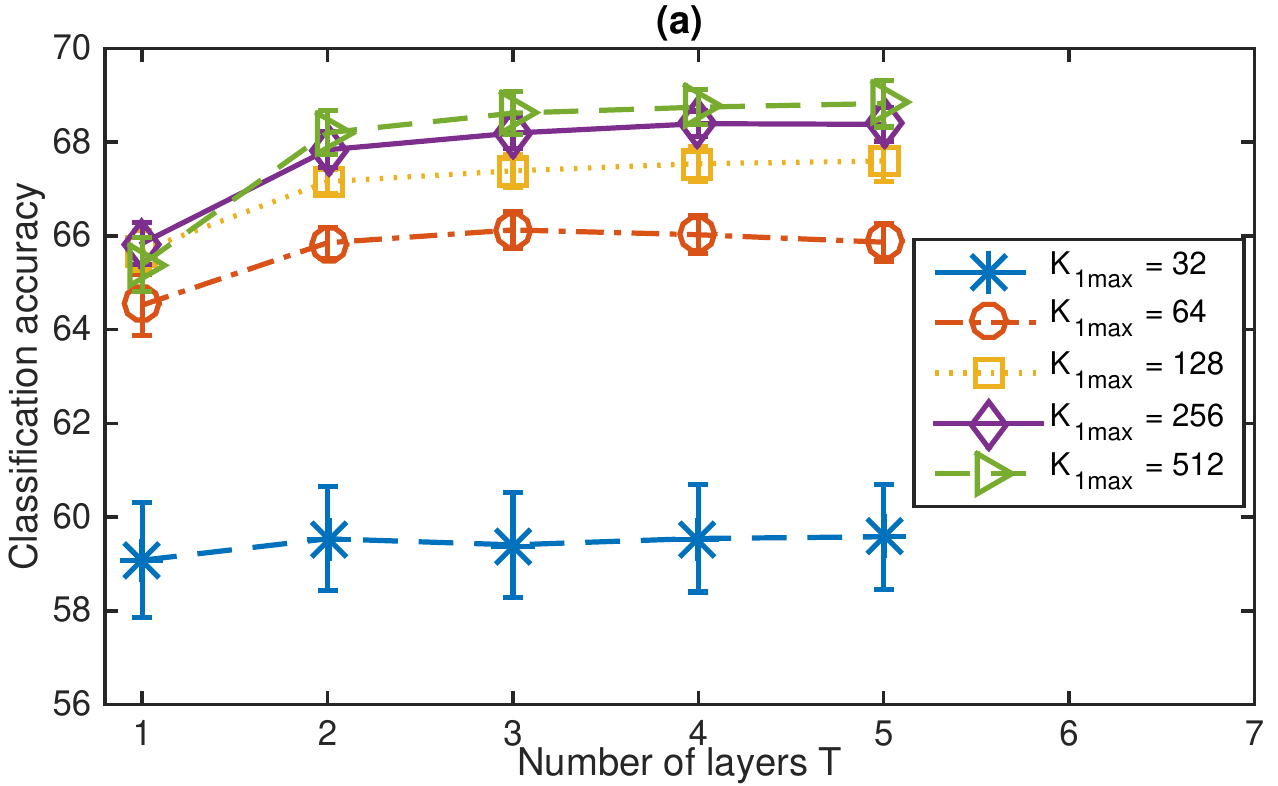}
\,\,
\includegraphics[width=0.45\textwidth]{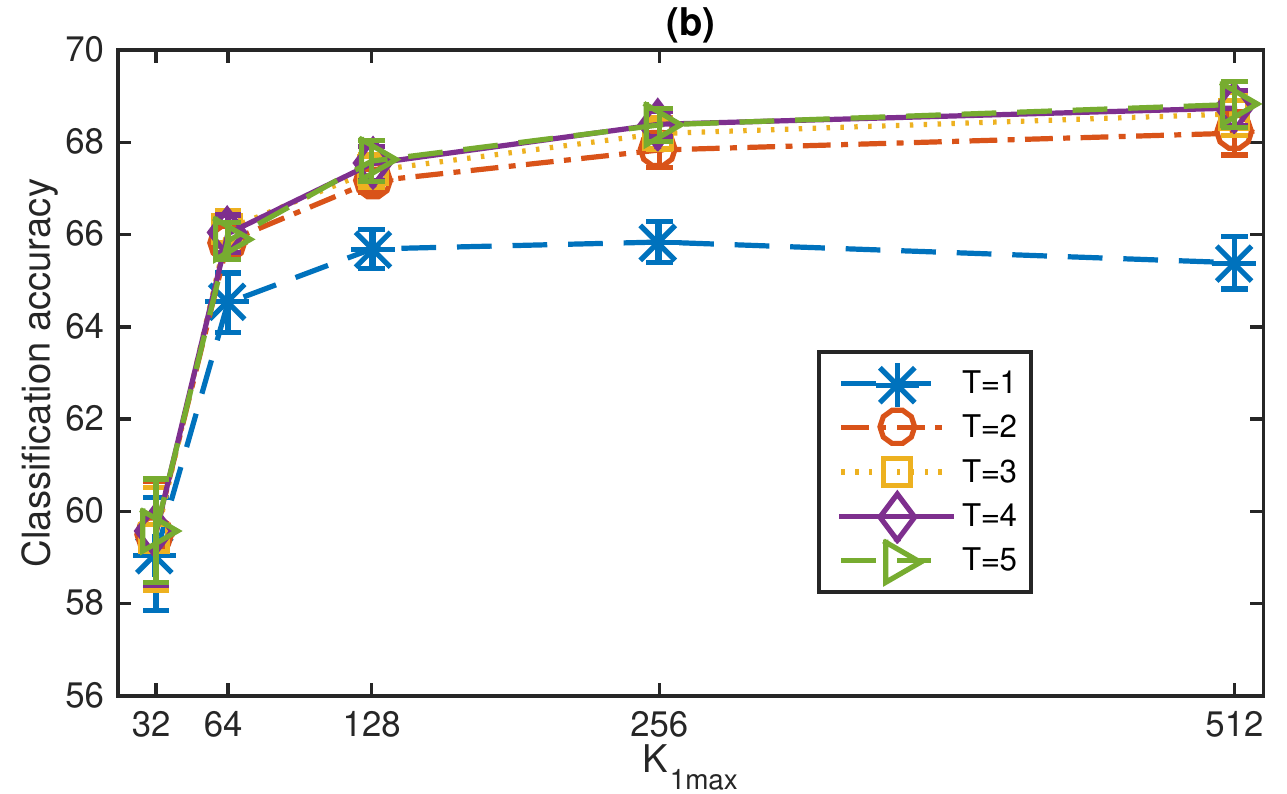}
\end{center}
\vspace{-6mm}
\caption{\small \label{fig:full_20newsTop2000}
Analogous plots to Figure \ref{fig:full_20news} with the vocabulary size restricted to be 2000, including the most frequent 2000 terms after removing a standard list of stopwords. 
The widths of the hidden layers are automatically inferred. 
In a random trial, 
the inferred network widths $[K_1,\ldots,K_5]$ for $K_{1\max}=32,64,128,256$, and $512$ 
are 
$[32,  32,  32 ,  32 ,  32]$, 
$[  64 , 64 ,  64 ,  59 ,  59]$,
$[ 128 , 125 , 118  ,106 ,  87]$,
$[ 256 , 224 , 124 , 83  ,65]$, and
 $[ 512,  187,  89 , 78 , 62]$, respectively. 
%
 \vspace{-3mm}
}
\end{figure}

 We test the PGBNs for multi-class classification on 20newsgroups. After removing a standard list of stopwords and the terms that appear less than five times, we obtain a vocabulary with $V= 33,420$. We set 
 $C_t=500$ and $\eta^{(t)} = 0.05$ for all $t$; we set $B_t=1000$ for all $t$ if $K_{1\max}\le 400$, and set $B_1=1000$ and $B_t=500$ for $t\ge 2$ if $K_{1\max}>400$. We use all 11,269 training documents to infer a set of networks with $T_{\max}\in\{1,\ldots,5\}$ and $K_{1\max}\in \{50, 100, 200, 400, 600,800\}$, and mimic the same testing procedure used for binary classification to extract low-dimensional feature vectors, with which each testing document is classified to one of the 20 news groups using the $L_2$ regularized logistic regression.
 
 Figure \ref{fig:full_20news} shows a clear trend of improvement in classification accuracy by increasing the network depth with a limited first-layer width, or by increasing the upper bound of the width of the first layer with the depth fixed. For example, a single-layer PGBN with $K_{1\max}=100$ could add one or more layers to slightly outperform a single-layer PGBN with $K_{1\max}=200$, and a single-layer PGBN with $K_{1\max} =200$ could add layers to clearly outperform a single-layer PGBN with $K_{1\max}$ as large as $800$.

The proposed Gibbs sampler also exhibits several desirable computational properties. 
Each iteration of jointly training multiple layers usually only costs moderately more than that of training a single layer, 
  e.g., with $K_{1\max}=400$, a training iteration on a single core of an Intel Xeon 2.7 GHz CPU takes about 
  $5.6$, $6.7$, $7.1$ seconds for the PGBN with $1$, $3$, and $5$ layers, respectively. 
Since the per iteration cost increases approximately  as a linear function of the inferred $K_1$ and as a linear function of the size of the data set, 
 given a fixed computational budget, one may choose a moderate $K_{1\max}$ to allow adding a sufficiently large number of hidden layers. 
In addition, the samplings of $x_{vkj}^{(t)}$,  $\phiv_k^{(t)}$, and $\theta_{kj}^{(t)}$ in each layer can all be made embarrassingly parallel with blocked Gibbs sampling, and hence can potentially significantly benefit from implementing the algorithm using graphics processing units (GPUs) or other parallel computing architectures. 


 
%

 Examining the inferred network structure also reveals interesting details. For example, in a random trial with Algorithm 1, with $\eta^{(t)}=0.05$ for all $t$, 
the inferred network widths $[K_1,\ldots,K_5]$ for $K_{1\max}=50,100,200,400,600$, and $800$ 
are 
$[ 50 , 50  , 50  , 50,  50]$, 
$[100,  99,  99 , 94 , 87]$, 
$[200,  161 , 130 , 94,  63]$, 
$[396 , 109 , 99  ,82 ,  68]$, 
$[528 , 129  ,109,  98,  91]$, and
[608,  100,  99,  96,  89], respectively. 
This indicates that for a network with an insufficient budget on its first-layer width, as the network depth increases, its inferred layer widths decay 
 more slowly than a network with a sufficient or surplus budget on its first-layer width; and a network with a surplus budget on its first-layer width may only need relatively small widths for its higher hidden layers.
 
In order to make comparison to related algorithms, we also consider restricting the vocabulary to the 2000 most frequent terms of the vocabulary after moving a standard list of stopwords. We repeat the same experiments with the same settings except that we set $K_{1\max}\in \{32, 64, 128, 256, 512\}$, $B_1=1000$, $C_1 =500$, and $B_t=C_t=500$ for all $t\ge2$. 
We show the results in Figure \ref{fig:full_20newsTop2000}. Again, we observe a clear trend of improvement by increasing the network depth with a limited first-layer width, or by increasing the upper bound of the width of the first layer with the depth fixed. In a random trial with Algorithm 1, 
the inferred network widths $[K_1,\ldots,K_5]$ for $K_{1\max}=32,64,128,256$, and $512$ 
are 
$[32,  32,  32 ,  32 ,  32]$, 
$[  64 , 64 ,  64 ,  59 ,  59]$,
$[ 128 , 125 , 118  ,106 ,  87]$,
$[ 256 , 224 , 124 , 83  ,65]$, and
 $[ 512,  187,  89 , 78 , 62]$, respectively.


For comparison, we first consider the same $L_2$ regularized logistic regression 
multi-class classifier, trained either on the raw word counts or normalized term-frequencies of the 20newsgroups training documents using five-folder cross-validation.
As summarized in Table \ref{tab:LR} of Appendix \ref{sec:fig}, when using the raw term-frequency word counts as covariates, the same classifier achieves $69.8\%$ ($68.2\%$) accuracy on the 20newsgroups test documents if using the top 2000 terms that exclude (include) a standard list of stopwords, achieves $75.8\%$ if using all the $61,188$ terms in the vocabulary, and achieves $78.0\%$ if using the $33,420$ terms remained after removing a standard list of stopwords and the terms that appear less than five times; 
and when using the normalized term-frequencies as covariates, the corresponding accuracies are $70.8\%$ ($67.9\%$) if using the top 2000 terms excluding (including) stopwords, $77.6\%$ with all the $61,188$ terms, and $79.4\%$ with the $33,420$ selected terms.

As summarized in Table \ref{tab:ORS} of Appendix \ref{sec:fig}, for multi-class classification on the same data set, with a vocabulary size of 2000 that consists of the 2000 most frequent terms after removing stopwords and stemming, the DocNADE \citep{larochelle2012neural} and the over-replicated softmax \citep{srivastava2013modeling} provide the accuracies of $67.0\%$	and $66.8\%$, respectively, for a feature dimension of $K=128$, and provide  the accuracies of $68.4\%$ and	$69.1\%$, respectively, for a feature dimension of $K=512$.

As shown in Figure \ref{fig:full_20newsTop2000} and summarized in Table \ref{tab:PGBN} of  Appendix \ref{sec:fig}, with the same vocabulary size of 2000 (but different terms due to different preprocessing),
the proposed PGBN provides $65.9\%$ ($67.5\%$) with $T=1$ ($T=5$) for $K_{1\max}=128$, and $65.9\%$ ($69.2\%)$ with $T=1$ ($T=5)$ for $K_{1\max}=512$, which may be further improved if we also consider the stemming step, as done in these two algorithms, for word preprocessing, or if we set the values of $\eta^{(t)}$ to be smaller than 0.05 to encourage a more complex network structure. We also summarize in Table \ref{tab:PGBN} the classification accuracies shown in Figure \ref{fig:full_20news} for the PGBNs with $V=33,420$. 
Note that the accuracies in Tables \ref{tab:ORS} and \ref{tab:PGBN} are provided to show that the PGBNs are in the same ballpark as both the DocNADE \citep{larochelle2012neural} and over-replicated softmax \citep{srivastava2013modeling}. Note these results are not intended to provide a head-to-head comparison, which 
is possible if the same data preprocessing and classifier were used and the error bars were shown in \citet{srivastava2013modeling}, or we could obtain the code to replicate the experiments using the same preprocessed data and classifier.

Note that $79.4\%$ achieved using the $33,420$ selected features is the best accuracy reported in the paper, 
which is unsurprising since all the unsupervisedly extracted latent feature vectors have much lower dimensions and are not optimized for classification \citep{zhu2012medlda,NBP_CountMatrix}. In comparison to using appropriately preprocessed high-dimensional features, our experiments show that while
text classification performance often clearly deteriorates if one trains a
multi-class classifier on the lower-dimensional features extracted using  ``shallow'' unsupervised latent feature models, 
one could obtain much improved results using appropriate ``deep'' generalizations. For further improvement, one may consider adding an extra supervised component into the model, which is shown to boost the classification performance for latent Dirichlet allocation \citep{mcauliffe2008supervised,zhu2012medlda}, a shallow latent feature model related to the PGBN with a single hidden layer.


\subsubsection{Perplexities for Heldout Words}
In addition to examining the performance of the PGBN for unsupervised feature learning, we also consider a more direct approach that we randomly choose 30\% of the word tokens in each document as training, and use the remaining ones to calculate per-heldout-word perplexity. We consider both all the 18,774 documents of 
the 20newsgroups corpus, limiting the vocabulary to the 2000 most frequent terms after removing a standard list of stopwords, and
the NIPS12 (\href{http://www.cs.nyu.edu/~roweis/data.html}{http://www.cs.nyu.edu/$\sim$roweis/data.html}) corpus whose stopwords have already been removed, limiting the vocabulary to the 2000 most frequent terms.
We set $\eta^{(t)}=0.05$ and $C_t=500$ for all $t$, 
set $B_1=1000$ and $B_t=500$ for $t\ge2$, and consider five random trials. 
 Among the $B_t+C_t$ Gibbs sampling iterations used to train layer $t$, we collect one sample per five iterations during the last $500$ iterations, for each of which we draw the topics $\{\phiv^{(1)}_k\}_k$ and topics weights $\thetav_{j}^{(1)}$, to compute the per-heldout-word perplexity using Equation (34) of \citet{NBP2012}. 
 This evaluation method is similar to those used in \citet{newman2009distributed}, \citet{wallach09}, and \citet{DILN_BA}.

\begin{figure}[!tb]
\begin{center}
\includegraphics[width=0.95\textwidth]{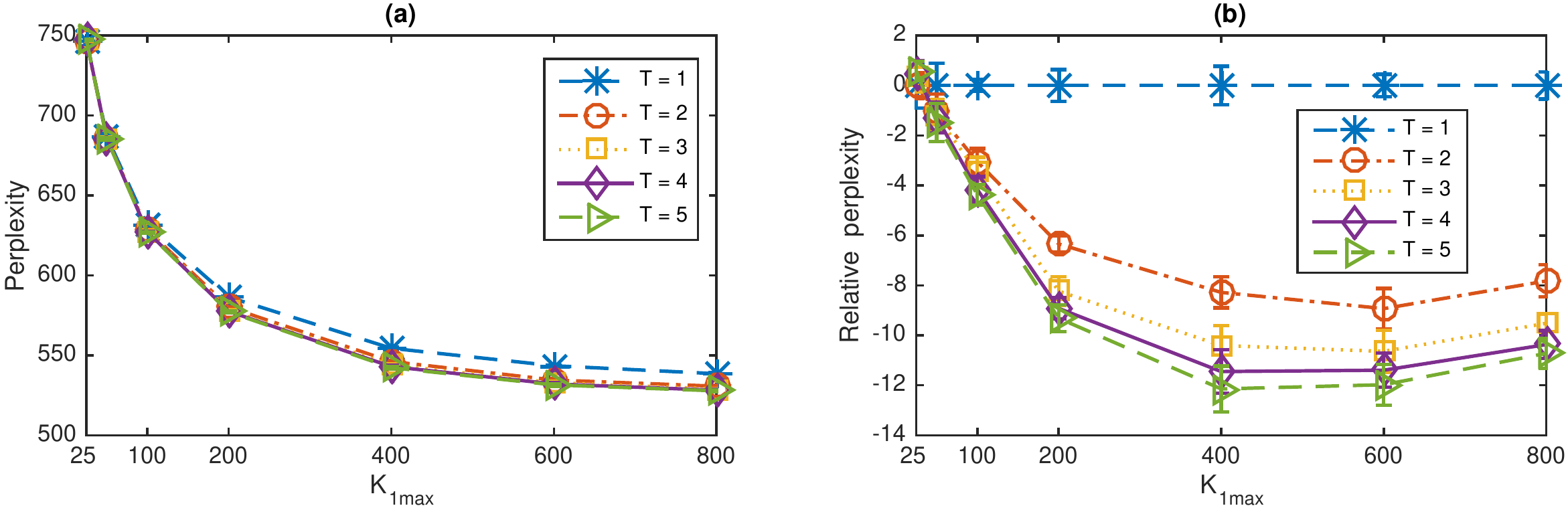}
\end{center}
\vspace{-6mm}
\caption{\small \label{fig:Perplexity}
(a) per-heldout-word perplexity (the lower the better) for the NIPS12 corpus (using the 2000 most frequent terms) as a function of the upper bound of the first layer width $K_{1\max}$ and network depth $T$, 
with $30\%$ of the word tokens in each document used for training and $\eta^{(t)}=0.05$ for all $t$. 
(b) for visualization, each curve in (a) is reproduced by subtracting its values from the average perplexity of the single-layer network. In a random trial, 
the inferred network widths $[K_1,\ldots,K_5]$ for $K_{1\max}=25,50,100,200, 400,600$, and $800$ are
 $[25,25,25,25,25]$, $[50,50,50,49,42]$, 
$[100,99,93,78,54]$,
$[200,164,106,60,42]$, $[400,130,83,52,39]$, $[596,71,68,58,37]$, and $ [755,57,53,46,42]$, respectively. 
}

\begin{center}
\includegraphics[width=0.95\textwidth]{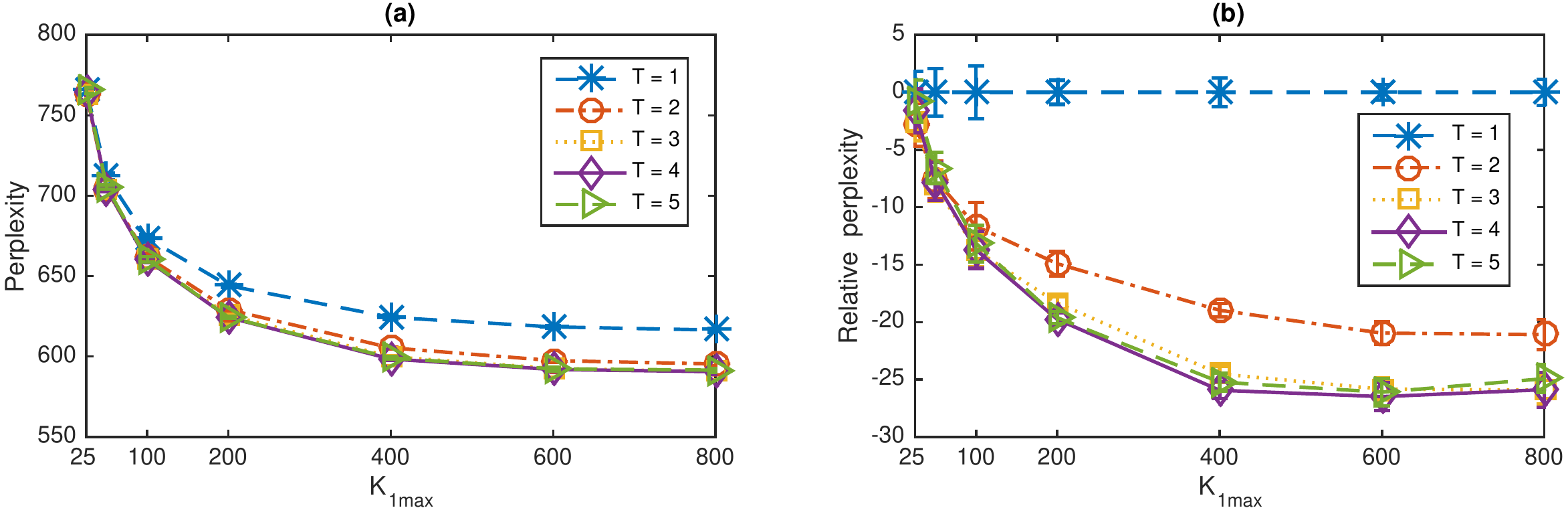}
\end{center}
\vspace{-6mm}
\caption{\small \label{fig:Perplexity_20news}
Analogous plots to Figure \ref{fig:Perplexity} for the 20newsgroups corpus (using the 2000 most frequent terms after removing a standard list of stopwords). %
In a random trial, 
the inferred network widths $[K_1,\ldots,K_5]$ for $K_{1\max}=25,50,100,200, 400,600$, and $800$ 
are
$[25,25,25,25,25]$, $[50,50,50,50,50]$, $[100,99,99,97,97]$, $[200,194,177,152,123]$, $[398,199,140,116,105]$, $[557,156,133,118,103]$, and $ [701,119,116,112,103]$, respectively.
}
\end{figure}
 
 As shown in both Figures \ref{fig:Perplexity} and \ref{fig:Perplexity_20news}, we observe a clear trend of improvement by increasing both $K_{1\max}$ and $T$. 
We have also examined the topics and network structure learned on the NIPS12 corpus. Similar to the exploratory data analysis performed on the 20newsgroups corpus, as described in detail in Section \ref{sec:interprete_network_structure}, the inferred deep networks also allow us to extract trees and subnetworks to visualize various aspects of the NIPS12 corpus from general to specific and reveal how they are related to each other. We omit these details for brevity and instead provide a brief description:
with $K_{1\max}=200$ and $T=5$, the PGBN infers a network with $[K_1,\ldots,K_5]=
[200, 164 ,106, 60 ,42]$ in one of the five random trials. 
The ranks, according to the weights $r_k^{(t)}$ calculated in \eqref{eq:weight}, and the top five words of three example topics for layer $T=5$ are
``6 network units input learning training,''
``15 data model learning set image,'' and
``34 network learning model input neural;''
while these of five example topics of layer $T=1$ are 
``19 likelihood em mixture parameters data,''
``37 bayesian posterior prior log evidence,''
``62 variables belief networks conditional inference,''
``126 boltzmann binary machine energy hinton,''
and
``127 speech speaker acoustic vowel phonetic.'' It is clear that the topics of the bottom hidden layers are very specific whereas these of the top hidden layer are quite general.

\subsubsection{Generating Synthetic Documents}





We have also tried drawing $\thetav_{j'}^{(T)}\sim\mbox{Gam}\big(\rv,1/c_{j'}^{(T+1)}\big)$ and downward passing it through a $T$-layer network trained on a text corpus  to generate synthetic bag-of-words documents, which are found to be quite interpretable and reflect various general aspects of the corpus used to train the network. 
We consider  the PGBN with $[K_1,\ldots,K_5]=[608,  100,  99,  96,  89]$, 
which is trained on the training set of the 
20newsgroups corpus with $K_{1\max}=800$ and $\eta^{(t)}=0.05$ for all $t$. 
We set $c_{j'}^{(t)}$ as the median of the inferred $\{c_j^{(t)}\}_j$ of the training documents for all~$t$. Given $\{\Phimat^{(t)}\}_{1,T}$ and $\rv$, we first generate $\thetav_{j'}^{(T)}\sim\mbox{Gam}\left(\rv,1\big/c_{j'}^{(T+1)}\right)$ and then downward pass it through the network by drawing nonnegative real random variables, one layer after another, from the gamma distributions as in 
\eqref{eq:PGBN}.
With the simulated $\thetav_{j'}^{(1)}$, we calculate the Poisson rates for all the $V$ words using $\Phimat^{(1)}\thetav_{j'}^{(1)}$ and 
display the top 100 words ranked by their Poisson rates. As shown in the text file available at \href{http://mingyuanzhou.github.io/Results/GBN-BOW.txt}{http://mingyuanzhou.github.io/Results/GBN-BOW.txt},  
the synthetic documents generated in this manner are all easy to interpret and reflect various general aspects of the 20newsgroups corpus on which the PGBN is trained. 

%
%
%
%
%
%
%

\subsection{Multilayer Representation for Binary Data}
We apply the BerPo-GBN to extract multilayer representations for high-dimensional sparse binary vectors. The BerPo link is proposed in \citet{EPM_AISTATS2015} to construct edge partition models for network analysis, 
 whose computation is mainly spent on pairs of linked nodes and hence is scalable to big sparse relational networks.  That link function and its inference procedure have also been adopted by \citet{hu2015zero} to analyze big sparse binary tensors. 
 
 We consider the same problem of feature learning for multi-class classification studied in detail in Section \ref{sec:multiclass}. We consider the same setting except that the original term-document word count matrix is now binarized into a term-document indicator matrix, the $(v,j)$ element of which is set as one if and only if $n_{vj}\ge 1$ and set as zero otherwise. 
We test the BerPo-GBNs on the 20newsgroups corpus, with $\eta^{(t)}=0.05$ for all layers. 
As shown in Figure \ref{fig:full_20newsBerPo}, given the same upper-bound on the width of the first layer, increasing the depth of the network clearly improves the performance. Whereas given the same number of hidden layers, the performance initially improves and then fluctuates as the upper-bound of the first layer increases. 
Such kind of fluctuations when $K_{1\max}$ reaches over 200 are expected, since the width of the first layer is inferred to be less than 190 and hence the budget as small as $K_{1\max}=200$ is already large enough to cover all active factors. 

\begin{figure}[!tb]
\begin{center}
\includegraphics[width=0.46\textwidth]{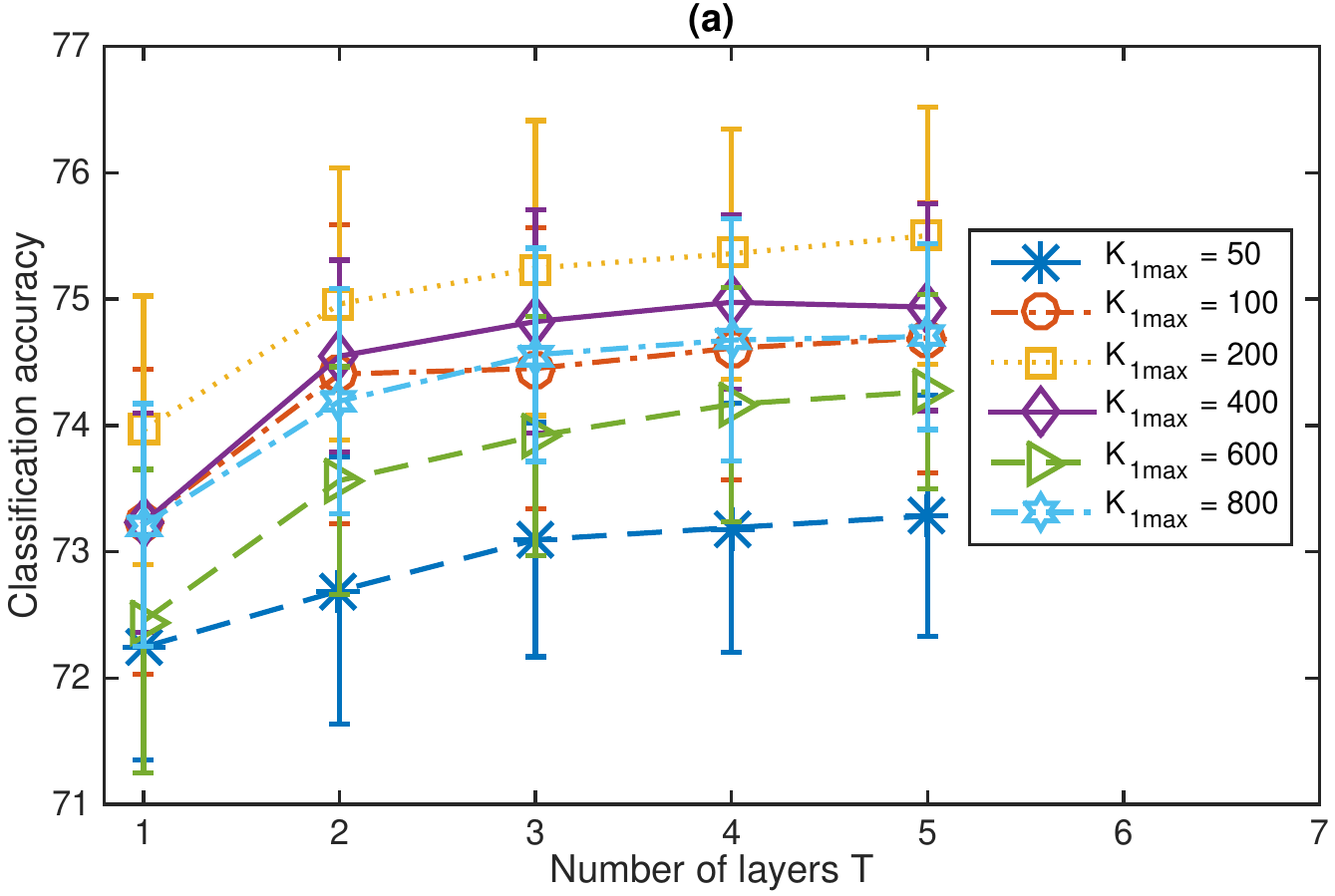}
\,\,
\includegraphics[width=0.46\textwidth]{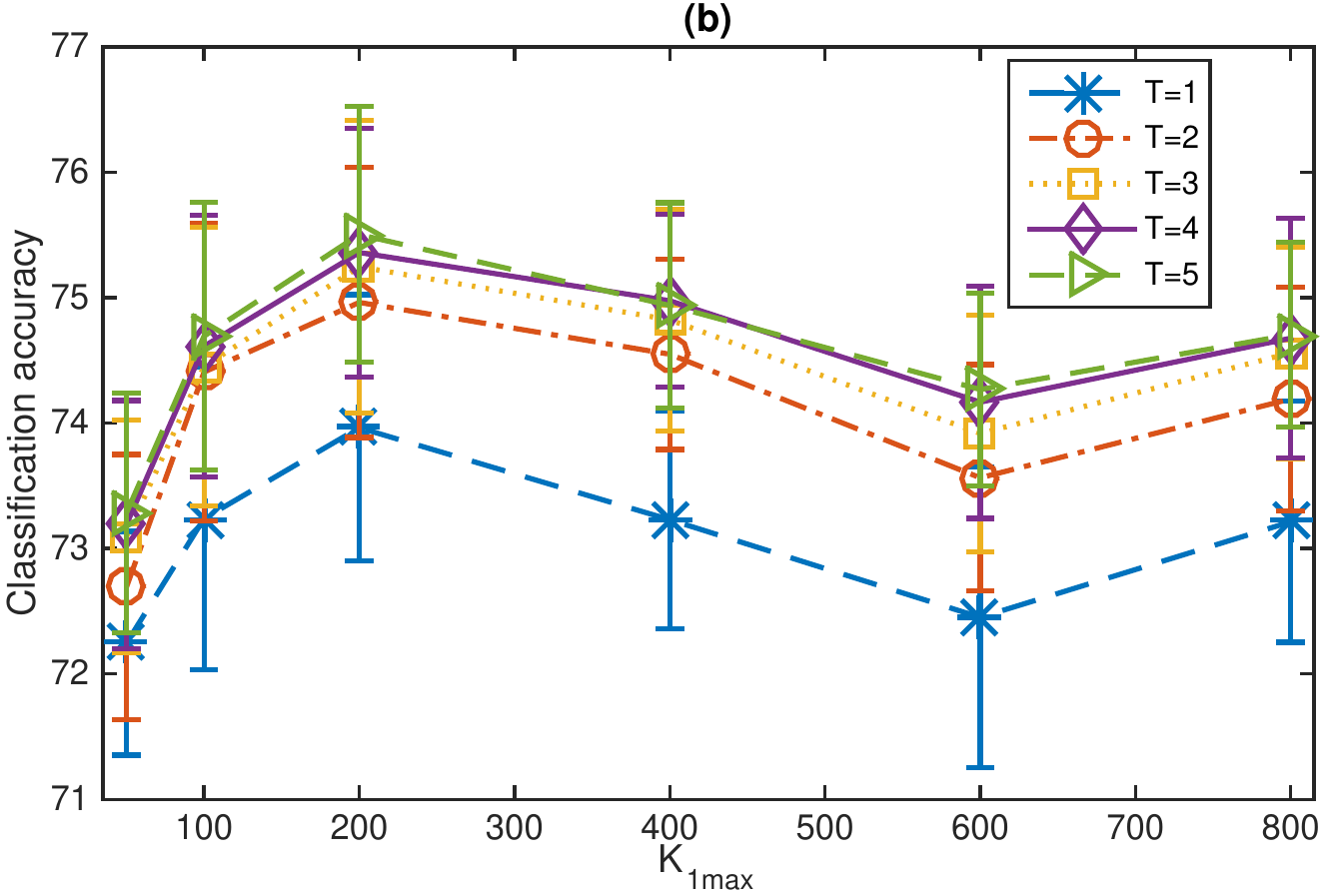}
\end{center}
\vspace{-6mm}
\caption{\small \label{fig:full_20newsBerPo}
Analogous plots to Figure \ref{fig:full_20news} for the BerPo-GBNs on the binarized 20newsgroups term-document count matrix. 
The widths of the hidden layers are automatically inferred.
In a random trial with Algorithm 2, 
the inferred network widths $[K_1,\ldots,K_5]$ for $K_{1\max}=50,100,200,400,600$, and $800$ 
are 
$[ 50 , 50  , 50  , 50,  50]$, 
$[100,  97,  95 , 90 , 82]$, 
$[178,  145 , 122 , 97,  72]$, 
$[184 , 139 , 119  ,101 ,  75]$, 
$[172 , 165  ,158,  138,  110]$, and
$[156,  151,  147,  134,  117]$, respectively.
}
\end{figure}

\subsection{Multilayer Representation for Nonnegative Real Data}

We use the PRG-GBN to unsupervisedly extract features from nonnegative real data. 
We consider the MNIST data set ({\href{http://yann.lecun.com/exdb/mnist/}{http://yann.lecun.com/exdb/mnist/}}), which consists of $60000$ training handwritten digits and $10000$ testing ones. We divide the gray-scale pixel values of each $28\times 28$ image by 255 and represent each image as a 784 dimensional nonnegative real vector. 
We set $\eta^{(1)}=0.05$ and use all training digits to infer the PRG-GBNs with $T_{max}\in\{1,\cdot\cdot\cdot,5\}$ and $K_{1max}\in \{50,100,200,400\}$. We consider the same problem of feature extraction for multi-class classification studied in detail in Section \ref{sec:multiclass}, and we follow the same experimental settings over there. As shown in Figure \ref{fig:MNIST_class}, both increasing the width of the first layer and the depth of the network could clearly improve the performance in terms of unsupervisedly extracting features that are better suited for multi-class classification.


\begin{figure}[!tb]
 \centering
\includegraphics[width=0.46\textwidth]{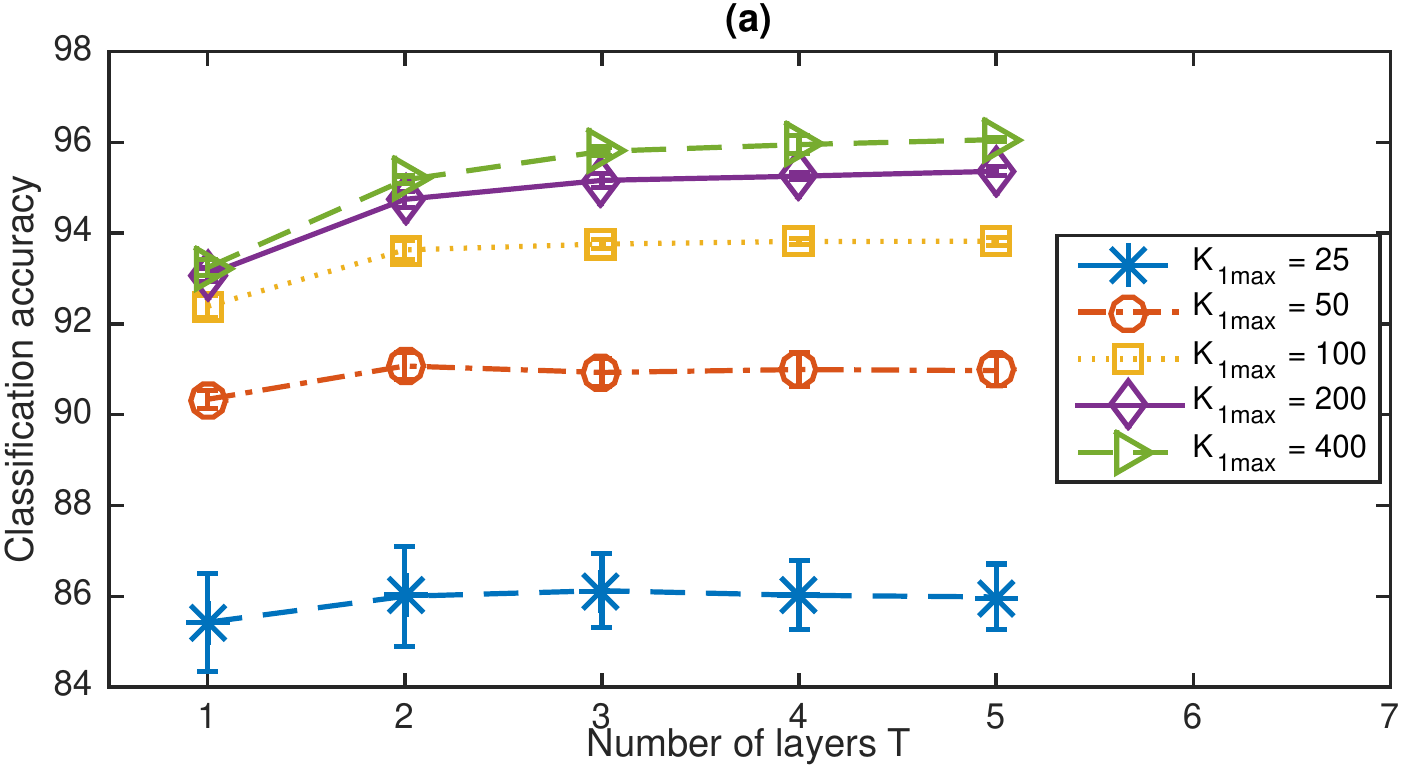} \, \,
\includegraphics[width=0.46\textwidth]{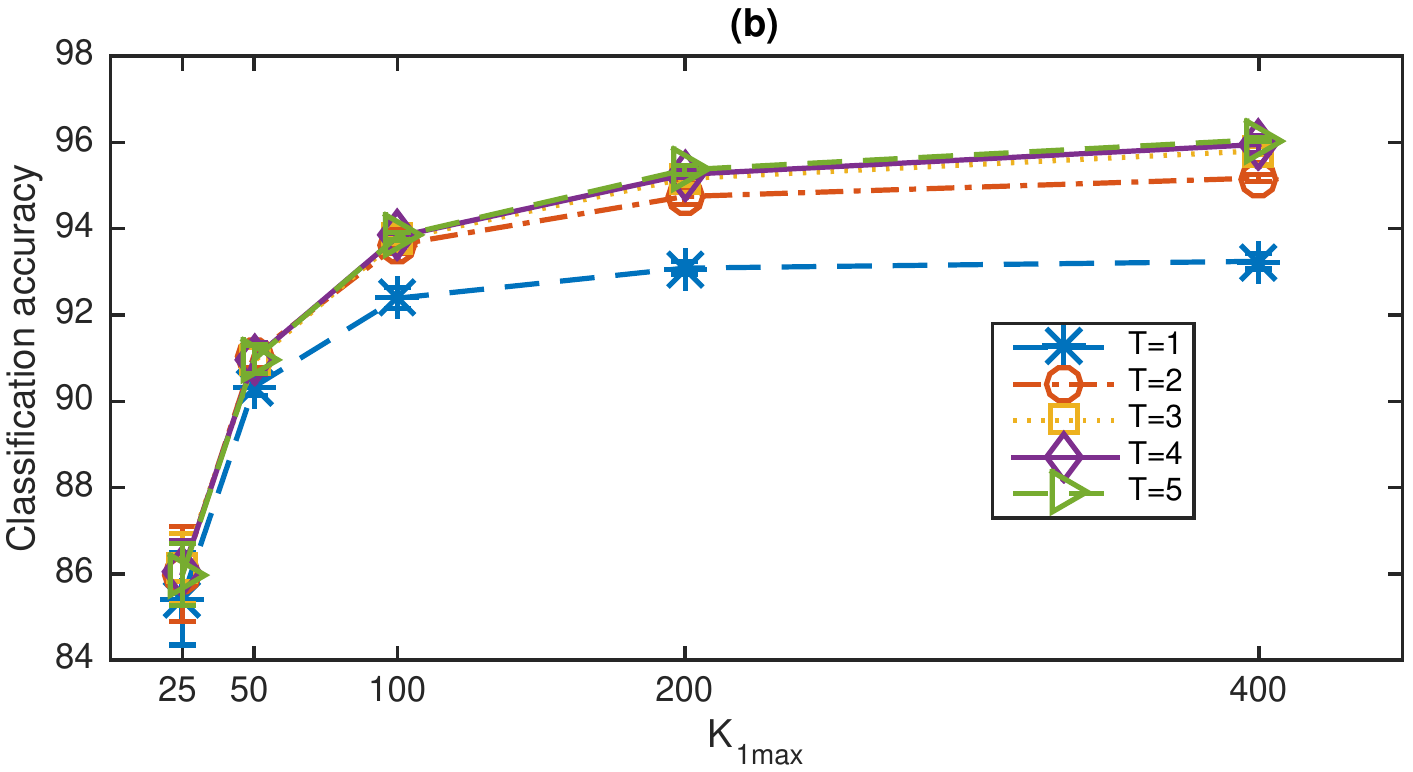}
\vspace{-0.3cm}
 \caption{\small
 Analogous plots to Figure \ref{fig:full_20news} for the PRG-GBNs on the MNIST data set. 
In a random trial with Algorithm 2, 
the inferred network widths $[K_1,\ldots,K_5]$ for $K_{1\max}$=$50$, $ 100$, $200$, and $400$ are $[50,50,50,50,50]$, $[100,100,100,100,100]$, $[200,200,200,200,200]$, and $[400,400,399,385,321]$, respectively.
 \label{fig:MNIST_class}} 
\end{figure}
%

Note that the PRG distribution might not be the best distribution to fit MNIST digits, but nevertheless, displaying the inferred features at various layers as images provides a straightforward way to visualize the latent structures inferred from the data and hence provides an excellent example to understand the properties and working mechanisms of the GBN. We display the projections to the first layer of the factors $\Phimat^{(t)}$ at all five hidden layers as images for $K_{1\max}=100$ and $K_{1\max}=400$ in Figures \ref{fig:AllProjPhis100} and \ref{fig:AllProjPhis400}, respectively, which clearly show that the inferred latent factors become increasingly more general as the layer increases. In both Figures \ref{fig:AllProjPhis100} and \ref{fig:AllProjPhis400}, the latent factors inferred at the first hidden layer represent filters that are only active at very particular regions of the images, those inferred at the second hidden layer represent larger parts of the hidden-written digits, and those inferred at the third and deeper layers resemble the whole digits. 

To visualize the relationships between the factors of different layers, we show in Figure~\ref{fig:MNIST_tree} in Appendix \ref{sec:fig} a subset of nodes of each layer and the nodes of the layer below that are connected to them with non-negligible weights. 

It is interesting to note that unlike \citet{Ng_deeplearning} and many other following works that rely on the convolutional and pooling operations, which are pioneered by \citet{LeCun89}, to extract hierarchical representation for images at different spatial scales, we show that the proposed algorithm, while 
not breaking the images into spatial patches, is already able to 
learn the factors that are active on very specific spatial regions of the image in the bottom hidden layer, and learn these increasingly more general factors covering larger spatial regions of the images as the number of layer increases.  However,  due to the lack of the ability to discover spatially localized features that can be shared at multiple different spatial regions, our algorithm does not at all exploit the 
redundancies of the spatially localized features inside a single image and hence may require much more data to train.  Therefore, it would  be interesting to investigate whether one can introduce convolutional and pooling operations into the GBNs, which may substantially improve their performance on modeling natural images.

\section{Conclusions}

The augmentable  gamma belief network (GBN) is proposed to extract a multilayer representation for high-dimensional count, binary, or nonnegative real vectors, with an efficient upward-downward Gibbs sampler to jointly train all its layers and a layer-wise training strategy to automatically infer the network structure.  
A GBN of $T$ layers can be broken into $T$ subproblems that are solved by repeating  the same subroutine,  with the computation mainly spent on training the first hidden layer. When used for deep topic modeling, the GBN extracts very specific topics at the first hidden layer and increasingly more general topics at deeper hidden layers. It provides an excellent way for exploratory data analysis through the visualization of the inferred deep network,  whose hidden units of adjacent layers are sparsely connected.
Its good performance is further demonstrated in unsupervisedly extracting features for document classification and predicting heldout word tokens. The extracted deep network can also be used to simulate very interpretable synthetic documents, which reflect various general aspects of the corpus that the network is trained on. When applied for image analysis, without using the convolutional and pooling operations, the GBN is already able to extract interpretable factors in the first hidden layer
that are active in very specific spatial regions  and interpretable factors in deeper hidden layers with increasingly more general spatial patterns covering larger spatial regions. 
For big data problems, in practice one may rarely have a sufficient budget to allow the first-layer width to grow without bound, thus it is natural to consider a deep network that can use a multilayer deep representation to better allocate its resource and increase its representation power with limited computational power. 
Our algorithm provides a natural solution to achieve a good compromise between the width of each layer and the depth of the network.


\acks{The authors would like to thank the editor and two anonymous referees for their insightful and constructive 
comments and suggestions, which have helped us improve the paper substantially.
M. Zhou thanks Texas Advanced Computing Center for computational support. B. Chen thanks the support
of the Thousand Young Talent Program of China, NSFC (61372132), NCET-13-0945, and NDPR-9140A07010115DZ01015.}

\vskip 0.2in
\appendix


\section{Randomized Gamma and Bessel Distributions}\label{app:bessel}

Related to our work, 
\citet{yuan2000bessel} proposed the randomized gamma distribution to generate a random positive real number as
\beq\notag
x\,|\, n,\nu\sim\mbox{Gam}(n+\nu+1,1/c), ~n\sim\mbox{Pois}(\lambda),
\eeq
where $\nu>-1$ and $c>0$.
As in \citet{yuan2000bessel}, the conditional posterior of $n$ can be expressed as
\beq\notag
(n\,|\,x,\nu,\alpha) \sim\mbox{Bessel}_{\nu}(2\sqrt{cx\lambda})
\eeq
where we denote $n\sim\mbox{Bessel}_{\nu}(\alpha)$ as the Bessel distribution with parameters $\nu>-1$ and $\alpha>0$, with PMF 
\beq\notag
\mbox{Bessel}_{\nu}(n;\alpha)=\frac{\left(\frac{\alpha}{2}\right)^{2n+\nu}}{I_{\nu} (\alpha) n! \Gamma(n+\nu+1)}, ~~n\in\{0,1,2,\ldots\}.
\eeq
Algorithms to draw Bessel random variables can be found in \citet{devroye2002simulating}.

\begin{figure}[t]
\begin{center}
\includegraphics[width=0.47\textwidth]{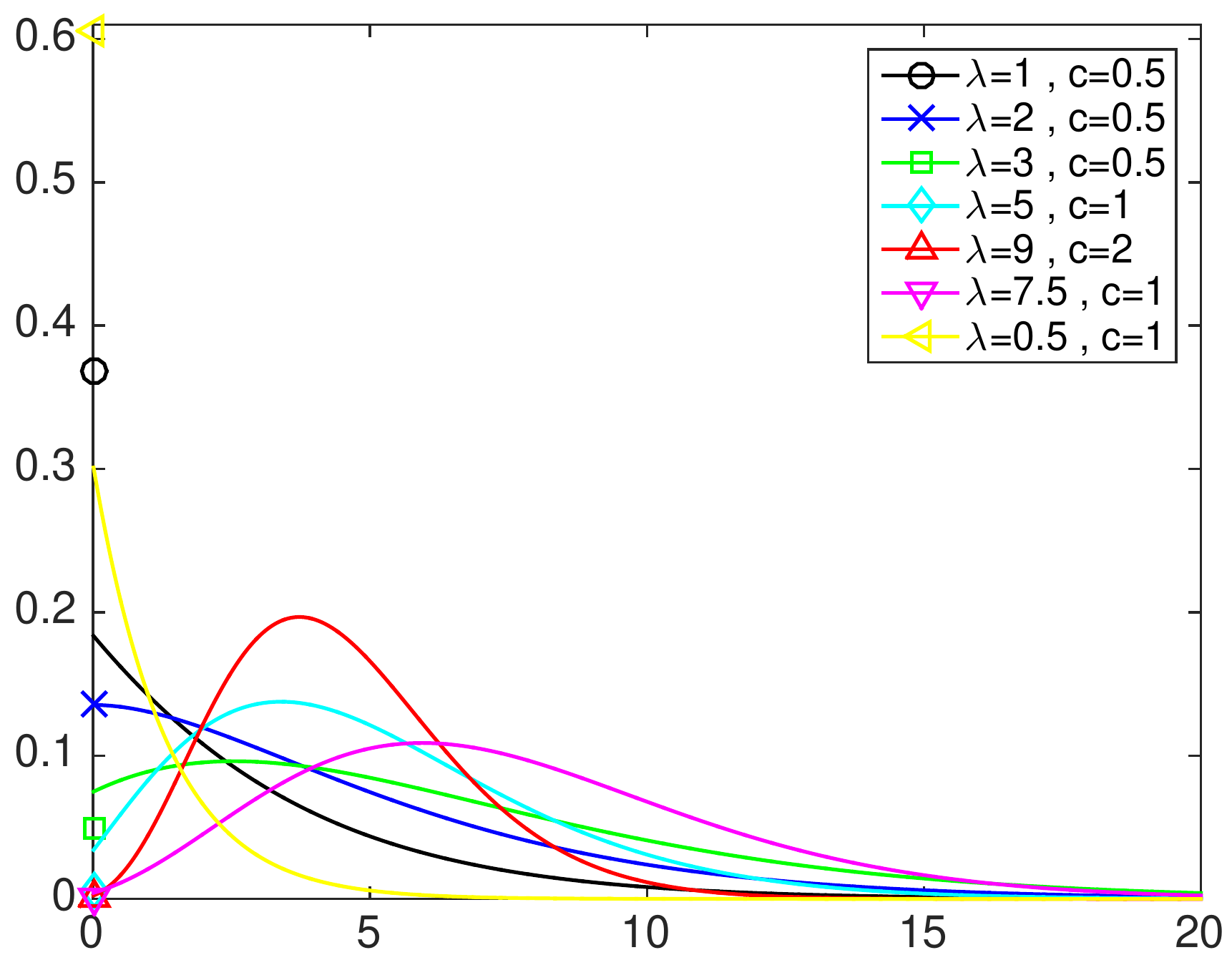}\,\,\,\,\,\,\,\,\,
\includegraphics[width=0.47\textwidth]{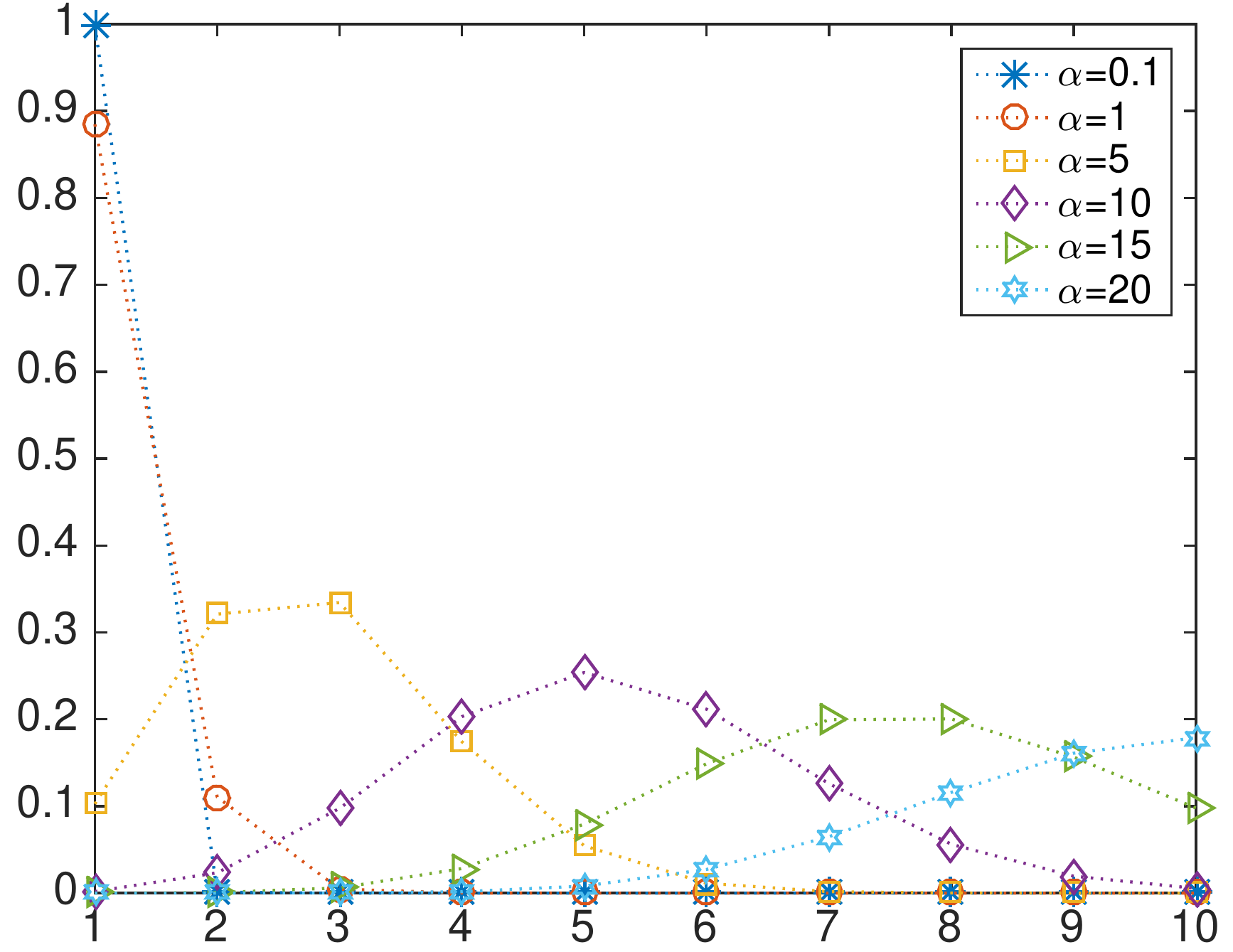}
\end{center}
\vspace{-7.mm}
\caption{\small \label{fig:pdf} 
Left: probability distribution functions for the Poisson randomized gamma (PRG) distribution $x\sim\mbox{PRG}(\lambda,c)$, where the sum of the probability mass at $x=0$ and the area under the probability density function curve for $x>0$ is equal to one; Right: probability mass functions for the truncated Bessel distribution $n\sim\mbox{Bessel}_{-1}(\alpha)$, where $n\in\{1,2,\ldots\}$. 
 }
\end{figure}

The proposed PRG is different from the randomized gamma distribution of \citet{yuan2000bessel} in that it models both positive real numbers and exact zeros, and the proposed truncated Bessel distribution $n\sim \mbox{Bessel}_{-1}(\alpha)$ is different from the Bessel distribution $n\sim \mbox{Bessel}_{\nu}(\alpha)$, where $\nu>-1$, in that it is defined only on positive integers. For illustration, we show in Figure \ref{fig:pdf} the probability distribution functions of both the PRG and truncated Bessel distributions under a variety of parameter settings.

\section{Upward-Downward Gibbs Sampling}\label{sec:sampling}
Below we first discuss Gibbs sampling for count data and then generalize it for both binary and nonnegative real data.
\subsection{Inference for the PGBN}
With Lemma \ref{lem:PGBN} and Corollary \ref{cor:PGBN} and the width of the first layer being bounded by $K_{1\max}$, 
we first consider multivariate count observations and develop an upward-downward Gibbs sampler for the PGBN, each iteration of which
proceeds as follows.

\emph{\textbf{Sample $x_{vjk}^{(t)}$}}. 
We can sample $x_{vjk}^{(t)}$ for all layers
using (\ref{eq:step1}). But for the first hidden layer, we may treat each observed count $x_{vj}^{(1)}$ as a sequence of 
word tokens at the $v$th term (in a vocabulary of size $V:=K_0$) in the $j$th document, and assign the $x_{\cdotv j}^{(1)}$ words $\{v_{ji}\}_{i=1,x_{\cdotv j}^{(1)}}$ one after another to the latent factors (topics), with both the topics $\Phimat^{(1)}$ and topic weights $\thetav_j^{(1)}$ marginalized out, as
\beqs
P(z_{ji}=k\,|\,-) 
\propto
\frac{\eta^{(1)}+x_{v_{ji}\cdotv k}^{(1)^{-ji}}}{V\eta^{(1)}+ x_{\cdotv \cdotv k}^{(1)^{-ji}}} \left(x_{\cdotv j k}^{(1)^{-ji}}+\phiv^{(2)}_{k:} \thetav_j^{(2)}\right),~~~k\in\{1,\ldots,K_{1\max}\}, \label{eq:z}
\eeqs
where $z_{ji}$ is the topic index for $v_{ji}$ and $x_{vjk}^{(1)} := \sum_{i}\delta(v_{ji}=v,z_{ji}=k)$ counts the number of times that term $v$ appears in document~$j$; we 
use $x^{-ji}$ to denote the count $x$ calculated without considering word $i$ in document $j$. The collapsed Gibbs sampling update equation shown above is related to the one developed in \citep{FindSciTopic} for latent Dirichlet allocation, and the one developed in \citep{BNBP_EPPF} for PFA using the beta-negative binomial process. When $T=1$, 
we would replace the terms $\phiv^{(2)}_{k:} \thetav_j^{(2)}$ with $r_k$ for PFA built on the gamma-negative binomial process \citep{NBP2012} (or with $\alpha \pi_k$ for hierarchical Dirichlet process latent Dirichlet allocation, see \citep{HDP} and \citep{BNBP_EPPF} for details), 
and add an additional term to account for the possibility of creating an additional factor \citep{BNBP_EPPF}. For simplicity, in this paper, we truncate the nonparametric Bayesian model with $K_{1\max}$ factors and let $r_k\sim\mbox{Gam}(\gamma_0/K_{1\max},1/c_0)$ if $T=1$. Note that although we use collapsed Gibbs sampling inference in this paper, if one desires embarrassingly parallel inference and possibly lower computation, then one may consider explicitly sampling $\{\phiv_k^{(1)}\}_k$ and $\{\thetav_{j}^{(1)}\}_j$ and sampling $x_{vjk}^{(1)}$ with (\ref{eq:step1}).  

\emph{\textbf{Sample $\phiv_k^{(t)}$}}. Given these latent counts, 
we sample the factors/topics $\phiv^{(t)}_k$ as
\beq
(\phiv^{(t)}_k\,|\,-)\sim\mbox{Dir}\left( \eta_1^{(t)}+ x^{(t)}_{1\cdotv k},\ldots, \eta_{K_{t-1}}^{(t)}+ x^{(t)}_{K_{t-1} \cdotv k} \right).\label{eq:step2}
\eeq

\emph{\textbf{Sample $x_{vj}^{(t+1)}$}}. We sample $\xv_j^{(t+1)}$ using (\ref{eq:CRT}), where we replace the term $\phiv_{v:}^{(T+1)}\thetav_{j}^{(T+1)}$ with~$r_v$. \\

\emph{\textbf{Sample $\rv$}}. Both $\gamma_0$ and $c_0$ are sampled using related equations in \citep{NBP2012}, omitted here for brevity. We sample $\rv$ as 
\beqs
&(r_v\,|\,-)\sim\mbox{Gam}\left({\gamma_0}/K_T+x_{v\cdotv}^{(T+1)},\left[{c_0-\textstyle \sum_j\ln\big(1-p_j^{(T+1)}\big)}\right]^{-1}\right). \label{eq:step5}
\eeqs

\emph{\textbf{Sample $\thetav_j^{(t)}$}}. 
Using (\ref{eq:deepPFA_aug}) and the gamma-Poisson conjugacy, we sample $\thetav_j$ as
\beqs
&(\thetav_j^{(T)}\,|\,-)\sim\mbox{Gam}\left(\rv + \mv_j^{(T)(T+1)},\left[{c_j^{(T+1)}-\ln\left(1-p_j^{(T)}\right)}\right]^{-1}\right), \notag \\ 
&\vdots
\notag\\
&(\thetav_j^{(t)}\,|\,-)\sim\mbox{Gam}\left(\Phimat^{(t+1)}\thetav_j^{(t+1)} + \mv_j^{(t)(t+1)},\left[{c_j^{(t+1)}-\ln\left(1-p_j^{(t)}\right)}\right]^{-1}\right), \notag\\ 
&\vdots
\notag\\
&(\thetav_j^{(1)}\,|\,-)\sim\mbox{Gam}\left(\Phimat^{(2)}\thetav_j^{(2)} + \mv_j^{(1)(2)},\left[{c_j^{(2)}-\ln\left(1-p_j^{(1)}\right)}\right]^{-1}\right), \label{eq:step4}
\eeqs

%
%
\emph{\textbf{Sample $c_j^{(t)}$}}. With $\theta_{\cdotv j}^{(t)}:=\sum_{k=1}^{K_{t}}\theta_{kj}^{(t)}$ for $t\le T$ and $\theta_{\cdotv j}^{(T+1)}:=r_{\cdotv}$, we sample $p_j^{(2)}$ and $\{c_j^{(t)}\}_{t\ge 3}$ as
\begin{align}
&(p_j^{(2)}\,|\,-)\sim\mbox{Beta}\left(a_0\!+\! m_{\cdotv j}^{(1)(2)},b_0\!+\!\theta_{\cdotv j}^{(2)}\!\right),~~ (c_j^{(t)}\,|\,-)\sim \mbox{Gam}\Big(e_0 \!+\! \theta_{\cdotv j}^{(t)}, \left[{f_0\!+\!\theta_{\cdotv j}^{(t-1)}}\!\right]^{-1}\!\Big), \label{eq:step6}
\end{align}
and calculate $c_j^{(2)}$ and $\{p_j^{(t)}\}_{t\ge3}$ with (\ref{eq:p}).


\subsection{Handling Binary and Nonnegative Real Observations}
For binary observations that are linked to the latent counts at layer one as $b_{vj}^{(1)}=\mathbf{1}(x_{vj}^{(1)}\ge 1)$, we first sample the latent counts at layer one from the truncated Poisson distribution as
\beq\label{eq:Ber-PGBN}
\big(x_{vj}^{(1)}\,|\,- \big)\sim b_{vj}^{(1)}\cdotv \mbox{Pois}_{+}\left( \sum_{k=1}^{K_1}\phi_{vk}^{(1)}\theta_{kj}^{(1)}\right)
\eeq
and then sample $x_{vjk}^{(t)}$ for all layers
using (\ref{eq:step1}).

For nonnegative real observations $y_{vj}^{(1)}$ that are linked to the latent counts at layer one~as $$y_{vj}^{(1)}\sim\mbox{Gam}(x_{vj}^{(1)},1/a_j),$$ we let $x_{vj}^{(1)}=0$ if $y_{vj}^{(1)}=0$ and sample $x_{vj}^{(1)}$ from the truncated Bessel distribution as
\beq\label{eq:Gam-PGBN}
\big(x_{vj}^{(1)}\,|\, - \big)\sim \mbox{Bessel}_{-1}\left( 2\sqrt{a_j y_{vj}^{(1)} \sum_{k=1}^{K_1}\phi_{vk}^{(1)}\theta_{kj}^{(1)}} \right)
\eeq
if $y_{vj}^{(1)}>0$. We let $a_j\sim\mbox{Gam}(e_0,1/f_0)$ in the prior and sample $a_j$ as
\beq\label{eq:c_g}
(a_j\,|\,-)\sim\mbox{Gam}\left(e_0+\sum_{v} x_{vj}^{(1)},\frac{1}{f_0+\sum_{v} y_{vj}^{(1)}}\right).
\eeq
 We then sample $x_{vjk}^{(t)}$ for all layers
using (\ref{eq:step1}).

\section{Additional Tables and Figures}\label{sec:fig}

\begin{algorithm}[h]
\small
 \caption{\small The PGBN upward-downward Gibbs sampler that uses a layer-wise training strategy to train a set of networks, each of which adds an additional hidden layer on top of the previously inferred network, retrains all its layers jointly, and prunes inactive factors from the last layer. 
  \textbf{Inputs:} observed counts $\{x_{vj}\}_{v,j}$, upper bound of the width of the first layer $K_{1\max}$, upper bound of the number of layers $T_{\max}$, number of iterations $\{B_T,S_T\}_{1,T_{\max}}$, and hyper-parameters.\newline
  \textbf{Outputs:} A total of $T_{\max}$ jointly trained PGBNs with depths $T=1$, $T=2$, $\ldots$, and $T=T_{\max}$.
 }\label{tab:algorithm}
 \begin{algorithmic}[1] 
  \For{\text{$T=1,2,\ldots,T_{\max}$}} Jointly train all the $T$ layers of the network
  \State
  Set $K_{T-1}$, the inferred width of layer $T-1$, as $K_{T\max}$, the upper bound of layer $T$'s width. 
  \For{\text{$iter=1: B_T+C_T$ 
  }} Upward-downward Gibbs sampling
   
   \State 
    \text{Sample $\{z_{ji}\}_{j,i}$ using collapsed inference; Calculate $\{x_{vjk}^{(1)}\}_{v,k,j}$};
     \text{Sample $\{x_{vj}^{(2)}\}_{v,j}$} ;
   \For{\text{$t=2,3,\ldots,T$}}
    \State 
    \text{Sample $\{x_{vjk}^{(t)}\}_{v,j,k}$} ; 
     \text{ Sample $\{\phiv_k^{(t)}\}_{k}$} ;
    \text{ Sample $\{x_{vj}^{(t+1)}\}_{v,j}$} ;
   \EndFor
   \State
    \text{Sample $p_j^{(2)}$ and Calculate $c_j^{(2)}$}; Sample $\{c^{(t)}_j\}_{j,t}$ and Calculate $\{p^{(t)}_j\}_{j,t}$ for $t=3,\ldots,T+1$;\!\!
    \For{\texttt{$t=T,T-1,\ldots,2$}}
   \State 
    \text{Sample $\rv$} if $t=T$;
    \text{Sample $\{\thetav_j^{(t)}\}_j$} ;
   \EndFor
    
   \If{ $iter=B_T$}
   \State 
   Prune layer $T$'s inactive factors $\{\phiv_{k}^{(T)}\}_{k:x_{\cdotv \cdotv k}^{(T)}=0}$;
   \State let $K_T=\sum_k {\delta(x_{\cdotv \cdotv k}^{(T)}>0)}$ and 
    update $\rv$; 

   \EndIf
  \EndFor

  \State
  Output the posterior means (according to the last MCMC sample) of all remaining factors 
  $\{\phiv_k^{(t)}\}_{k,t}$
  as the inferred network of $T$ layers, and $\{r_k\}_{k=1}^{K_T}$ as 
  the gamma shape parameters of layer $T$'s 
  hidden units. 
   \EndFor
 \end{algorithmic} \normalsize
\end{algorithm}%

\begin{algorithm}[t]
\small
\caption{\small The upward-downward Gibbs samplers for the Ber-GBN and PRG-GBN are constructed by using Lines 1-8 shown below to substitute Lines 4-11 of the PGBN Gibbs sampler shown in Algorithm \ref{tab:algorithm}. 
 }\label{tab:algorithm2}
  \begin{algorithmic}[1] 
    \State 
     \text{Sample $\{x_{vj}^{(1)}\}_{v,j}$} using (\ref{eq:Ber-PGBN}) for binary observations; \text{Sample $\{x_{vj}^{(1)}\}_{v,j}$} using (\ref{eq:Gam-PGBN}) and sample $a_j$ using \eqref{eq:c_g} for nonnegative real observations;
   \For{\text{$t=1,2,\ldots,T$}}
    \State 
    \text{Sample $\{x_{vjk}^{(t)}\}_{v,j,k}$} ; 
     \text{ Sample $\{\phiv_k^{(t)}\}_{k}$} ;
    \text{ Sample $\{x_{vj}^{(t+1)}\}_{v,j}$} ;
   \EndFor
   \State
   \text{Sample $p_j^{(2)}$ and Calculate $c_j^{(2)}$}; Sample $\{c^{(t)}_j\}_{j,t}$ and Calculate $\{p^{(t)}_j\}_{j,t}$ for $t=3,\ldots,T+1$;
    \For{\texttt{$t=T,T-1,\ldots,1$}}
   \State 
   \text{Sample $\rv$} if $t=T$;
    \text{Sample $\{\thetav_j^{(t)}\}_j$} ;
   \EndFor

    \end{algorithmic}
 \normalsize
\end{algorithm}

\begin{table}[h]
\begin{footnotesize}
\begin{center}
\begin{tabular}{ c c c c}
\toprule
 $V=61,188$ & $V=61,188$ & $V=33,420$ & $V=33,420$ \\
 with stopwords & with stopwords & remove stopwords & remove stopwords \\
 with rare words& with rare words & remove rare words & remove rare words\\
 raw word counts & term frequencies & raw word counts & term frequencies \\
\midrule
 75.8\% & 77.6\% & 78.0\% & 79.4\%\\
\bottomrule
\end{tabular}
\end{center}
%
%
\begin{center}
\begin{tabular}{ c c c c}
\toprule
 $V=2000$ & $V=2000$ & $V=2000$ & $V=2000$ \\
 with stopwords & with stopwords & remove stopwords & remove stopwords \\
 raw counts & term frequencies & raw counts & term frequencies \\
\midrule
 68.2\% & 67.9\% & 69.8\% & 70.8\%\\
\bottomrule
\end{tabular}
\end{center}
\vspace{-3mm}
\caption{\small Multi-class classification accuracies of $L_2$ regularized logistic regression.}\label{tab:LR}
\end{footnotesize}
\end{table}%

\begin{table}[h]
\begin{footnotesize}
\begin{center}
\begin{tabular}{ c | c c}
\toprule
 &$V=2000$, $K=128$ & $V=2000$, $K=512$\\
 & remove stopwords, stemming & remove stopwords, stemming \\
\midrule
DocNADE & 67.0\% & 68.4\%\\
Over-replicated softmax &66.8\% & 69.1\%\\
\bottomrule
\end{tabular}
\end{center}
\vspace{-3mm}
\caption{\small Multi-class classification accuracies of the DocNADE \citep{larochelle2012neural} and over-replicated softmax \citep{srivastava2013modeling}.}\label{tab:ORS}
\end{footnotesize}
\end{table}%

\begin{table}[h]
\begin{footnotesize}
%
%
%
\begin{center}
\begin{tabular}{ c | c c c}
\toprule
 &$V=2000$, $K_{1\max}=128$ &$V=2000$, $K_{1\max}=256$ & $V=2000$, $K_{1\max}=512$\\
 & remove stopwords& remove stopwords& remove stopwords \\
\midrule
PGBN ($T=1$) & $65.9\% \pm 0.4\%$& $66.3\% \pm0.4\%$ & $65.9\% \pm0.4\%$\\
PGBN ($T=2$) & $67.1\% \pm 0.5\%$& $67.9\% \pm0.4\%$ & $68.3\% \pm0.3\%$\\
PGBN ($T=3$) & $67.3\% \pm0.3\%$ & $68.6\% \pm0.5\%$ & $69.0\% \pm0.4\%$\\
PGBN ($T=5$) & $67.5\% \pm0.4\%$ & $68.8\% \pm0.3\%$ & $69.2\% \pm0.4\%$\\
\bottomrule
\end{tabular}
\end{center}
\begin{center}
\begin{tabular}{ c | c c c}
\toprule
 &$V=33,420$, $K_{1\max}=200$ &$V=33,420$, $K_{1\max}=400$ & $V=33,420$, $K_{1\max}=800$\\
 & remove stopwords& remove stopwords& remove stopwords \\
 & remove rare words& remove rare words& remove rare words \\
\midrule
PGBN ($T=1$) & $74.6\% \pm 0.6\%$& $75.3\% \pm0.6\%$ & $75.4\% \pm0.4\%$\\
PGBN ($T=2$) & $76.0\% \pm 0.6\%$& $76.9\% \pm0.5\%$ & $77.5\% \pm0.4\%$\\
PGBN ($T=3$) & $76.3\% \pm0.8\%$ & $77.1\% \pm0.6\%$ & $77.8\% \pm0.4\%$\\
PGBN ($T=5$) & $76.4\% \pm0.5\%$ & $77.4\% \pm0.6\%$ & $77.9\% \pm0.3\%$\\
\bottomrule
\end{tabular}
\end{center}
\vspace{-3mm}
\caption{\small Multi-class classification accuracies of the PGBN trained with $\eta^{t}=0.05$ for all $t$.}\label{tab:PGBN}
\end{footnotesize}
\end{table}%



\clearpage

\begin{figure}[!t]
\begin{center}
 \includegraphics[width=0.965\textwidth]{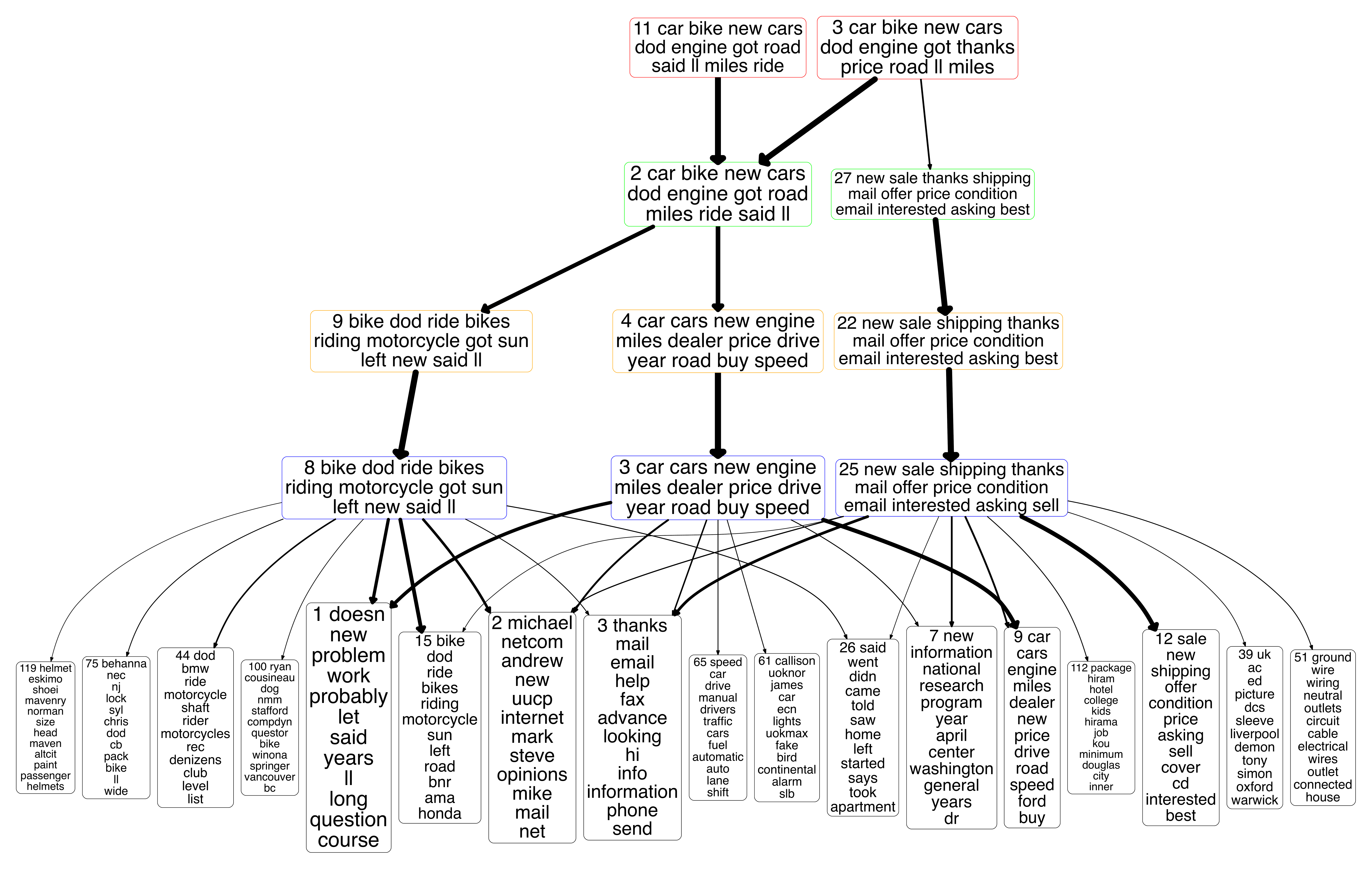}
\end{center}
\vspace{-8mm}
\caption{\small \label{fig:car}
Analogous plots to Figure \ref{fig:windows} 
for a subnetwork on ``car \& bike'', consisting of two trees rooted at nodes 3 and 11, respectively, of layer one. 
}

\begin{center}
 \includegraphics[width=0.965\textwidth]{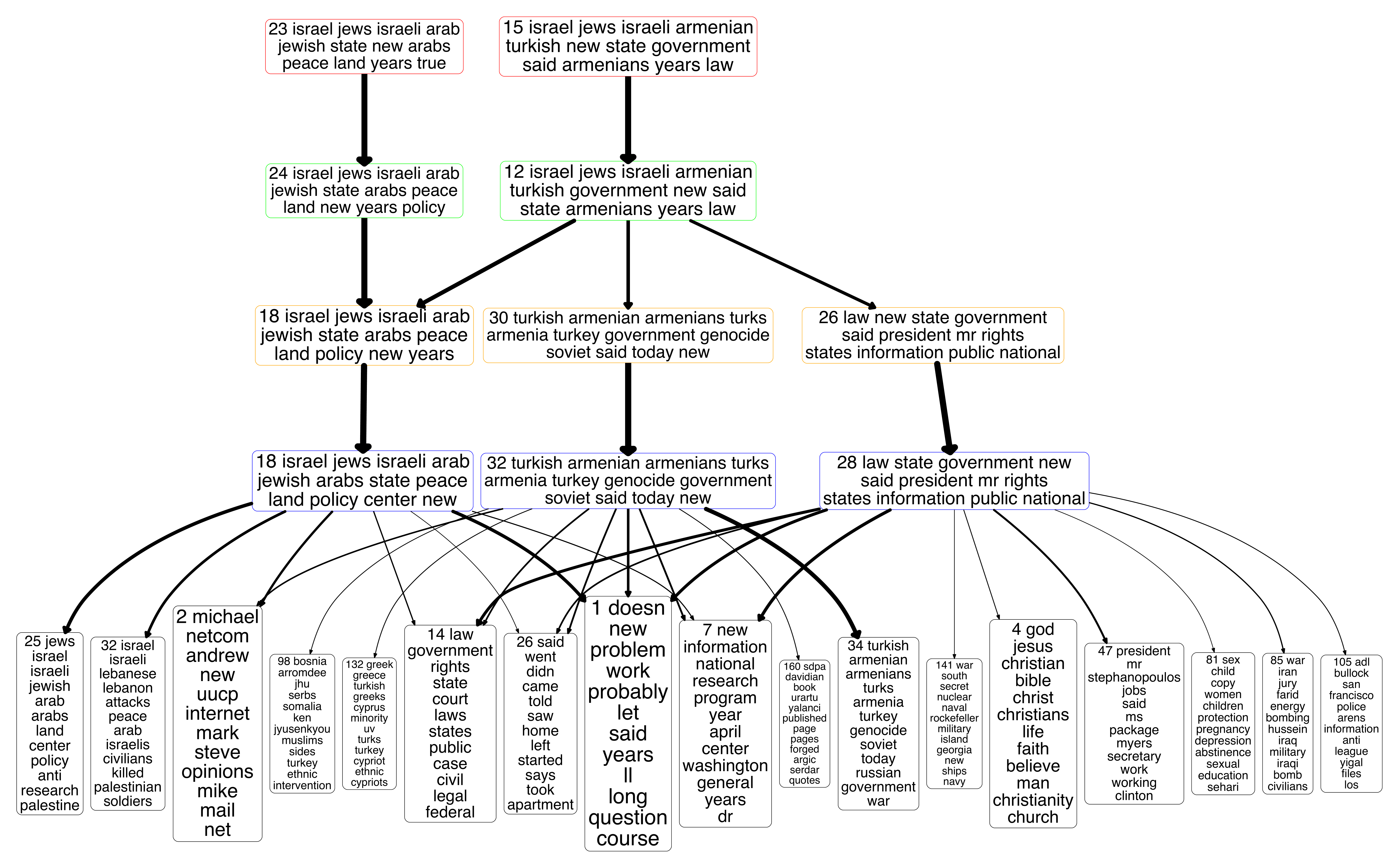}\, \,
\end{center}
\vspace{-10mm}
\caption{\small \label{fig:mideast}
Analogous plot to Figure \ref{fig:windows} 
for a subnetwork on ``Middle East,'' consisting of two trees rooted at nodes 15 and 23, respectively, of layer one. 
}
\end{figure}

\begin{figure}[!t]
\begin{center}
\includegraphics[width=0.76\textwidth]{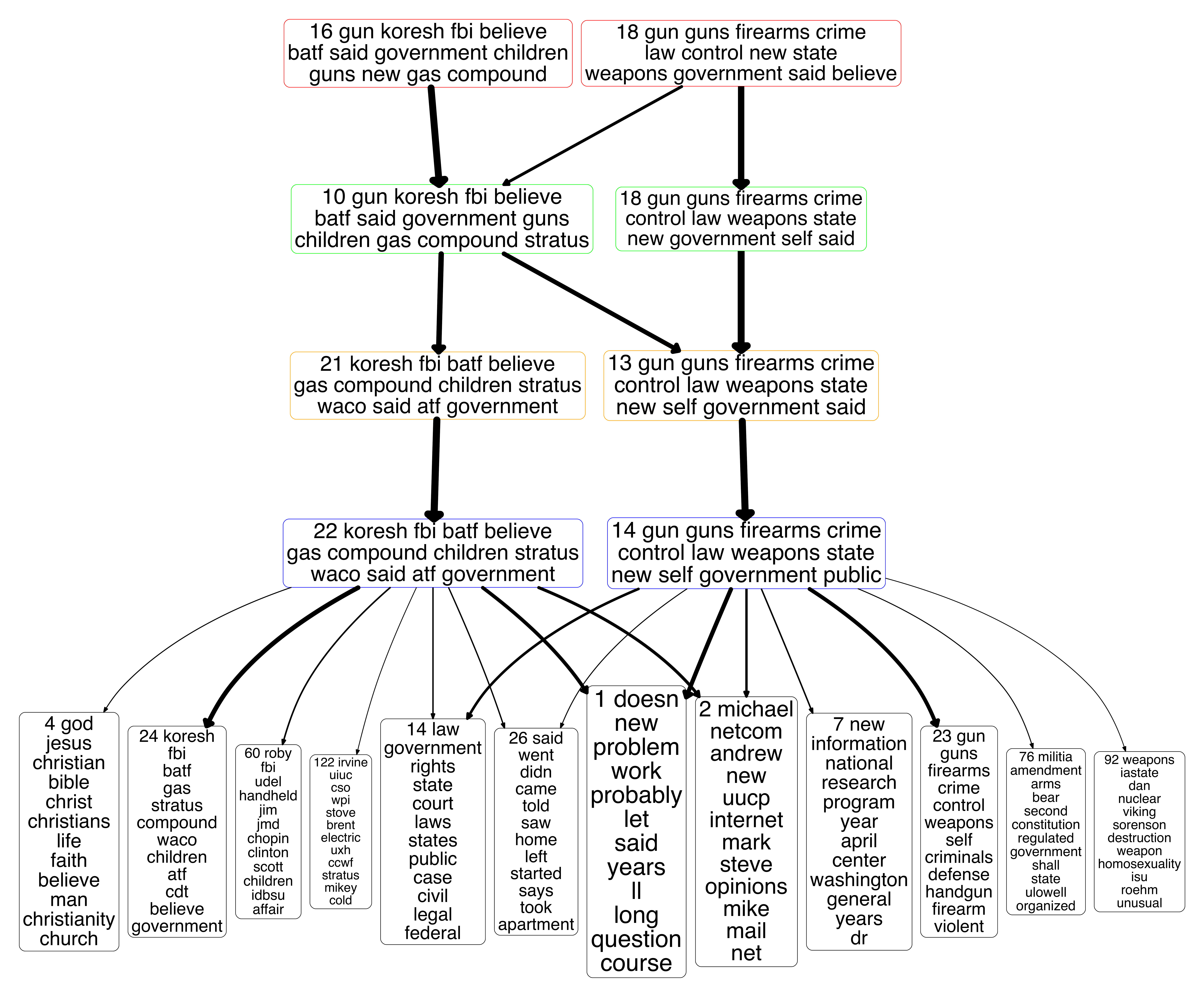}
\end{center}
\vspace{-6.mm}
\caption{\small \label{fig:gun} 
Analogous plot to Figure \ref{fig:windows} 
for a subnetwork related to ``gun,'' consisting of two trees rooted at nodes 16 and 18, respectively, of layer one.
 }

\begin{center}
 \includegraphics[width=0.83\textwidth]{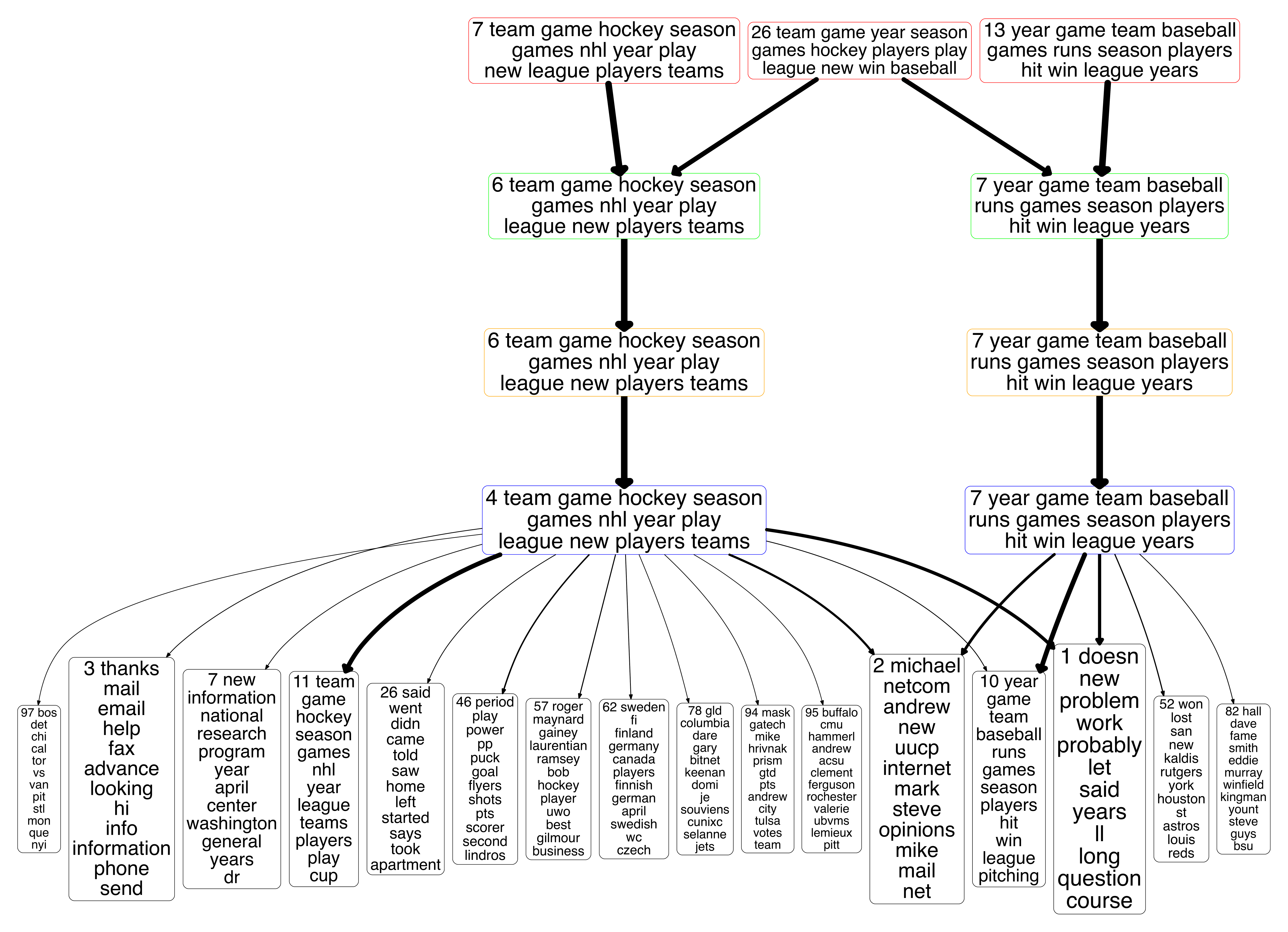}
\end{center}
\vspace{-10mm}
\caption{\small \label{fig:sports}
Analogous plot to Figure \ref{fig:windows} 
for a subnetwork on ``ice hockey'' and ``baseball,'' consisting of three trees rooted at nodes 7, 13, and 26, respectively, of layer one.
} 
\end{figure}

\begin{figure}[!t]
\begin{center}
 \vspace{-4mm}\includegraphics[width=0.9\textwidth]{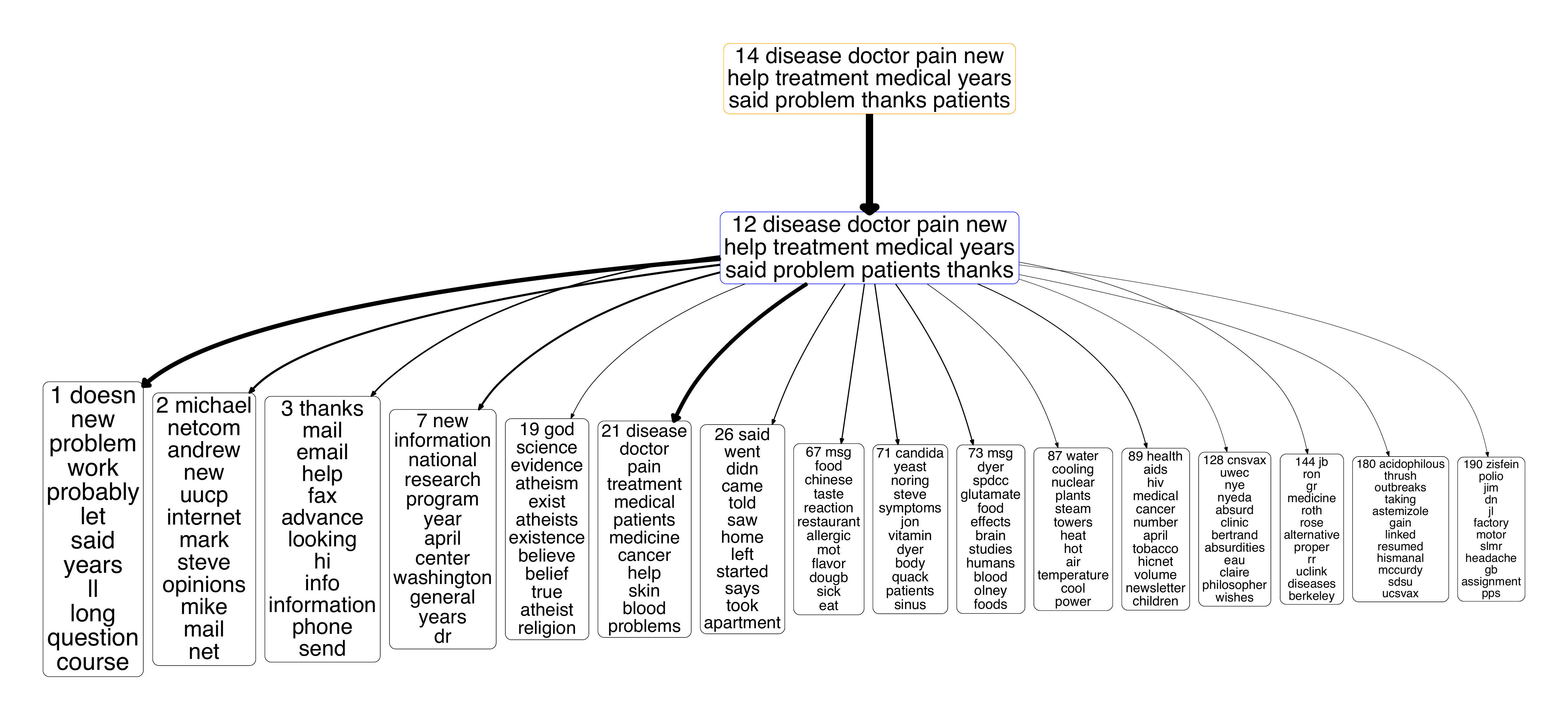}\vspace{-4mm}
 \includegraphics[width=0.8\textwidth]{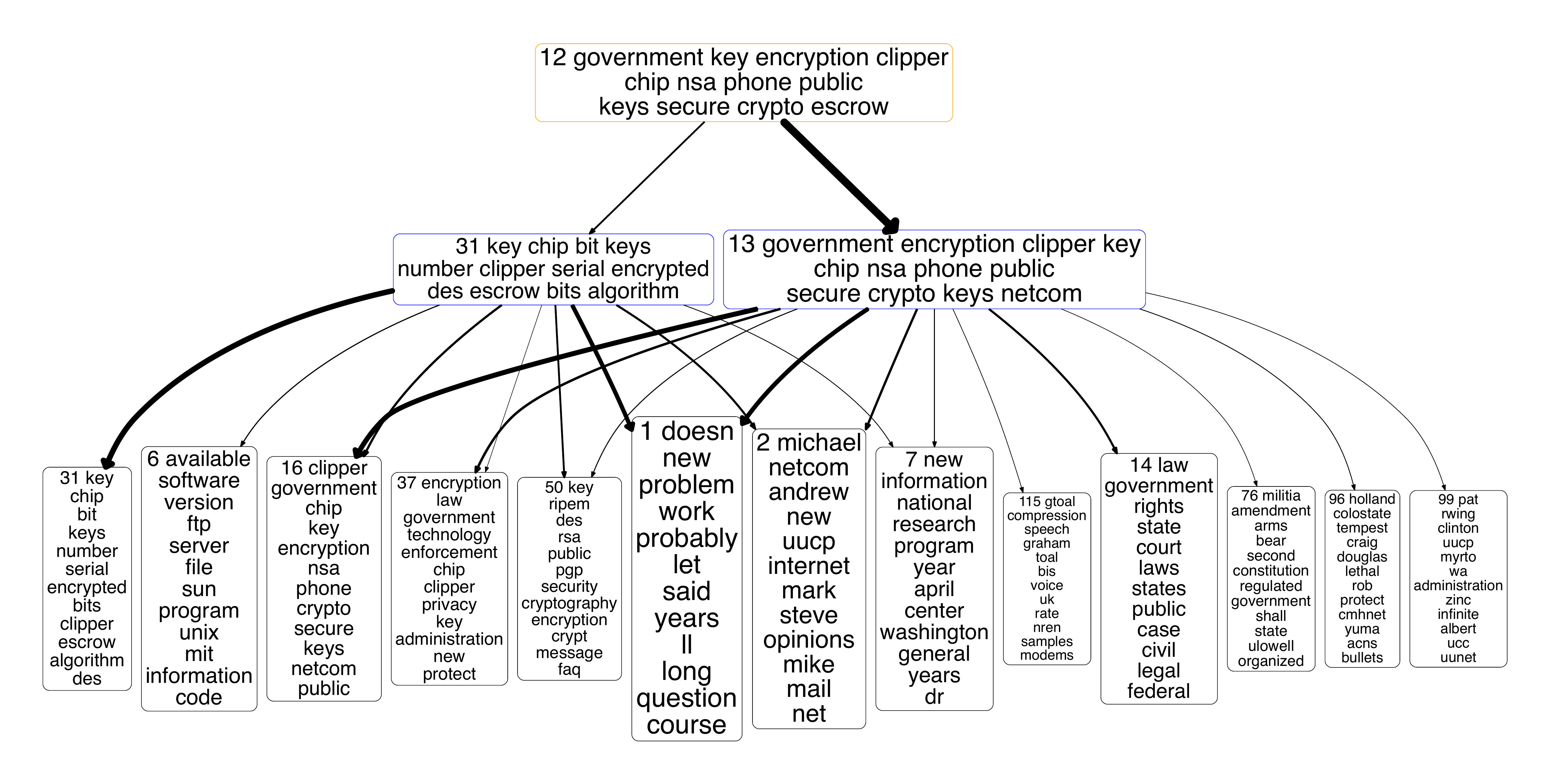}\vspace{-3mm}
 \includegraphics[width=0.94\textwidth]{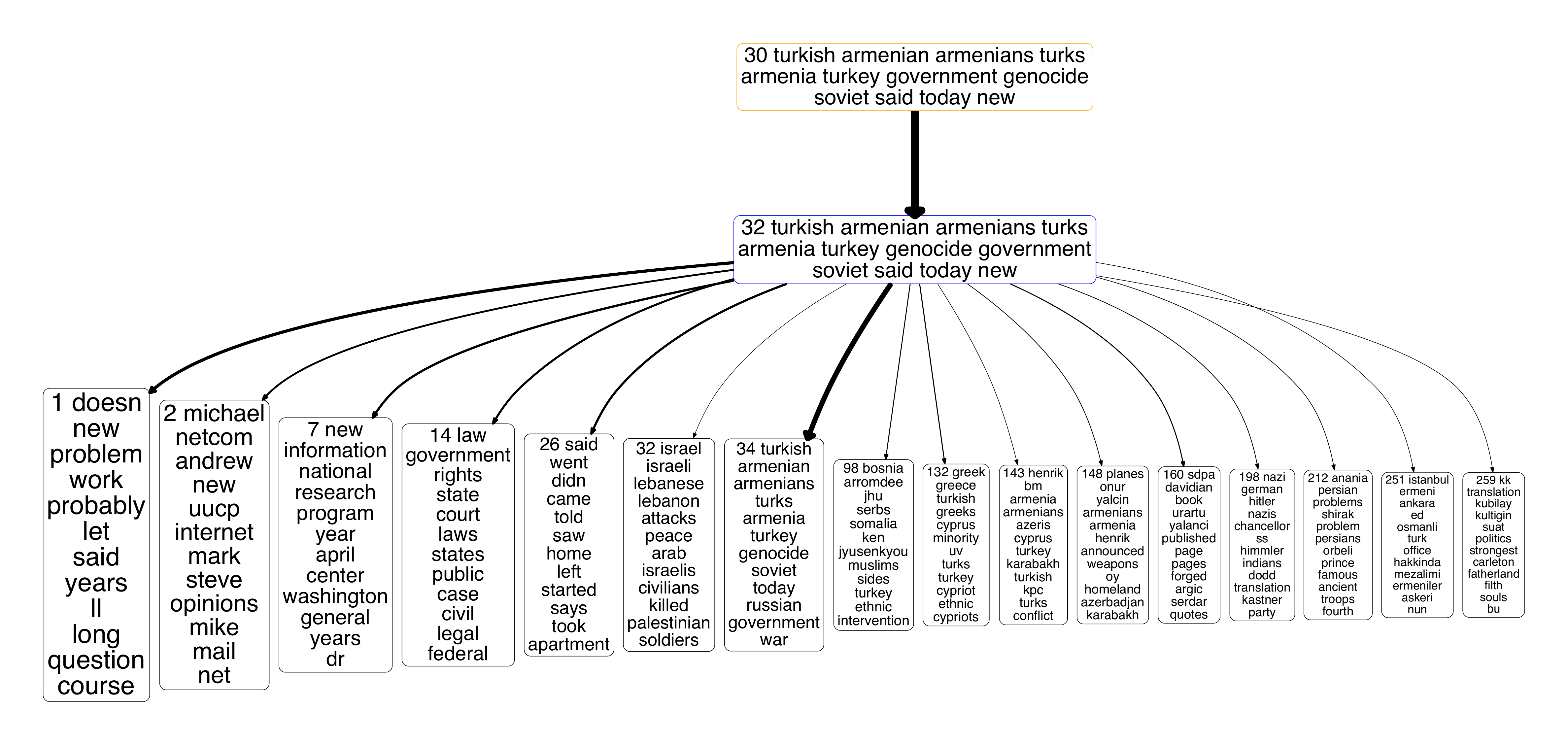}
\end{center}
\vspace{-8mm}
\caption{\small \label{fig:turkey}
Analogous plots to Figure \ref{fig:windows}, with $\tau_t=1$ to reveal more weak links. 
Top: the tree rooted at node 14 of layer three on ``medicine.'' Middle: the tree rooted at node 12 of layer three on ``encryption.'' Bottom: 
the tree rooted at node 30 of layer three on ``Turkey \& Armenia.''
}
\end{figure}

\begin{figure}[!tb]
 \centering
 \subfigure[] {
\includegraphics[scale=0.27]{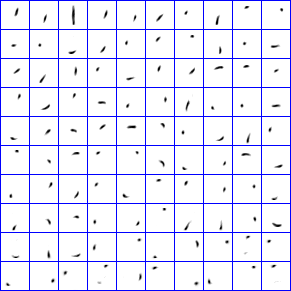}
}\!\!
 \subfigure[] {
\includegraphics[scale=0.27]{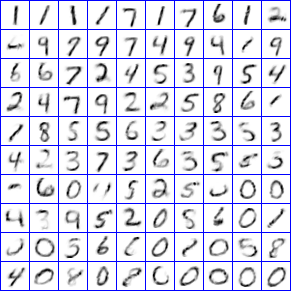}
}\!\!
 \subfigure[] {
\includegraphics[scale=0.27]{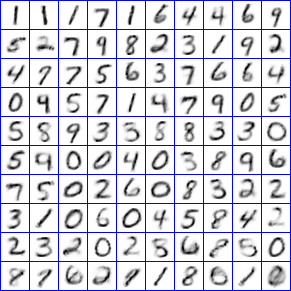}
}\!\!
 \subfigure[] {
\includegraphics[scale=0.27]{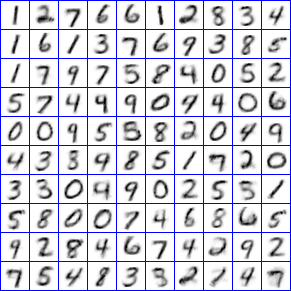}
}\!\!
 \subfigure[] {
\includegraphics[scale=0.27]{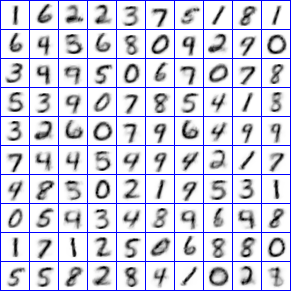}
}
\vspace{-0.3cm}
 \caption{\small Visualization of the inferred $\{\Phimat^{(1)},\cdot\cdot\cdot,\Phimat^{(T)}\}$ on the MNIST data set using the PRG-GBN with $K_{1max}=100$ and $\eta^{(t)}=0.05$ for all $t$. The latent factors  of all layers are projected to the first layer: (a) $\Phimat^{(1)}$, (b) $\Phimat^{(1)}\Phimat^{(2)}$, (c) $\Phimat^{(1)}\Phimat^{(2)}\Phimat^{(3)}$, (d) $\Phimat^{(1)}\Phimat^{(2)}\Phimat^{(3)}\Phimat^{(4)}$, and (e) $\Phimat^{(1)}\Phimat^{(2)}\Phimat^{(3)}\Phimat^{(4)}\Phimat^{(5)}$. \label{fig:AllProjPhis100}}
\end{figure}
%

%
\begin{figure}[!tb]
 \centering
 \subfigure[] {
\includegraphics[scale=0.22]{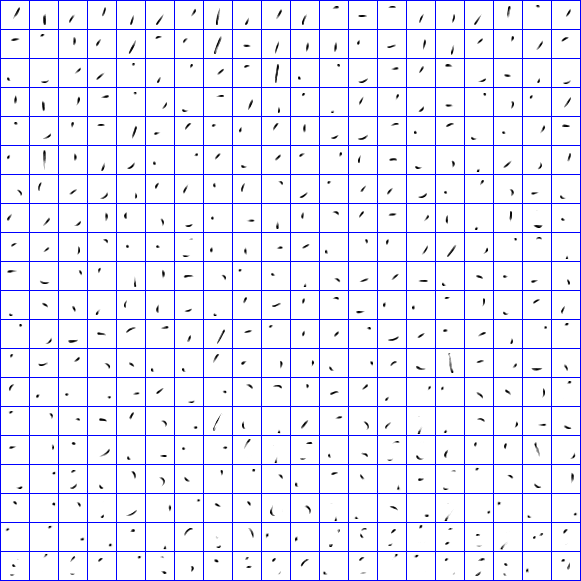}
}
 \subfigure[] {
\includegraphics[scale=0.22]{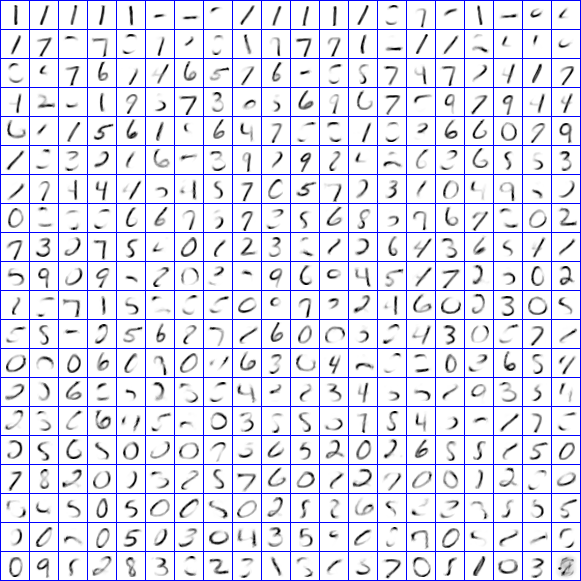}
}\\
 \subfigure[] {
\includegraphics[scale=0.22]{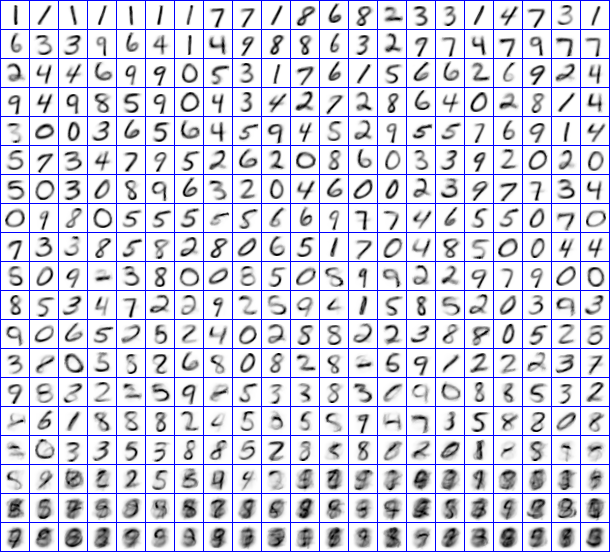}
}
 \subfigure[] {
\includegraphics[scale=0.22]{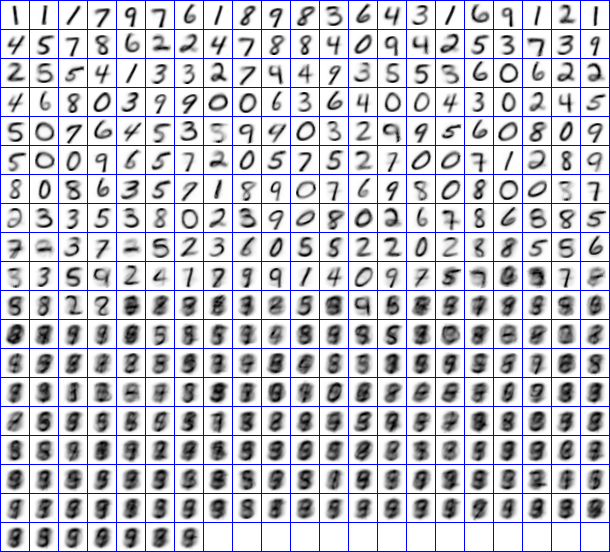}
}
 \subfigure[] {
\includegraphics[scale=0.22]{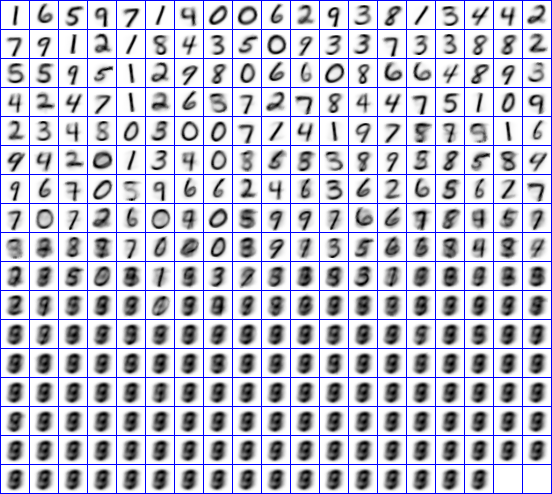}
}
\vspace{-0.3cm}
 \caption{\small Visualization of the inferred $\{\Phimat^{(1)},\cdot\cdot\cdot,\Phimat^{(T)}\}$ on the MNIST data set using the PRG-GBN with $K_{1max}=400$ and $\eta^{(t)}=0.05$ for all $t$. The latent factors of all layers are projected to the first layer: (a) $\Phimat^{(1)}$, (b) $\Phimat^{(1)}\Phimat^{(2)}$, (c) $\Phimat^{(1)}\Phimat^{(2)}\Phimat^{(3)}$, (d) $\Phimat^{(1)}\Phimat^{(2)}\Phimat^{(3)}\Phimat^{(4)}$, and (e) $\Phimat^{(1)}\Phimat^{(2)}\Phimat^{(3)}\Phimat^{(4)}\Phimat^{(5)}$. \label{fig:AllProjPhis400}}
\end{figure}

%

\begin{figure}[!tb]
 \centering
 \subfigure[] {
\includegraphics[width=0.46\textwidth]{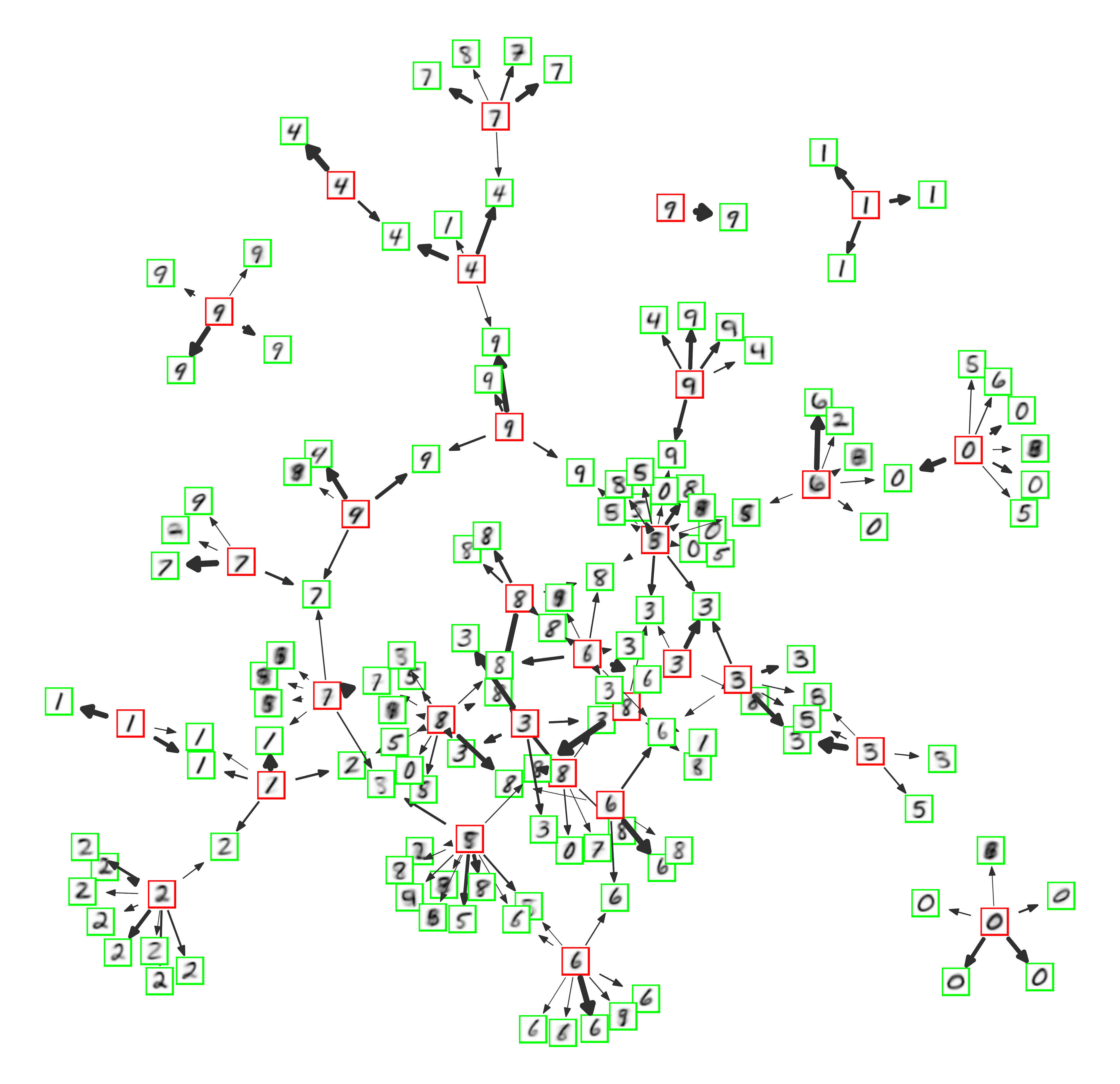}
}
 \subfigure[] {
\includegraphics[width=0.45\textwidth]{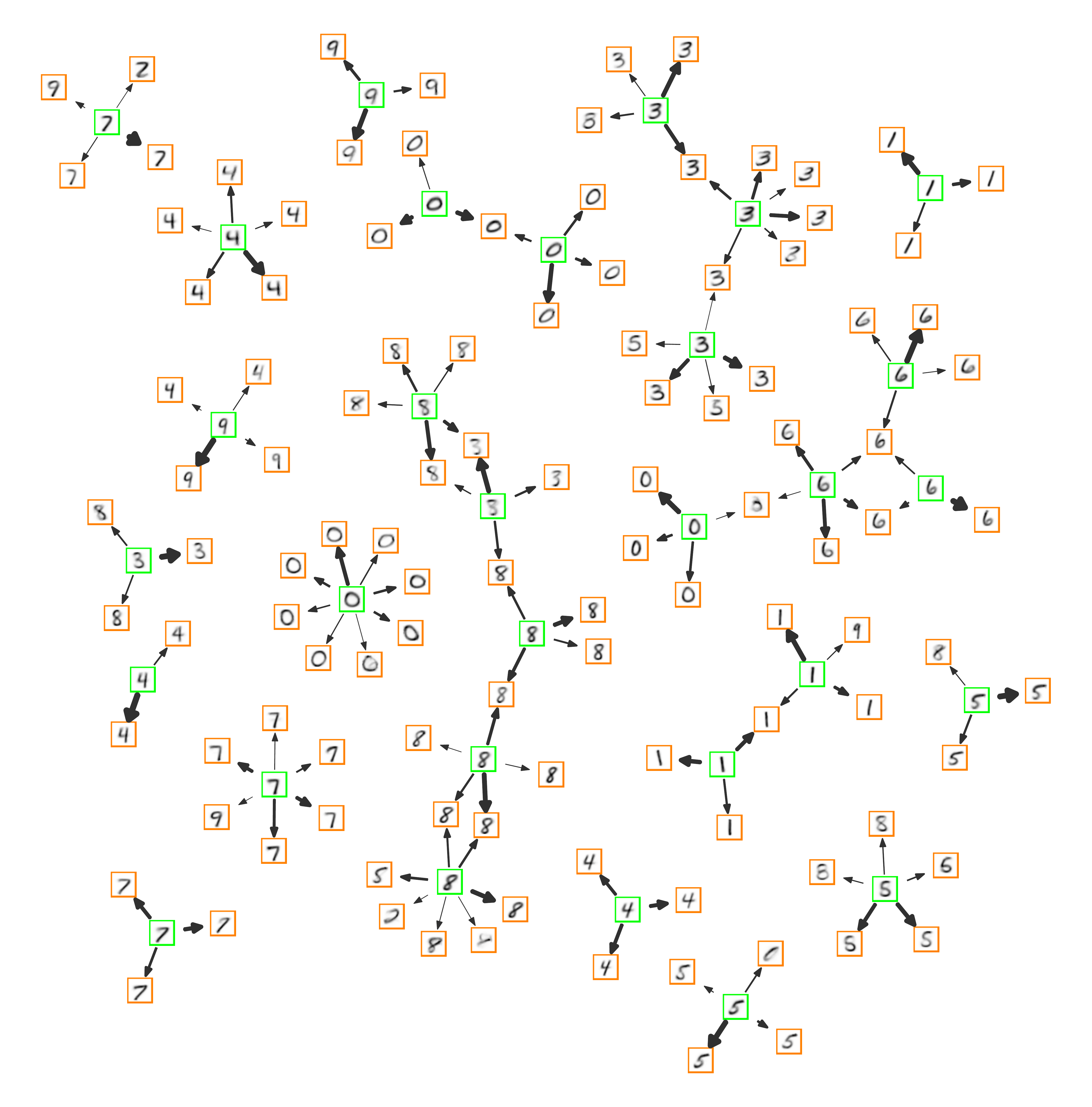}
}\\
 \subfigure[] {
\includegraphics[width=0.49\textwidth]{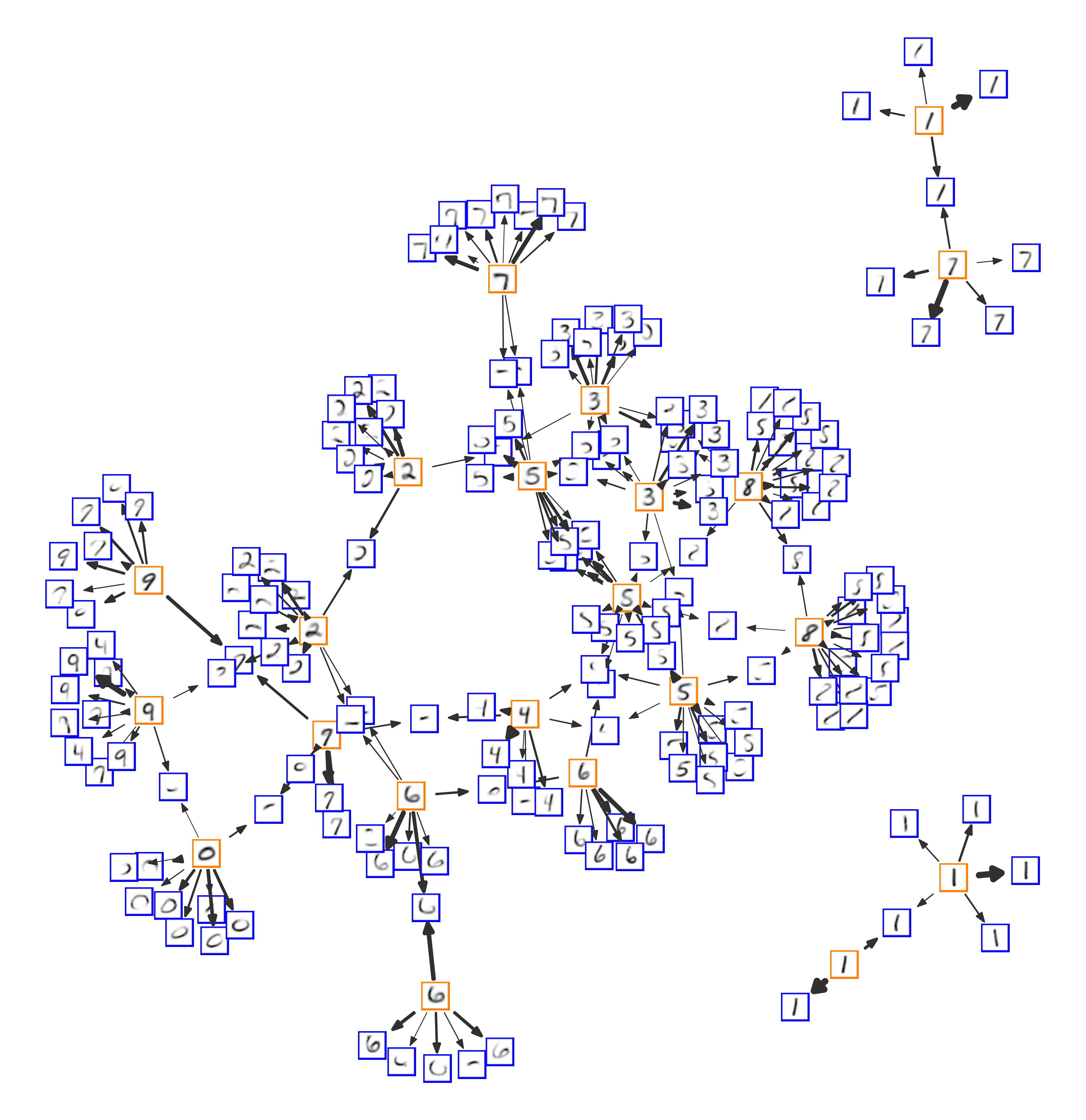}
}
 \subfigure[] {
\hspace{-0.00\textwidth}\includegraphics[width=0.40\textwidth]{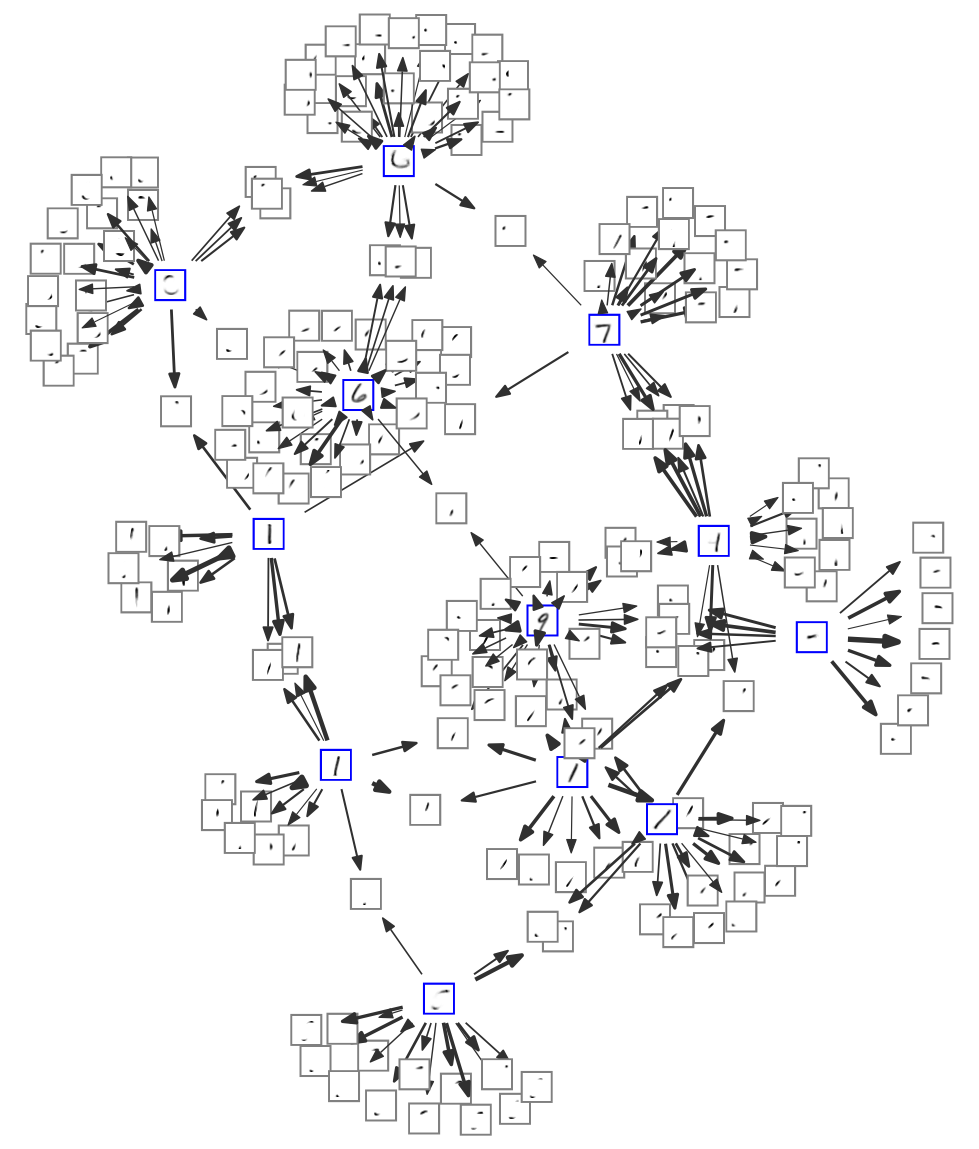}
}
\vspace{-0.2cm}
 \caption{\small 
 Visualization of the network structures inferred by the PRG-GBN on the MNIST data set with $K_{1max}=400$. 
 (a) Visualization of the factors $(\phiv^{(5)}_1,\phiv^{(5)}_{11},\phiv^{(5)}_{21},\ldots,\phiv^{(5)}_{111})$ of layer five and those of layer four that are strongly connected to them. (b) Visualization of the factors $(\phiv^{(4)}_1,\phiv^{(4)}_{6},\phiv^{(4)}_{11},\ldots,\phiv^{(4)}_{106})$ of layer four and those of layer three that are strongly connected to them. (c)the
 Visualization of the factors $(\phiv^{(3)}_1,\phiv^{(3)}_{6},\phiv^{(3)}_{11},\ldots,\phiv^{(3)}_{146})$ of layer three and those of layer two that are strongly connected to them. (d) Visualization of the factors $(\phiv^{(2)}_1,\phiv^{(2)}_{6},\phiv^{(2)}_{11},\ldots,\phiv^{(2)}_{146})$ of layer two and those of layer one that are strongly connected to them. 
 \label{fig:MNIST_tree}
 }
\end{figure}

\clearpage

\bibliography{References052016}

\end{document}